\documentclass[preprint]{amsart}

\input{preamble.tex}

\begin{document}

\title{Monge-Kantorovich Fitting With Sobolev Budgets}

\author[F.\ Kobayashi]{Forest Kobayashi}
\address{Forest Kobayashi\\ Department of Mathematics \\ University of
  British Columbia\\ Vancouver, Canada\\ {Email:
    {fkobayashi@math.ubc.ca}}}

\author[J.\ Hayase]{Jonathan Hayase}
\address{Jonathan Hayase\\ Paul G. Allen School of Computer Science \&
  Engineering \\ University of Washington \\ Seattle, USA \\ {Email:
    jhayase@cs.washington.edu}}

\author[Y.H.\ Kim]{Young-Heon Kim}
\address{Young-Heon Kim\\ Department of Mathematics \\ University of
  British Columbia\\ Vancouver, Canada\\ {Email: yhkim@math.ubc.ca}}



\thanks{FK is supported by the 4YF doctoral fellowship of the
  University of British Columbia. JH is supported by NSF award 2134012
  and the NSF Graduate Research Fellowship Program. YHK is partially
  supported by the Natural Sciences and Engineering Research Council
  of Canada (NSERC), with Discovery Grant RGPIN-2019-03926, as well as
  Exploration Grant (NFRFE-2019-00944) from the New Frontiers in
  Research Fund (NFRF). YHK is also a member of the Kantorovich
  Initiative (KI) which is supported by the PIMS Research Network
  (PRN) program of the Pacific Institute for the Mathematical Sciences
  (PIMS). We thank PIMS for their generous support; report identifier
  PIMS-20240925-PRN01. Part of this work was completed during YHK and
  FK's visit at the Korea Advanced Institute of Science and Technology
  (KAIST), and we thank them for their hospitality and the excellent
  environment. \copyright 2024 by the author. All rights reserved.}

\maketitle

\begin{abstract}
  Given $\ddim < \cdim$, we consider the problem of ``best''
  approximating an $\cdim\text{-d}$ probability measure $\cmeas$ via
  an $\ddim\text{-d}$ measure $\imeas$ such that $\supp \imeas$ has
  bounded total ``complexity.'' When $\cmeas$ is concentrated near an
  $\ddim\text{-d}$ set we may interpret this as a \emph{manifold
    learning problem with noisy data}. However, we do not restrict our
  analysis to this case, as the more general formulation has broader
  applications.

  We quantify $\nu$'s performance in approximating $\rho$ via the
  Monge-Kantorovich (also called \emph{Wasserstein}) $\objp$-cost
  $\wpp(\cmeas, \imeas)$, and constrain the complexity by requiring
  $\supp \nu$ to be coverable by an $\fmap : \rddim \to \rcdim$ whose
  $\wkp$ Sobolev norm is bounded by $\ell \geq 0$. This allows us to
  reformulate the problem as minimizing a functional $\objfp(\fmap)$
  under the Sobolev ``budget'' $\ell$. This problem is closely related
  to (but distinct from) \emph{principal curves with length
    constraints} when $\ddim=1, \sobk = 1$ and an unsupervised
  analogue of \emph{smoothing splines} when $\sobk > 1$. New
  challenges arise from the higher-order differentiability condition.

  We study the ``gradient'' of $\objfp$, which is given by a certain
  vector field that we call the {\em barycenter field}, and use it to
  prove a nontrivial (almost) strict monotonicity result. We also
  provide a natural discretization scheme and establish its
  consistency. We use this scheme as a toy model for a generative
  learning task, and by analogy, propose novel interpretations for the
  role regularization plays in improving training.
\end{abstract}

\clearpage

\tableofcontents

\clearpage

\section{Introduction}
\label{sec:introduction}

We begin with a motivating example.
\subsection{Motivation}
\label{sec:motivation}
\emph{Given a rope of length $\budg$, what shape should one bend it
  into to best ``fill'' $\cset?$} In simple cases it is easy to
distinguish between ``good'' and ``bad'' solutions. The idea is that a
``good'' solution should minimize the average distance from a point in
$\cset$ to the closest point on the rope. With this in mind, one can
clearly see the curve in \cref{fig:good-fill} performs better than the
one in \cref{fig:bad-fill} (note the sample points are identical).
\begin{figure}[H]
  \centering
  \begin{subfigure}{.49\linewidth}
    \centering
    \includegraphics{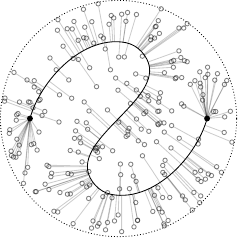}
    \caption{A good solution.}
    \label{fig:good-fill}
  \end{subfigure}
  \begin{subfigure}{.49\linewidth}
    \centering
    \includegraphics{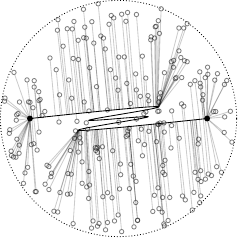}
    \caption{A bad solution.}
    \label{fig:bad-fill}
  \end{subfigure}
  \caption{Closest points for $N=250$ uniformly-sampled $\cpt \in
    \cset$. }
\end{figure}
Explicitly: Parametrizing the curves by $\fgood$, $\fbad : [0,1] \to
\cset$, we have
\begin{equation}
  \EE\bk{\inf_{\dpt \in [0,1]} \abs{\cpt - \fgood(\dpt)}}  <
  \frac{1}{2} \EE\bk{\inf_{\dpt \in [0,1]} \abs{\cpt - \fbad(\dpt)}}.
  \label{eq:expectations}
\end{equation}
If we endow $\cset$ with a measure $\cmeas$ more concentrated along
the horizontal we find the ``bad'' solution now performs better than
the ``good'' one. Thus we may generalize \eqref{eq:expectations} by
replacing the expectations with $\cmeas$-expectations. As a real-world
example, this could model a courier service wanting to find an
efficient delivery route for clients distributed as $\cmeas$ in a city
$\cset$.
\begin{figure}[H]
  \centering
  \begin{subfigure}{.49\linewidth}
    \centering
    \includegraphics{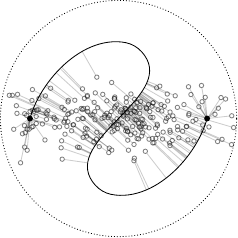}
    \caption{The new ``bad'' solution.}
    \label{fig:good-now-bad-curve}
  \end{subfigure}
  \begin{subfigure}{.49\linewidth}
    \centering
    \includegraphics{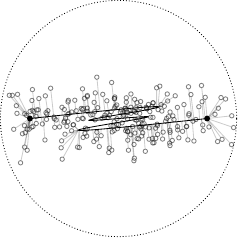}
    \caption{The new ``good'' solution.}
    \label{fig:bad-now-good-curve}
  \end{subfigure}
  \caption{Closest point projections for $N=250$ points sampled
    according to $\cmeas = \mc N(\mb 0; [\sigma_X^2\ 0;\ 0\
    \sigma_Y^2])$ (truncated to $\cset$), where $\sigma_X^2 = 1$,
    $\sigma_Y^2=9/121$.}
  \label{fig:new-target-measure}
\end{figure}
This problem is closely related to the \emph{average-distance problem}
introduced in \cite{buttazzo2002,buttazzo2003}; those works considered
optimizing the expectations in \eqref{eq:expectations} over closed,
connected sets $\Sigma \subseteq \cset$ with Hausdorff 1-measure $\mc
H^1(\Sigma) \leq \budg$. In a similar vein, we might try formalizing
our problem as an optimization over $\fmap \in C([0,1]; \cset)$ with
$\haus^1(\fmap([0,1])) \leq \budg$. However, as we now discuss, this
has some disadvantages.

Note \eqref{eq:expectations} is equally meaningful in the context of
trying to fill a $\cmeas$ on $\cset \subseteq \rcdim$ with some kind
of $\ddim$-surface. For example, this could represent the problem of
trying to design an optimal catalytic surface for reagents in a vessel
$\cset \subseteq \RR^3$, or a manifold learning problem where data
noise causes $\cmeas$ to ``spread'' off of the generating
$m$-manifold. Regardless, when $\ddim > 1$, the $\ddim$-dimensional
Hausdorff measure $\haus^\ddim(\fmap(\dset))$ is no longer a
meaningful constraint, since one may obtain arbitrarily-good
approximations of $\cmeas$ for which $\haus^\ddim(\fmap) = 0$ by
degenerating $\fmap(\dset)$ to an $(\ddim-1)$-surface. Thus in our
problem we will replace $\haus^\ddim$ with a constraint $\cost$ that
at least includes the derivatives $\partial_i \fmap^j$, where
$\fmap^j$ denotes the $j$\textsuperscript{th} component function of
$\fmap$. Our particular choice will be the Sobolev norm we discuss
later in \cref{sec:the-constraint}; for now, a simpler example is
\[
  \cost(\fmap) = \pn{\sum \norm[]{\partial_i
  \fmap^j}_{L^\sobp(\dset)}^\sobp}^{1/\sobp}.
\]

Thinking again of the courier-service example, we see it is also quite
natural to add higher-order terms to $\cost$: Delivery routes that
contain frequent acceleration might perform poorly in terms of fuel
economy, thus incurring extra costs. The particular constraint we use
in this paper shall be the Sobolev norm $\sobnorm{\fmap}$, with which
we can choose to penalize (weak) derivatives up to order $\sobk$.

Lastly, we should also notice that it is meaningful to consider
exponents $p$ other than $p=1$ on the distance function in
\eqref{eq:expectations}. In analogy with $L^\objp$ norms, we may
interpret this as weighting the relative importance of ``outliers'' in
$\cmeas$. Hence we end up with the fairly-general problem of trying to
minimize the objective function
\[
  \fmap \mapsto \EE_\cmeas\bk{\inf_{\dpt \in \dset} \abs{\cpt -
      \fmap(\dpt)}^\objp} \hbox{ subject to a Sobolev constraint of
    $\fmap$.}
\]
With this motivation in mind we now formally define the problem in a
slightly more general context, so as to facilitate comparisons to
existing works. Note, from now on we will write the expectation
$\EE_\cmeas$ as an integral.

\subsection{The General Problem}
\label{sec:the-general-problem}

Consider the following class of problems:
\begin{align*}
  \hbox{Fix $\ddim < \cdim$ and
  let $\dset \subseteq \rddim$ and $\cset \subseteq \rcdim$.}
\end{align*}
Also fix a class of admissible functions $\admiss \subseteq \set{\fmap
  \MID \fmap : \dset \to \cset}$. Let $\cmeas$ be a probability
measure on $\cset$ and for some $\objp \in \pn{0,\infty}$ define the
objective function $\objfp : \admiss \to \RR^{\geq 0}$ by
\begin{equation}
  \label{eqn:fp}
  \objfp(\fmap) = \int_\cset \inf_{\dpt \in \dset} \abs{\cpt -
    \fmap(\dpt)}^\objp \dd \cmeas(\cpt).
\end{equation}
Then, given a constraint function $\cost : \admiss \to \RR^{\geq 0}$
and a budget $\budg \geq 0$, the goal is
\begin{adjustwidth}{1em}{0em}
  \vspace{.25em}
  \begin{leftbar} \vspace{-.75em}
    \begin{problem}[Hard Constraint (HC)]
      \xlabel[(HC)]{prob:hard-constraint}
      Minimize $\objfp(\fmap)$ over $\set{\fmap \in \admiss \MID
        \cost(\fmap) \leq \budg}$.
    \end{problem} \vspace{-.25em}
  \end{leftbar}
  \vspace{-.375em}
\end{adjustwidth}
We define the \emph{optimal-value function} for a given budget $\budg$
by
\begin{equation}\label{eqn:J-l}
  \objell(\budg) = \inf_{\cflb} \objfp(\fmap).
\end{equation}
The typical constraints considered are usually either integral
constraints, \eg\ $\cost(\fmap) = \arclen(\fmap) = \int \abs{\fmap'}
\dd \dpt$ \cite{kegl2000}, or geometric constraints, \eg\
$\cost(\fmap) = \haus^1(\fmap(\dset))$ where $\haus^1$ denotes the
Hausdorff 1-measure \cite{buttazzo2002,buttazzo2003}. Much work has
been done on these problems and the literature is quite diverse so we
will postpone a detailed discussion until \cref{sec:relationships}.
For now, we simply note that frameworks like these provide a way to
study manifold learning problems where data noise is \emph{explicitly}
assumed---see \cite[\S 3.5.2]{Meila2024Apr}, and the references
therein---as well as certain generative learning problems
(\cref{sec:unequal-dimensional-ot}).

As succinctly stated in \cite{kirov-slepcev-2017}, one may think of
$\objfp$ as measuring how well $\fmap$ approximates $\cmeas$ while
$\cost(\fmap)$ restricts the complexity of the approximation. In
\cref{prop:connection-to-ot} we show $\objfp$ has an interpretation in
terms of an optimal transport cost. It is also common
\cite{lu-slepcev-2013, lu-slepcev-2016, kirov-slepcev-2017} to replace
the hard constraint $\cost(\fmap) \leq \budg$ with a soft penalty by
defining for a fixed parameter $\lambda > 0$
\begin{align}\label{eqn:fp-lambda}
  \objfpl(\fmap) = \objfp(\fmap) + \lambda \cost(\fmap),
\end{align}
whence the problem becomes
\begin{adjustwidth}{1em}{0em}
  \vspace{.25em}
  \begin{leftbar} \vspace{-.75em}
    \begin{problem}[Soft Penalty (SP)]
      \xlabel[(SP)]{prob:soft-penalty}
      Minimize $\objfpl$ over $\set{\cost < \infty}$.
    \end{problem} \vspace{-.25em}
  \end{leftbar}
  \vspace{-.375em}
\end{adjustwidth}
\xref{prob:hard-constraint} and \xref{prob:soft-penalty} are very
closely related but not always equivalent for every choice of $\cost$;
see \cref{sec:comparison-to-soft-penalty}. For now, we simply note
that solutions to \xref{prob:soft-penalty} recover solutions to
\xref{prob:hard-constraint} but not necessarily vice versa, and also
that \xref{prob:soft-penalty} is often more direct to implement
numerically. Thus we will typically use \xref{prob:soft-penalty} in
lieu of \xref{prob:hard-constraint} in our simulations.

\subsection{Optimal Transport Interpretation}
\label{sec:ot-interpretation}
We can interpret \xref{prob:hard-constraint} in terms of a
``free-source'' version of the unequal-dimensional optimal transport
studied in \cite{McCann2020Dec,nenna_pass_2020}. Recall that given
compact domains $\uneqds \subseteq \RR^\uneqdd, \uneqis \subseteq
\RR^\uneqid$ (for now making no assumptions about the relative sizes
of $\uneqdd$, $\uneqid$), probability measures $\uneqdm \in
\pmeas(\uneqds)$, $\uneqim \in \pmeas(\uneqis)$, and a cost function
$c : \uneqds \times \uneqis \to \RR$, the \emph{optimal transport
  cost} between $\uneqdm$ and $\uneqim$ is given by \[ \otcost
  (\uneqdm,\uneqim) = \inf_{\uneqtp \in \uneqtps(\uneqdm,\uneqim)}
  \int_{\uneqds \times \uneqis} c(\uneqdp,\uneqip) \dd
  \uneqtp(\uneqdp,\uneqip), \] where $\uneqtps(\uneqdm, \uneqim)$
denotes the set of all probability measures on $\uneqds \times
\uneqis$ with marginals $\uneqdm$ and $\uneqim$. When $\uneqdd =
\uneqid$ and $c(\uneqdp,\uneqip) = \abs{\uneqdp - \uneqip}^p$, we
let \begin{equation} \label{eqn:Wassp} \wassp(\uneqdm,\uneqim) =
  (\otcost(\mu,\nu))^{1/p} \end{equation} and call it the
\emph{Monge-Kantorovich $p$-distance} (also called the
\emph{Wasserstein} $p$-distance); see \eg\ the books
\cite{villanitopics,Santambrogio2015,Peyre2018Mar}. It is
straightforward to show (see \cref{prop:connection-to-ot}) that
\begin{align}\label{eqn:J-OT}
  \objell(\budg) = \inf_{\cflb} \inf_{\imeas
  \in \pmeas(\fmap(\dset))} \wpp(\cmeas, \imeas).
\end{align}

\subsection{Hypotheses}
\label{sec:assumptions}
Having discussed the general problem, we now specify the particular
instance we consider. The core hypotheses we employ are \medskip
\begin{hyps}[leftmargin=4em]
\item $\dset \subseteq \rddim$ is a simply connected, compact domain
  with nonempty interior and $\partial \dset \in
  C^1$, \label{hyp:dset}
\item $\cset \subseteq \rcdim$ $(\ddim < \cdim)$ is a compact, convex
  domain, \label{hyp:cset}
\item $\cost(\fmap)$ is given by a $\wkp$ Sobolev norm
  $(\sobk \geq 1)$ where \label{hyp:sob}
  \begin{enumerate}[leftmargin=2em, label=(H.3.\alph*)]
    \item $1 < \sobp < \infty$, and \label{hyp:sobp}
    \item $\sobk \sobp > \ddim$, and \label{hyp:kqm}
  \end{enumerate}
\item $\objp \geq 1$. \label{hyp:p}
\end{hyps}
\medskip

Broadly speaking, \ref{hyp:dset}--\ref{hyp:sob} are used for getting
things like compactness and convexity of the feasible set and suitable
continuity of $\objfp$, while \ref{hyp:p} is used to give $\objfp$ enough
regularity to get a ``gradient'' in \cref{sec:local-strategies}. On
that note, while it is not a core hypothesis, the following addition
of \ref{hyp:ac} often allows one to dramatically simplify statements in
\cref{sec:local-strategies}; we will note this explicitly when
relevant: \medskip
\begin{hypsopt}
\item $\cmeas$ is absolutely continuous with respect to the Lebesgue
  measure. \label{hyp:ac}
\end{hypsopt}
\medskip The primary contribution of our work is the generality of
\ref{hyp:dset}--\ref{hyp:p}, especially the free choice of $\ddim <
\cdim$. A particular challenge is posed by the fact that taking $\sobk
> 1$ includes higher-order terms in $\cost(\fmap)$; this makes $\cost$
severely parametrization-dependent. As a consequence, typical
local-modification arguments (see \cite[Lemma 3.12]{buttazzo2002},
\cite[Lemma 3.1]{delattre2020}) for improving $\objfp$ become
unusable; we discuss this more in \cref{sec:arclen-constraints}. For
now, we simply note that the higher-order terms fundamentally change
the flavor of the problem. We are forced to take a more abstract
approach than in the $\cost(\fmap) = \arclen(\fmap)$ case and go on
some detours to set up machinery. On the other hand, our proof
strategies tend to be more general, and so might arguably expose more
fundamental geometric features of the problem.

\subsection{Discussion of Main Results}

Under \ref{hyp:dset}, \ref{hyp:cset}, and \ref{hyp:sob} we establish
existence of optimizers (\cref{thm:existence}) and discuss some
trivial counterexamples to uniqueness (\cref{cex:nonuniqueness}). We
exhibit fundamental properties of $\objfp$ such as joint continuity in
$(\fmap, \cmeas)$ (\cref{thm:joint-continuity}) and lack of
convexity/concavity (\cref{sec:non-concavity-non-convexity}).
Similarly, for the optimal-value function $\objell(\budg)$ we
establish some basic properties including continuity
(\cref{thm:objell-properties}; \cf\ \cite[Lem. 3.1.i]{delattre2020})
and an asymptotic lower bound in terms of the budget $\budg$
(\cref{prop:asymptotics}; \cf\ \cite[Thm 3.16]{buttazzo2002}).

A natural question is whether $\objell$ must be strictly monotone. In
other words,
\begin{equation}\label{eqn:strict-question}
  \hbox{\itshape ``Does increasing the budget allow us to create
  strictly-better solutions?''}
\end{equation}
This is surprisingly nontrivial. A closely related result we can show
via elementary methods is that for all $\budg$ large enough to
guarantee optimizers are nonconstant
(\cref{prop:nontrivial-optimizers-nonconstant}), there exist
optimizers saturating the constraint
(\cref{cor:saturating-optimizers}). By contrast, the converse to
\eqref{eqn:strict-question} is immediate: Decreasing $\objell$ always
requires increasing the budget, since $\objell$ is trivially
nonincreasing.

The difficulty in getting a strict monotonicity result
\eqref{eqn:strict-question} is mainly due to the significant
parametrization-dependence of $\cost$ when $\sobk > 1$. Hence much of
the remainder of the paper is dedicated to establishing local
modification strategies that are compatible with $\cost$ and provably
improve $\objfp$. Fundamental to this goal is the question:
\begin{center}
  \itshape ``What is the gradient of $\objfp(\fmap)$?''
\end{center}
To that end we consider a natural object, the first variation of
$\objfp$. Versions of this have been considered in the literature (see
\eg\ \cite{buttazzo-mainini-stepanov2009, kirov-slepcev-2017,
  delattre2020}); the point is that we treat the case of general
$\objp \geq 1$, and importantly remove common hypotheses on $\cmeas$
(\eg\ ``vanishing on small sets'') for $\objp > 1$ while also
establishing sharper hypotheses for the $\objp = 1$ case, thus
unifying the results scattered in the literature. On the other hand,
we also use this to provide a novel and nontrivial result on making a
local modification to an optimizer while remaining in the Sobolev
class.

The approach is as follows. First, in \cref{sec:local-strategies},
under \ref{hyp:p} we express the first variation of $\objfp$ as a
certain vector field attached to the map $\fmap$; we refer to this
field as the \emph{barycenter field}.
\begin{repdefinition}{def:fbary}[Barycenter Field]
  Let $\pfd(\cpt) \in \argmin_{\dpt \in \dset} \abs{\cpt -
    \fmap(\dpt)}$ and let $(\set{\cmdpt}_{\dpt \in \dset}, \cpushdh)$
  be the disintegration of $\cmeas$ via $\pfd$. Then we define the
  \emph{barycenter field} at $\dpt$ to be
  \[
    \fbdh(\dpt) = \objp \int_{\pfd^{-1}(\dpt)} \abs{\cpt -
      \fmap(\dpt)}^{\objp - 2} (\cpt - \fmap(\dpt)) \dd \cmdpt(\cpt).
  \]
  Letting $\iset = \fmap(\dset)$ and $\pfi(\cpt) \in \argmin_{\ipt \in
    \iset} \abs{\cpt - \ipt}$ we may define $\fbih$ analogously.
\end{repdefinition}
In the case $(\ddim, \objp) = (1,2)$ this trivially recovers the
objects used in \cite{kirov-slepcev-2017,delattre2020}; see also
similar ideas in
\cite{gerber_dimensionality_2009,buttazzo-mainini-stepanov2009}. Note
that in this definition there are some technical wrinkles arising from
the fact that there might be multiple possible choices for the value
$\pfd$ can take for $\cpt$ in the so-called \emph{ambiguity set}
$\ambf$; typically these are addressed by making assumptions on
$\cmeas$ \cite{buttazzo-mainini-stepanov2009, kirov-slepcev-2017} or
by exploiting structure particular to the case $\objp = 2$
\cite{delattre2020}. We bypass such issues by applying a simple result
on measurable selections; see \cref{sec:disintegration-details}.

Let $\pfdset$ denote the set of all measurable selections of $\cpt
\mapsto \argmin_{\dpt \in \dset} \abs{\cpt - \fmap(\dpt)}$ and for all
$\pfd \in \pfdset$ let $\cpushdh$ denote the pushforward of $\cmeas$
by the map $\pfd$. It turns out that the barycenter field is
essentially the $L^2(\cpushdh)$-gradient of $\objfp$ at $\fmap$:
\begin{reptheorem}{thm:fbary-grad-J}
  Take \ref{hyp:dset}, \ref{hyp:cset}, and \ref{hyp:p}. Fix $\fmap \in
  C(\dset; \cset)$ and $\pert \in C(\dset; \rcdim)$, and for all
  $\varepsilon > 0$ let $\feps \coloneqq \fmap + \varepsilon \pert$
  and $\iset = \fmap (\dset)$, $\iseps=\feps(\dset)$. If $\objp = 1$,
  let $\iset_{\rm out} = \set{\ipt \in \iset \MID \liminf_{\varepsilon
      \to 0} 1_{\iseps}(\ipt) = 0}$, and take the extra hypothesis
  $\cmeas(\iset_{\rm out}) = 0$. Then
  \begin{equation*}
    \lim_{\varepsilon \to 0}
    \frac{\objfp(\feps) - \objfp(\fmap)}{\varepsilon} =
    \min_{\pfd \in \pfdset} \int_\dset
    \ip[]{-\fbdh(\dpt), \pert(\dpt)} \dd
    \cpushdh(\dpt).
  \end{equation*}
\end{reptheorem}
Here the minimization of the set $\pfdset$ arises due to the ambiguity
set $\ambf$.
\begin{boxedremark}
  An analogous result for the $(\ddim, \cost(\Sigma), \objp) = (1,
  \haus^1(\Sigma), 1)$ case was established in
  \cite{buttazzo-mainini-stepanov2009} under the assumption that
  $\pert \in C^\infty$ and $\cmeas(E) = 0$ for all $E$ such that
  $\haus^{\cdim - 1}(E) = 0$. Similarly, if one takes \ref{hyp:ac},
  $\fmap, \pert \in C^2$, and $(\ddim, \cost(\fmap), \objp) = (1,
  \arclen(\fmap), 2)$ then \cref{thm:fbary-grad-J} is implied by the
  analysis in \cite{kirov-slepcev-2017}. In \cite{delattre2020} it was
  further shown that if $\fmap$ is an optimizer of
  \xref{prob:hard-constraint} and $\pert$ is bounded and Borel, then
  exploiting an invariant particular to $\objp = 2$ allows removing
  the \ref{hyp:ac} hypothesis of \cite{kirov-slepcev-2017}. We
  reiterate that our hypotheses are more general, requiring only
  continuity of $\fmap, \pert$, and a mild extra hypothesis on
  $\cmeas$ for $\objp = 1$.
\end{boxedremark}
When $\cmeas(\ambf) = 0$ (\eg, via \ref{hyp:ac}), adding an
injectivity hypothesis to $\fmap$ and rewriting in terms of $\fbih$
yields a much simpler form---see \cref{cor:fbary-grad-J-y}. However,
even with this simplified form, it is nontrivial to construct local
improvements to $\fmap$ while remaining in the Sobolev class. Indeed,
in \cref{sec:barycenter-field-and-properties} we study some properties
of $\fbdh$, $\fbih$, and demonstrate that they can be discontinuous
even when $\fmap \in C^\infty$ (\cref{cex:f-smooth-fbary-discont}) or
$\iset$ is a $C^1$-manifold (\cref{cex:Y-c1-fbary-discont}). This
complicates attempts to employ $\fbdh$, $\fbih$ for gradient flow
(\cf\ \cite[\S 5]{gerber_dimensionality_2009}), as such a scheme
would require $\fmap + \varepsilon \fbdh$ to remain within the
admissible set, that is, the Sobolev class for $k\ge 1$. Nevertheless,
we show that if $\fbdh$ is nontrivial we can construct a local
improvement regardless of the choice of $\sobk$:
\begin{reptheorem}{thm:smooth-improvement}
  \label{thm:smooth-improvement-fake}
  Take the hypotheses of \cref{thm:fbary-grad-J} and suppose that
  \begin{align}\label{eqn:nontrivial-F}
  \hbox{for some $\pfd \in \pfdset$ we have $\cpushdh(\set[]{\fbdh
    \neq 0}) > 0$.}
  \end{align}
  Then for all $\varepsilon > 0$ there exists $\perteps \in
  C^\infty(\dset; \rcdim)$, depending on $\fbdh$, such that
  \[
    \objfp(\fmap + \perteps) < \objfp(\fmap) \qquad \text{and} \qquad
    \cost(\fmap + \perteps) < \cost(\fmap) + \varepsilon .
  \]
  Furthermore, if $\cpushdh(\finv(\partial \cset)) = 0$, then
  $\perteps$ can be chosen so that $\fmap + \perteps$ takes values
  only in $\cset$, whence $\fmap$ is not a local minimum of $\objfp$
  in $\wkpsets$.
\end{reptheorem}
This can of course be restated in terms of $\iset$, $\fbih$; see
\cref{cor:fbary-grad-J-y}. In any case, this theorem almost
establishes the desired ``strict monotonicity''
\eqref{eqn:strict-question} of $\objell$, modulo the nontriviality
condition $\cpushdh(\set[]{\fbdh \neq 0}) > 0$; such nontriviality was
called ``default of self-consistency'' in \cite[Lemma
3.2]{delattre2020} where it was established for the special case
$(\ddim, \cost(\fmap), \objp) = (1, \arclen(\fmap), 2)$. The
assumption $\cpushdh(\set[]{\fbdh \neq 0}) > 0$ is crucial for us, and
we conjecture that it holds for each optimizer $f$ of the hard
constraint problem $\objell(\budg)$ \eqref{eqn:J-l}, though
unfortunately the proof of \cite{delattre2020} does not generalize to
our case, with challenges being provided by both $\sobk > 1$ and
$\sobp > 1$.

In the later part of the paper we turn toward numerical simulations
and applications to machine learning. In
\cref{sec:the-discrete-problem} we demonstrate consistency results for
both \xref{prob:hard-constraint} and \xref{prob:soft-penalty}
(\cref{cor:consistency,cor:consistency-sp}). We demonstrate an
algorithm for computing discrete approximations to $\fbdh$ and show an
example output. In \cref{sec:generative-learning} we highlight a
connection between \xref{prob:hard-constraint} and generative learning
problems. We give an informal overview of why \emph{a priori} we might
expect such a connection (summarized in
\cref{sec:unequal-dimensional-ot}) and give numerical experiments
demonstrating that the arc length constraint (\cf\
\cite{lu-slepcev-2016}, \cite{delattre2020}) is a powerful
regularization method in a simple image-generation problem, clearly
outperforming an unregularized network as well as modestly
outperforming weight-decay. Further experiments would be required to
ensure this is not an artifact of lucky hyperparameter tuning, as well
as to understand the impact of our constraint when $\sobk > 1$. This
is left for a future work.

\subsection{Example Applications}
\label{sec:applications}
We believe there are many practical directions that our results and
perspectives can find new applications to. Here we list two examples.
\subsubsection{Factorization of Wasserstein GANs}
\label{sec:unequal-dimensional-ot}

Given the relationship between $\objfp$ and $\wpp(\cmeas, \imeas)$
(see \cref{sec:ot-interpretation} and \cref{prop:connection-to-ot}),
our results can be connected to ideas in machine learning research,
especially Wasserstein GANs \cite{arjovsky2017}. First, one may
trivially extend $\wassp(\uneqdm, \uneqim)$ to the case $\uneqdd <
\uneqid$ by fixing a measurable $\uneqf : \uneqds \to \uneqis$,
letting $c(\uneqdp, \uneqip) = \abs{\uneqf(\uneqdp) - \uneqip}^p$, and
taking the infimum over $\uneqtp \in \uneqtps(\uneqf_{\# \uneqdm},
\uneqim)$. Then the essential idea of {\em Wasserstein GANs (WGANs)}
is the problem of learning the best choice of $\uneqf$ to minimize
$\wassp(\uneqf_{\#\uneqdm}, \uneqim)$.

Our problem is slightly different. From the double optimization
structure \eqref{eqn:J-OT} we may think of \xref{prob:hard-constraint}
as trying to find the ``best'' $\fmap \in \cflb$ when we are free to
choose \emph{any} probability measure to put on $\fmap(\dset)$, rather
than just $\uneqf_{\# \uneqdm}$ for the given $\uneqdm$. Here we may
interpret the restriction to $\cflb$ as penalizing overfitting of an
empirical measure $\cmeas_N$ in the presence of high-frequency noise,
with increasing $\sobk$ yielding solutions $\fmap$ that have fewer
``sharp bends.'' Additionally, we may view the $\sobk = 1$ term as
penalizing ``backtracking,'' thus incentivizing choices of $\fmap$ for
which locality in $\cset$ corresponds more nicely to locality in
$\dset$. This may be used to ``factor'' problems like WGAN (between
$\uneqdm$ and $\cmeas$) into two steps.

\vspace{.25em}
\begin{adjustwidth}{.75em}{0em}
  \begin{leftbar}
    \vspace{-.25em}
    \emph{Factorization of WGAN (A Two-Step Approach):}
    \begin{enumerate}[label=\arabic*)]
      \item A ``hard'' step \xref{prob:hard-constraint} consisting of
        finding a low-complexity $\ddim$-to-$\cdim$ map $\fmap$
        whose image best approximates a measure $\cmeas$ on $\rcdim$,
        and
      \item An ``easy'' step (reparametrization) consisting of finding an
        $\ddim$-to-$\ddim$ map $\varphi$ such that $(\fmap \circ
        \varphi)_{\# \uneqdm} = \argmin_{\imeas \in \pmeas(\fmap(\dset))}
        \wpp(\cmeas, \imeas)$.
    \end{enumerate}
    \vspace{-.75em}
  \end{leftbar}
\end{adjustwidth}
We hypothesize that performing this sort of factorization could yield
performance improvements over the existing methods, which generally
try to solve the two problems simultaneously.

Essentially, the idea is that in the factored problem, the explicit
regularization of $\fmap$ can help avoid situations where $\fmap$ gets
stuck in a well whose local minimum has extremely low regularity
(\eg\ being not even continuous). Such low-regularity functions can
be difficult to express in a given model class, leading to undesirable
qualities such as slow convergence or training instabilities.

In fact, we propose that even in the ``unfactored'' context, adding
regularization resembling $\cost$ could provide important
benefits---see \cite{Vardanyan2024Jul} for a similar idea. We discuss
this and more connections to generative learning in
\cref{sec:generative-learning}; for now, we simply propose that
theoretical analysis of
\xref{prob:hard-constraint}/\xref{prob:soft-penalty} could be useful
in developing qualitative intuition for problems like WGAN.

\subsubsection{A Routing Problem with Sobolev constraint}
\label{sec:routing-problems}
The following application gives an idea of why it is relevant to
consider for $\cost(f)$ the Sobolev norm of $f$, \ie\ including all
the lower order terms.

In \cref{sec:motivation} we mentioned the example of a company trying
to find an efficient route for delivering packages. We describe a
similar drone-routing application involving $W^{2,\sobp}$, inspired by
\cite[\S 7.2.1]{lebrat2019}. Note, however, that the authors of that
work considered $L^\infty$ constraints on the linear and angular
velocities of the drone, and also required it to return to the start
at predetermined times for charging. This is in contrast to our
constraint, which is $L^\sobp$ and directly incorporates the
acceleration ($\sobk = 2$) term, thus allowing the optimization
process itself to determine when the drone should return for charging.

Suppose an organization plans to fly drones over remote,
wildfire-prone areas to attempt to locate new burns. Based on past
data and weather modeling they have a rough idea of the distribution
$\cmeas$ where the risk of fire is the highest. What is the best path
$\fmap(t)$ for the drone to take, given its finite battery capacity?

There are three main factors. First, as the drone flies farther from
the launch station, it requires more power to transmit a strong enough
signal to stream data back; we encode this with $\norm{\fmap}_{L^{1}}$
(alternatively, one could interpret the $0$\textsuperscript{th}-order
term as penalizing the expected cost of retrieval were the drone to
malfunction and crash in the wilderness, as typically the
transmission-power variation is small). Second, assuming low speeds,
the drone must constantly expend energy proportional to its velocity
in order to overcome drag; we encode this with
$\norm[]{D\fmap}_{L^{1}}$. Lastly, the drone will need to expend
energy any time it accelerates; we encode this by
$\norm[]{D^2\fmap}_{L^{1}}$. Hence, up to a choice of coefficients
$c_0, c_1, c_2 > 0$ our constraint becomes
\[
  \norm{\fmap}_{W^{2,1}([0,1]; \cset)} = c_0 \norm{\fmap}_{L^1} +
  c_1 \norm{D \fmap}_{L^1} + c_2 \norm{D^2 \fmap}_{L^1}.
\]
Thus we see the Sobolev norm of $\fmap$ (the sum of lower and higher
order terms), especially in this case with $k=2$, is a meaningful
constraint. We may interpret the replacement of $\norm{\cdot}_{L^1}$
with $\norm{\cdot}_{L^\sobp}^\sobp$ as a form of reweighting to
penalize outliers.

For this particular application, we believe a discretized version of
the barycenter field $\fbdh$ (\cref{def:discrete-barycenter-field})
could be particularly helpful if $\cmeas$ is being approximated via
data collected in real-time. For example, if satellite data shows that
a region far from the current flight path is experiencing severe dry
conditions, $\fbdh$ could inform the operators how to most-efficiently
modify the drone path to investigate the new area without having to
recompute the entire trajectory from scratch.

\subsection{Relationships with Existing Work}
\label{sec:relationships}

\subsubsection{Non-parametric vs.\ Parametric Formulation}
When $\ddim=1$, \xref{prob:hard-constraint} is closely related to the
``average distance problem'' or ``irrigation problem'' first
introduced in \cite{buttazzo2002,buttazzo2003} (a review is given in
\cite{lemenant2012}). That work considered the particular case $\objp
= 1$ and optimized over closed, connected sets $\Sigma \subseteq
\rcdim$ with finite Hausdorff $\haus^1$ measure. We may view this as
\xref{prob:hard-constraint} taking $\admiss = C(\dset; \cset)$ and
$\cost(\fmap) = \haus^1(\fmap(\dset))$.

A key distinguishing feature of the $\haus^1$ formulation is that it
makes it easy to make changes to the topology of $\fmap(\dset)$: We
may freely create ``branches'' in $\Sigma$, so long as it remains
connected and we maintain $\haus^1(\Sigma) \leq \budg$. By contrast,
in the parametrization-dependent formulation of the cost $\cost$ these
kinds of modifications are disincentivized: Creating a branch
typically requires some form of ``backtracking,'' which is usually
highly inefficient. In fact in \cite{delattre2020} it was actually
shown for the parametric formulation $(\cost(\fmap), \cdim, \objp) =
(\arclen(\fmap), 2, 2)$, that optimizers of
\xref{prob:hard-constraint} \emph{must} be injective, thus precluding
backtracking of any kind.

We may interpret this as a reflection that the $\haus^1$ formulation
is best suited for problems with some form of parallel structure,
\eg\ designing an irrigation network: Once water flows in the pipes
we can deliver it to all customers simultaneously. By contrast, the
parametrization-dependent formulation is suited more to problems with
a sequential structure, such as the courier problem: A delivery truck
can only visit one address at a time.

\subsubsection{Principal Curves and Surfaces}
\label{sec:principal-curves}
In \cite{hastie1984} Hastie introduced the notion of \emph{principal
  curves and surfaces}, defined (in our notation) to be injective maps
$\fmap$ for which the $\objp = 2$ barycenter field $\fbih \equiv 0$. A
concise reference for this framework is found in \cite{hastie1989}. In
analogy with the standard technique of \emph{principal component
  analysis}, we may understand principal curves as essentially
performing nonlinear dimension reduction. However, this definition
proved to be difficult to work with theoretically, hence the authors
of \cite{kegl2000} chose to redefine principal curves as (in our
notation) solutions of \xref{prob:hard-constraint} when $\ddim=1$,
$\objp=2$, and $\cost(\fmap) = \arclen(\fmap)$. This is typically
taken as the standard definition now \cite{delattre2020}. Clearly,
this problem is very similar to our case $\cost(\fmap) =
\sobnorm{\fmap}$, but there are some key differences as we now
discuss.

\subsubsection{The Arc Length Constraint}
\label{sec:arclen-constraints}
When $\ddim=1$, $\sobk = 1$, our problem becomes similar to the case
where $\cost(\fmap) = \arclen(\fmap)$. Indeed, since $\abs{\fmap}$ is
bounded by $\diam(\cset)$ while $\abs{D\fmap}$ is unbounded, we see
that as $\cost(\fmap) \to \infty$, $\norm{\fmap}_{W^{1,\sobp}} \approx
\norm{D\fmap}_{L^\sobp}$, which is similar in spirit to
$\arclen(\fmap)$. When $\sobk > 1$ however the problem becomes very
different.

To see this, suppose $\dset=[-1,1]$, $\cset = [-1,1]^2$ with $\cunif$,
and $\fmap(\dpt) = (\dpt, 0)$. Suppose we wanted to prove that for all
$\objp$, given $\varepsilon$ extra budget, we can modify $\fmap$ in a
way that improves $\objfp$. If $\sobk = 1$ this is easy: We may add a
``spike'' of height $O(\varepsilon)$ by reparametrizing $\fmap$ to be
constant on some $\bk{\dz - \delta, \dz + \delta}$ ($\delta \ll 1$)
and adding a perturbation (\cref{fig:post-spike}). Importantly, the
reparametrization step is trivial since we can freely introduce jump
discontinuities in $D \fmap$ at $\dz \pm \delta$. By contrast, the
cost of such a reparametrization is generally quantized when $\sobk >
1$ and $D \fmap(\dz) > 0$.

Thus if $\varepsilon$ is sufficiently small, inserting a
spike is no longer possible and the best we can do is typically
something like \cref{fig:smooth-spike}. This makes analyzing the
effects on $\objfp$ nontrivial since the image of the initial curve is
no longer a subset of the image of the perturbed curve. This is the
essential difference between the two problems.
\begin{figure}[H]
  \centering
  \begin{subfigure}{.32\linewidth}
    \centering
    \includegraphics{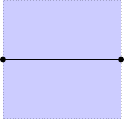}
    \caption{Pre-spike.}
    \label{fig:pre-spike}
  \end{subfigure}
  \begin{subfigure}{.32\linewidth}
    \centering
    \includegraphics{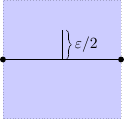}
    \caption{Post-spike.}
    \label{fig:post-spike}
  \end{subfigure}
  \begin{subfigure}{.32\linewidth}
    \centering
    \includegraphics{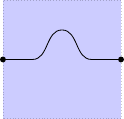}
    \caption{Smooth ``spike.''}
    \label{fig:smooth-spike}
  \end{subfigure}
  \vspace{-.5em}
  \caption{Adding two kinds of spikes.}
  \label{fig:adding-a-spike}
\end{figure}

\subsubsection{The Soft Penalty Formulation}
\label{sec:comparison-to-soft-penalty}
Earlier we mentioned the hard constraint and soft penalty (see
\cite{kirov-slepcev-2017,lu-slepcev-2021}) formulations are not
entirely equivalent. Let us briefly discuss why. First note that the
soft penalty problem \xref{prob:soft-penalty} recovers optimizers to
the hard constraint problem \xref{prob:hard-constraint}:
\vspace{.25em}
\begin{adjustwidth}{.75em}{0em}
  \begin{leftbar}
    \vspace{-1.5em}
    \begin{align*}
      \text{If }
      &\fhat \text{ is a minimizer of the soft penalty problem
        $\objfpl(\fmap)$, then}\\[-.25em]
      &\fhat \text{ is a minimizer of the hard penalty problem
        $\objfp(\fmap)$ over $\set[]{\cost(\fmap) \leq \hbudg
        \coloneqq \cost(\fhat)}$.}
    \end{align*}
    \vspace{-2em}
  \end{leftbar}
\end{adjustwidth}
We also have an easy monotonicity result:
\begin{lemma}\label{lem:monotone-c-lambda}
  For each $\lambda$ let $ \fmap_\lambda$ be an optimizer of
  $\objfp^\lambda$ of \xref{prob:soft-penalty}. Then,
  \begin{align*}
    \cost(\fmap_{\lambda_2})
    \le \cost(\fmap_{\lambda_1}) \hbox{ for all $\lambda_1 \le \lambda_2$.}
  \end{align*}
\end{lemma}
\begin{proof}
  Consider optimizers $\fmap_{\lambda_i}$, for $\lambda_i$, $i=1,2$,
  respectively. By optimality,
  \begin{align*}
    \objfp(\fmap_{\lambda_{1}}) + \lambda_1 \cost(\fmap_{\lambda_{1}})
    &\leq \objfp(\fmap_{\lambda_{2}}) + \lambda_1
      \cost(\fmap_{\lambda_{2}}), \text{ and} \\
    \objfp(\fmap_{\lambda_{2}}) + \lambda_2 \cost(\fmap_{\lambda_{2}})
    &\leq \objfp(\fmap_{\lambda_{1}}) + \lambda_2 \cost(\fmap_{\lambda_{1}}).
  \end{align*}
  Adding these and rearranging yields
  \begin{equation*}
    0 \leq (\lambda_2 - \lambda_1) (\cost(\fmap_{\lambda_1}) -
    \cost(\fmap_{\lambda_2})),
  \end{equation*}
  which completes the proof.
\end{proof}

Regarding whether \xref{prob:hard-constraint} solutions are
\xref{prob:soft-penalty} solutions, it should depend on the choice of
$\cost$. First of all, solutions to the soft penalty problem are
``maximally efficient'' in the sense that if there exist multiple
values of $\budg$ for which $\objell(\budg) = \objfp(\fhat)$, then
$\objfpl$ will select the one with the smallest $\budg$. A priori the
hard constraint problem has less restriction. In particular, it is not
true for general $\cost$ that optimizers for $\objfp$ over $\cflb$
saturate the constraint (\ie\ occur only on $\set{\cost(\fmap) =
  \budg}$), though this is expected for our choice of Sobolev cost
$\cost$. Second, even with such a result, one would need to show that
for every choice of budget $\budg$ (excluding two endpoint cases
corresponding to $\lambda=0$ and $\lambda=\infty$) there exists a
corresponding $\lambda_\budg$ such that a maximally-efficient
$\budg$-optimizer of \xref{prob:hard-constraint} is a
$\lambda_\budg$-optimizer of \xref{prob:soft-penalty}, so that
\xref{prob:hard-constraint} and \xref{prob:soft-penalty} would become
``equivalent'' in the sense of having identical solution sets. This
should be the case in our framework, though we were not able to prove
it. To the best of our knowledge, the only result in this direction is
a proof in \cite{delattre2020} that in the case $(\cost(\fmap), \ddim,
\objp) = (\arclen(\fmap), 1, 2)$ optimizers saturate the constraint,
with more general cases remaining open.

Anyways, the maximal-efficiency property of \xref{prob:soft-penalty}
is generally desirable in real-world implementations. After all, why
pay more if you could get the same performance for less? On the other
hand, \xref{prob:hard-constraint} is easier to analyze, as it makes
the relationship between $\budg$ and the optimal value
$\objell(\budg)$ more explicit. Indeed, in \xref{prob:soft-penalty} it
is not always obvious a priori how the choice of $\lambda$ will relate
to the constraint value of an optimal solution. Still, we can derive
any qualitative results for \xref{prob:soft-penalty} by analyzing
\xref{prob:hard-constraint} as the solutions to the former belong to
the latter. On the other hand, for the computation of optimizers we
can use \xref{prob:soft-penalty}.

To the last point, let us briefly discuss the advantages of employing
\xref{prob:soft-penalty} in numerical contexts. By moving
$\cost(\fmap)$ into the objective function, \xref{prob:soft-penalty}
reduces the problem to an unconstrained optimization. Thus we eschew
the need to explicitly parametrize $\set{\cost(\fmap) \leq \budg}$,
whose structure might be complex. Compare this with the analogous
encoding of \xref{prob:hard-constraint}: Defining the indicator
$\chi_\budg(\fmap)$ to be $0$ if $\cost(\fmap) \leq \budg$ and
$\infty$ otherwise, one may consider the augmented objective
$\objfp(\fmap) + \chi_\budg(\fmap)$. However, this functional is
discontinuous with respect to the metric induced by $\cost(\fmap)$,
thus preventing the use of any optimization methods that rely on
calculus techniques.

\subsubsection{Similar Sobolev Constraints}
Our work is similar to some work in the field of image analysis; in
particular, it is somewhat complementary to the frameworks of
\cite{chauffert2017, lebrat2019}. Note that both of these papers treat
only the case $\ddim = 1$.

In \cite{chauffert2017} a convolution kernel is added to their
analogue of $\objfp$, which has a nice interpretation in terms of the
image problem they consider: Whether given a collection of discrete
pixels comprising a computer screen or a set of discrete ink dots
printed on a sheet of paper, at large enough distances the human eye
perceives the image to be ``continuous'' rather than comprised of
discrete pieces.

They show some results that are analogous to ours (particularly
regarding existence and consistency), though their analysis is
primarily concerned with discretized schemes. More similarly to our
work, they also include some theoretical analysis of the case $\admiss
= \sobw^{1, \infty}(\dset; \cset)$ and touch briefly on the case of
$\admiss \subseteq \wkpsets$ when exact $L^\sobp$ bounds are placed on
the derivatives of each order (\eg\ $\norm{D^\mind \fmap}_{L^\sobp}
\leq C_\mind$, where $\mind$ is a multiindex and $C_\mind$ is a given
constant). By contrast, our work is concerned exclusively with the
case $\sobp \in (1,\infty)$, and only places bounds on the overall
norm $\sobnorm{\fmap}$.

The work of \cite{lebrat2019} uses a more direct analogue of our
$\objf_2$ and focuses particularly on developing numerical algorithms
and highlighting some less-obvious example applications. Their
algorithms are primarily concerned with the case where $\sobk \in
\set{1,2}$, and $\sobp = \infty$, and also the case where $\cost$ is
taken to be a geometric constraint like the arc length or total
curvature. Thus, to the best of our knowledge our treatment of $\ddim
> 1$ and $1 < \sobp < \infty$ is novel.

\subsubsection{Spline Smoothing \& Surface Reconstruction} \label{sec:spline-smoothing}
We now discuss some similar ideas to \xref{prob:soft-penalty} in the
context of using splines to fit observed data. We begin with
\emph{smoothing splines}, or rather, one of their generalizations to
$\ddim \geq 1$, so-called \emph{thin-plate splines}; see \eg\ \cite[\S
7.9]{greensilverman1993}.

In contravention of our standard notation, temporarily suppose $\cdim
= 1$, and suppose that we are given data $\set{(\dpt_i,
  \cpt_i)}_{i=1}^N \subseteq \dset \times \cset$ where the $\cpt_i$
are assumed to be generated via some process
\begin{equation}
  \label{eq:data-model}
  \cpt_i = \fmap(\dpt_i) + \varepsilon_i.
\end{equation}
Here, the model $\fmap$ is to-be-determined, and $\varepsilon_i$ is a
noise term representing error in the observation process; typically
$\varepsilon_i$ is assumed to have a particular form (\eg\ i.i.d.\
Gaussians) based on domain-specific knowledge. Then, given a
``roughness'' parameter $\lambda > 0$, so-called \emph{thin-plate
  splines} are minimizers of the variational problem
\begin{equation}
  \inf_{\fmap \in \sobw^{\sobk, 2}(\dset; \RR)} \sum_{i=1}^N
  \pn{\cpt_i - \fmap(\dpt_i)}^2 + \lambda \int_{\dset}
  \sum_{\abs{\mind} = k} \pn{D^{\mind} \fmap(\dpt)}^2 \dd
  \dpt. \label{eq:thin-plate-spline}
\end{equation}
Note, in some texts the term \emph{thin-plate spline} refers
particularly to the case $\sobk = 2$, with $\sobk > 2$ referred to as
\emph{higher-order thin-plate splines}.

The idea is that when the observation noise $\varepsilon_i$ is
nonnegligible, the penalty term in \eqref{eq:thin-plate-spline} helps
the solution ``smooth out'' errors and recover the underlying
generating process $\fmap$ in \eqref{eq:data-model}. In correspondence
with our hypothesis of $\sobk \sobp > \ddim$ \ref{hyp:kqm}, solutions
of \eqref{eq:thin-plate-spline} exist iff $\sobk \cdot 2 > \ddim$ as
well.\footnote{A fantastic overview of this (and smoothing splines
  more generally) can be found in online lecture notes by Ryan
  Tibshirani. \emph{Nonparametric Regression: Splines and RKHS
    Methods}, Advanced Topics in Statistical Learning, Spring 2024.
  \url{https://www.stat.berkeley.edu/~ryantibs/statlearn-s24/lectures/splines_rkhs.pdf}.}
To summarize, the problem is that as soon as we leave the
supercritical regime $\sobk \cdot 2 > \ddim$, the uniform convergence
we will see in \cref{cor:wconv-uconv} can fail spectacularly. Since
the data-fitting term in \eqref{eq:thin-plate-spline} is dependent
only on the behavior of $\fmap$ at individual points, it becomes
possible to perfectly interpolate $\set{(\dpt_i, \cpt_i)}$ while
incurring arbitrarily-small roughness penalty, hence the inf in
\eqref{eq:thin-plate-spline} is not achieved.

We highlight that \eqref{eq:thin-plate-spline} deals particularly with
scalar-valued $\fmap$. In \cite{Miller1987Jan}, the general $\cdim
\geq 1$ case was treated under some restrictions on the covariance of
the data noise; in \cite{Fessler1991Apr} a solution is presented for
the $\ddim = 1$, $\cdim \geq 1$ case allowing for general covariances.
In any case, while \eqref{eq:thin-plate-spline} has some similarities
with \xref{prob:soft-penalty} (particularly the case where $\objp =2$
and $\cmeas$ is taken to be an empirical measure on $N$ samples), they
remain distinct. First, while our Sobolev penalty $\cost(\fmap)$
disincentivizes complexity of all orders in $\fmap$, the roughness
penalty in \eqref{eq:thin-plate-spline} contains only the
highest-order weak derivative. This has a nice interpretation as a
squared projection distance from $\sobw^{\sobk, 2}(\dset; \RR)$ onto
the set of $\ddim$-variable polynomials variables of degree at most
$\sobk -1$; see \eg\ \cite[Ch.\ 2.8]{Wang2011} or \cite[Ch.\
4.2.1.7]{Yee2015}.

The consequence is that \emph{global} $(\sobk-1)$-degree structure in
thin-plate splines is ``free;'' for example, if $\sobk = 2$ and the
data $(\dpt_i, \cpt_i)$ are restricted to an $\ddim$-hyperplane in
$\RR^{\ddim + 1}$, then the optimizer in \eqref{eq:thin-plate-spline}
will perfectly interpolate the data regardless of the choice of
$\lambda$. This is in contrast to the case of our full Sobolev
penalty, which would yield solutions that are $\approx 0$ when
$\lambda \gg 1$ and slowly-expanding $\ddim$-manifolds approximating
$\mrm{ConvexHull}(\set{(\dpt_i, \cpt_i)}_{i=1}^N)$ (with a bias toward
the mean of the data) as $\lambda \decto 0$. Note in particular that
solutions of \xref{prob:soft-penalty} are incentivized to stop growing
soon after they have successfully fitted the data.

A second (and perhaps more important) distinction between the two
problems is that the fitting performed by a thin-plate spline is
\emph{informed ahead of time} of the proper choice of $\dpt_i$ for
each $\cpt_i$. No such information is available in
\xref{prob:soft-penalty}; as a consequence, we lose the convexity
structure of (regularized) least squares and must use different
analysis (indeed, $\objfp$ lacks both convexity and concavity over
$\wkpsets$; see \cref{sec:non-concavity-non-convexity}).

Still, there are some spline-based numerical methods (as well as many
non-spline-based ones; see \eg\ the reviews
\cite{Khanna2014Mar,You,Zhu2022Jun}) that aim to fit
\emph{unstructured} data sets. This problem is more similar to our
case, though we note that we are not aware of a comprehensive abstract
treatment comparable to our formulation. The literature is extensive
and difficult to summarize succinctly; we simply mention, for example,
\cite{Ma1998Sep,Yang2005Sep,Wang2006Apr,Zhao2011Jun,Laube2018,Lan2024Apr},
particularly highlighting the similarity of the penalty term in
\cite[Eq. 12]{Wang2006Apr} to our Sobolev penalty.

Our context is distinct from these works in a key way: Fundamentally,
we do not make any assumptions about the dimension of our target
measure $\cmeas$---it could be uniform, supported on a fractal, or
even a dirac---though we are particularly motivated by examples where
$\cmeas$ is of dimension strictly greater than $\ddim$. By contrast,
when the $\fmap$ in \eqref{eq:data-model} is at least $C^1$
(essentially a necessary assumption given that the goal is to
approximate $\fmap$ by objects with regularity), we see that the
analogous target measure in the spline-fitting problem will always be
at most $\ddim$-dimensional (\cf\ \cref{prop:effectiveness}), modulo
some diffusion introduced by the error term $\varepsilon_i$.

As a tangible example: In \cref{fig:voronoi-sims}, we display
simulations for our problem in which a curve grows to approximate the
uniform measure on $[-1,1]^2$. Note that because the target is
fundamentally 2-dimensional, the curve becomes highly furrowed and
fractal-esque, approaching a space-filling curve in the limit. By
contrast, for a spline-fitting algorithm, the analogous problem would
be to learn the 1-dimensional boundary set $\set{\abs{\cpt}_\infty =
  1}$.

\subsection{Acknowledgements}
The authors owe a significant debt to Andrew Warren for bringing our
attention to the existing literature on principal curves, as well as
for pointing us toward the measurable selection tools in
\cref{sec:disintegration-details}. We would also like to thank Pablo
Shmerkin for providing the concise arguments in
\cref{lem:lebesgue-shrinking} and \cref{cor:radon-shrinking-bound},
and Rentian Yao for pointing out the similarities between our work and
smoothing splines. In addition, we would like to thank Dejan
Slep\v{c}ev, No\'e Ducharme, and Nitya Gadhiwala for many helpful
discussions.

\section{Preliminaries}
In this section we define the constraint and objective function and
prove existence of optimizers. For the sake of brevity, we omit the
proofs of the standard facts in \cref{sec:the-constraint}.

\subsection{The Sobolev Constraint}
\label{sec:the-constraint}
We use a vector-valued version of the Sobolev norm as our constraint
$\cost$. There are established ways to extend the scalar Sobolev norm
to such contexts; see \eg\
\cite{evseev2021,Caamano2021Jan,Kreuter2015}. In order to simplify
some arguments we have chosen a slightly different definition that
nonetheless yields an equivalent norm in our context. Note, since
Sobolev functions are defined on open sets we will temporarily replace
$\dset$ with $\dsub$ until we can recover $\dset$ via the Sobolev
inequality.
\begin{defbox}
  \begin{definition} \label{def:sob-func} As usual we let $\wkpu$
    denote the scalar-valued $(\sobk, \sobp)$ Sobolev space on
    $\dsub$. Endow $\wkpur \coloneqq \bigoplus_{j=1}^\cdim \wkpu$ with
    the norm
    \[
      \norm{\fmap}_{\wkpur} \coloneqq \pn[bigg]{\sum_{j=1}^\cdim
        \norm{\fcompj}_{\wkpu}^\sobp}^{1/\sobp},
    \]
    where $\fcompj$ denotes the $j$\textsuperscript{th} component
    function of $\fmap$. We let $\wkpuo$ be the restriction of
    $\wkpur$ to those $\fmap$ taking values in $\cset$ and endow it
    with the inherited norm $\sobnorm{\fmap}$ (note that since $\cset$
    is compact and convex \ref{hyp:cset}, $\wkpuo$ is not itself a Banach
    space, but rather a convex subset). We define our constraint
    function $\cost$ to be
      \begin{equation}
        \label{eqn:cost}
        \cost(\fmap)=\sobnorm{\fmap}.
      \end{equation}
  \end{definition}
\end{defbox}

\begin{boxedremark}
  \label{rem:rotation-invariance}
  Let $S : \dsub \to \dsub$ and $T : \cset \to \cset$ be smooth
  isometries. Note that in general $\cost(\fmap \circ S) =
  \cost(\fmap)$. By contrast we can only guarantee $\cost(T \circ
  \fmap) = \cost(\fmap)$ when $\sobp = 2$, since for $\sobp \neq 2$
  the $\sobp$-norms on $\rcdim$ are not rotationally-invariant.
\end{boxedremark}

\subsubsection{Basic technical facts for the Sobolev constraint $\cost$}
Since $1 < \sobp < \infty$ \ref{hyp:sobp}, $\wkpu$ is a reflexive,
separable Banach space whence $\wkpur$ is as well. Due to compactness
and convexity of $\cset$ \ref{hyp:cset}, applying textbook results
like \cite[Cor.\ 3.22,~Thm.\ 3.29]{brezis}, we get the following
lemma.
\begin{lemma}
  \label{lem:weak-compactness}
  Let $\cset$ be compact and convex \ref{hyp:cset} and let $1 < \sobp
  < \infty$ \ref{hyp:sobp}. Fix $\fmap \in \wkpuo$ and for concision
  let $\bell(\fmap) \coloneqq \ol{B}_\budg(\fmap;\ \wkpuo)$, the
  closed ball of radius $\budg$ around $f$. Then $\bell(\fmap)$ is
  weakly compact and weakly metrizable.
\end{lemma}

Weak metrizability makes weak sequential continuity equivalent to weak
continuity. One sees our objective function $\objfp$ \eqref{eqn:fp} is
continuous on $L^\infty(\dset; \cset)$, hence if we can get a
relationship between weak convergence in $\wkpuo$ and uniform
convergence this will make it easy to show weak continuity of
$\objfp$. To that end we first recall the Sobolev inequality, which is
readily extended to our context.
\begin{proposition}[General Sobolev Inequality] \label{prop:gen-sob}
  Suppose that $\dsub$ is bounded and that $\partial \dsub$ is $C^1$
  \ref{hyp:dset}. Let $\sobk, \sobp \in \NN$ such that $\sobk\sobp >
  \ddim$ \ref{hyp:kqm} and let $\fmap \in \wkpur$. Define
  \begin{align*}
    \holda
    &\coloneqq \sobk - 1 - \floor{\frac{\ddim}{\sobp}} \\
    \holdr
    &\coloneqq
      \begin{cases}
        1 - \pn{\frac{\ddim}{\sobp} - \floor{\frac{\ddim}{\sobp}}}
        & \frac{\ddim}{\sobp} \in \NN \\
        \text{any positive number} < 1
        & \frac{\ddim}{\sobp} \not \in \NN.
      \end{cases}
  \end{align*}
  Then $\fmap \in \holdur$ and there exists a constant $C$ depending
  only on $\sobk, \sobp, \ddim, \holdr$, and $\dsub$ such that
  \[
    \norm{\fmap}_{\holdur} \leq C
    \norm{\fmap}_{\wkpur}.
  \]
\end{proposition}
Conveniently, $\holdur$ gives enough structure to show that weak
convergence must be at least uniform.
\begin{corollary} \label{cor:wconv-uconv} Take \ref{hyp:dset},
  \ref{hyp:cset}, and \ref{hyp:kqm}. Let $\fjseq$ be a sequence in
  $\wkpsets$ with $\fj \wconv \fmap$. Then in fact $\fj \to \fmap$
  uniformly.
\end{corollary}
The proof of \cref{cor:wconv-uconv} is a straightforward
Arzel\`a-Ascoli argument. Compactness of $\cset$ \ref{hyp:cset} gives
a uniform bound to the $\set{\fmap_j}$ while \ref{hyp:dset},
\ref{hyp:kqm} mean \cref{prop:gen-sob} can be used to obtain
equicontinuity via a uniform bound to the H\"older coefficients. Then,
translating from a uniformly-convergent subsequence to uniform
convergence overall can be done via uniqueness of weak limits.

\subsubsection{Discussion of Assumptions}
Unless stated otherwise, from now on we assume \ref{hyp:dset},
\ref{hyp:cset}, and \ref{hyp:kqm}, whence we may make free use of
\cref{lem:weak-compactness}, \cref{prop:gen-sob}, and
\cref{cor:wconv-uconv}. In particular, we may apply
\cref{prop:gen-sob} to represent any $\fmap \in \wkpuo$ with an
a.e.-equivalent $\holdxc$ representative. Hence we will interpret
``$\wkpsets$'' to mean the embedded copy of $\wkpuo \into \holdxc$,
and we will use the notation $\cost(\fmap)$ to denote
$\norm{\fmap}_{\wkpsets}$.

\subsection{The Objective}
\label{sec:objective}
We now formally define the objective.
\begin{defbox}
  \begin{definition}[Objective Function] \label{def:objective} For $0
    < \objp < \infty$ let
    \begin{align*}
      {\objfp}(\fmap)
      &\coloneqq \int_\cset \inf_{\dpt \in \dset} \abs{\cpt -
        \fmap(\dpt)}^\objp \dd \cmeas(\cpt).
    \end{align*}
    When necessary we will use the notation $\objfp(\fmap; \cmeas)$ to
    specify the target measure $\cmeas$ explicitly.
  \end{definition}
\end{defbox}

We are not aware of a reference treating the $0 < \objp < 1$ regime in
detail; in similar fashion we will largely restrict our attention to
$1 \leq \objp < \infty$ \ref{hyp:p} unless the general case presents
no additional complexity. In any case, note $\objfp$ depends only on
the image $\fmap(\dset)$ and not on the parametrization $\fmap$. We
also note that $\objfp$ has a very close relationship to
Monge-Kantorovich transport costs \eqref{eqn:Wassp}:
\begin{proposition}
  \label{prop:connection-to-ot}
  Given fixed $\fmap : \dset \to \cset$, as usual let $\iset =
  \fmap(\dset)$ and let $\cpush$ denote the pushforward of $\cmeas$
  under the closest-point projection $\pf: \cset \to \iset$,
  tiebreaking measurably (\eg\ \cref{prop:measurable-selection}) when
  $\pf$ is not uniquely-determined. Then
  \[
    \objfp(\fmap) = \wpp(\cmeas, \cpush) = \inf_{\imeas \in
      \pmeas(\iset)} \wpp(\cmeas, \imeas).
  \]
\end{proposition}
\begin{proof}
  We begin with the first equality. Define $\besttp = (\id, \pf)_{\#
    \cmeas}$ and note $\besttp \in \uneqtps(\cmeas, \cpush)$. By
  definition $\objfp(\fmap) = \int_\cset \abs{\cpt - \pf(\cpt)}^\objp
  \dd \cmeas = \int_{\cset \times \iset} \abs{\cpt - \ipt}^\objp \dd
  \besttp(\cpt, \ipt)$, so we see $\objfp(\fmap) \geq \wpp(\cpush,
  \cmeas)$. Suppose, to obtain a contradiction, that the inequality
  were strict. Then there exists a more efficient transport plan
  $\bestertp \in \uneqtps(\cmeas, \cpush)$ and a set $A \subseteq
  \cset$  such that
  \[
    \int_{A \times \iset} \abs{\cpt - \ipt}^\objp \dd \bestertp(\cpt,
    \ipt) < \int_{A \times \iset} \abs{\cpt - \ipt}^\objp \dd
    \besttp(\cpt, \ipt) = \int_{A \times \iset} \abs{\cpt -
      \pf(\cpt)}^\objp \dd \cmeas(\cpt).
  \]
  By the marginal condition, $\bestertp(A \times \iset) = \cmeas(A) =
  \besttp(A \times \iset)$. So there exists at least one $(\cpt_0,
  \ipt_0)$ pair such that $\abs{\cpt_0 - \ipt_0}^\objp < \abs{\cpt -    \pf(\cpt)}^\objp$,
a  contradiction to the definition of $\pf$. So
  $\objfp = \wpp(\cpush, \cmeas)$ as desired.

  For the second equality, note we already have $\wpp(\cmeas, \cpush)
  \geq \inf_{\imeas} \wpp(\cmeas, \imeas)$. Hence suppose, to obtain a
  contradiction, that there exists $\imeas_0$ such that $\wpp(\cmeas,
  \cpush) > \wpp(\cmeas, \imeas_0)$. Let $\besttp \in \uneqtps(\cmeas,
  \cpush)$ and $\bestertp \in \uneqtps(\cmeas, \imeas_0)$. Since both
  measures share the same first marginal one may proceed exactly as
  before to derive a contradiction.
\end{proof}
This justifies the intuitive statements we made previously about
$\objfp(\fmap)$ measuring how well $\fmap$ ``approximates'' $\cmeas$.
\begin{boxedremark}
  Note that in light of \cref{prop:connection-to-ot}, when $0 < \objp
  < 1$ one could develop economic interpretations for $\objfp$
  analogous to those given for $\wpp(\cdot, \cdot)$ in
  \cite{gangboGeometryOptimalTransportation1996}. In particular, it
  would be interesting to study how qualitative properties of
  optimizers change at the critical value $\objp = 1$.
\end{boxedremark}

\subsubsection{Nonconcavity and Nonconvexity of $\objfp$}
\label{sec:non-concavity-non-convexity}
Given \cref{prop:connection-to-ot} it is natural to ask whether
$\objfp$ possesses any sort of concavity or convexity. This turns out
to be too much to hope for: in general, $\objfp$ is both nonconcave
and nonconvex.

For nonconcavity, note that when $\objp \geq 1$ \ref{hyp:p}, the
$\objp$\textsuperscript{th}-power function $x \mapsto |x|^\objp$ is a
nonlinear, convex function, and thus nonconcave. When $\objp > 1$,
$\objfp$ inherits this nonconcavity trivially, as can be seen by
interpolating between any pair of constant functions $\fmap_0 \neq
\fmap_1$. When $\objp = 1$ the nonconcavity behavior is slightly
different, but can still be obtained as long as $\cdim > 1$ (as can be
seen by considering $\cset = [0,1]$ and $\cmeas = \frac{1}{2}
\pn{\delta_0 + \delta_1}$, this $\cdim > 1$ hypothesis is essential).

On the other hand, even with convexity of $|x|^\objp$, it is easy to
demonstrate that $\objfp$ is not convex either. To make the argument
simpler we first note a trivial property of $\objfp$.
\begin{proposition}
  \label{prop:supset-better}
  Let $\fmap_0, \fmap_1 : \dset \to \cset$; denote their images by
  $\iset_0, \iset_1$. Suppose $\iset_0 \subseteq \iset_1$. Then
  $\objfp(\fmap_1) \leq \objfp(\fmap_0)$. If, further, there exists
  $\ipt \in \iset_1 \setminus \iset_0$ such that (1) $\ipt \in
  \supp(\cmeas)$ and (2) $d(\ipt, \iset_0) > 0$, then $\objfp(\fmap_1)
  < \objfp(\fmap_0)$.
\end{proposition}
Then we have:
\begin{counterexample}[Nonconvexity]
  \label{cex:convexity-rn}
  Suppose $\supp(\cmeas)$ contains at least two distinct points, call
  them $\cpt_0$ and $\cpt_1$. Fix an arbitrary smooth surjection
  $\varphi : \dset \to [0,1]$. (\eg\ projection on an axis plus
  normalization). By \ref{hyp:cset} $\cset$ is convex, so
  \[
    \fmap_0(\dpt) \coloneqq (1 - \varphi(\dpt)) \cpt_0 + \varphi(\dpt)
    \cpt_1
    \quad \text{and} \quad
    \fmap_1(\dpt) \coloneqq \varphi(\dpt) \cpt_0 + (1 - \varphi(\dpt))
    \cpt_1
  \]
  are both admissible in $\wkpsets$. Define $\fmap_{.5} =
  \frac{1}{2}(\fmap_0 + \fmap_1) \equiv \frac{\cpt_0 + \cpt_1}{2}$.
  Since $\objfp(\fmap_0) = \objfp(\fmap_1)$, applying
  \cref{prop:supset-better} gives $\frac{1}{2} (\objfp(\fmap_0) +
  \objfp(\fmap_1)) = \objfp(\fmap_0) < \objfp(\fmap_{.5})$.
  \begin{figure}[H]
    \centering
    \begin{subfigure}{.32\linewidth}
      \centering
      \includegraphics[scale=.7]{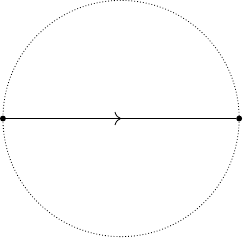}
      \caption{$\fmap_0$}
    \end{subfigure}
    \begin{subfigure}{.32\linewidth}
      \centering
      \includegraphics[scale=.7]{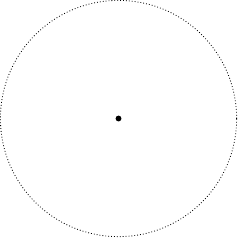}
      \caption{$\fmap_{.5}$}
    \end{subfigure}
    \begin{subfigure}{.32\linewidth}
      \centering
      \includegraphics[scale=.7]{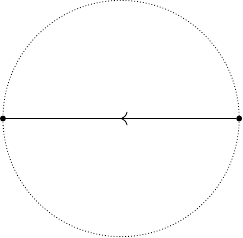}
      \caption{$\fmap_1$}
    \end{subfigure}
    \caption{Example for $\cset = B_1(0; \RR^2)$ and $(\pm 1,0) \in
    \supp(\cmeas)$.}
  \end{figure}
\end{counterexample}
\cref{cex:convexity-rn} is ``trivial'' in the sense that we have
$\iset_0 = \iset_1$. Many other counterexamples can be constructed in
which $\iset_0$, $\iset_1$ are equivalent under symmetry action on
$\cset$ (for example, $\iset_0 = [0,1] \times \set{0}$, $\iset_1 =
[-1, 0] \times \set{0}$). Thus one might wonder if it is possible to
salvage convexity of $\objfp$ by modding out by the ``right'' choice
of symmetry group of $\cset$. Unfortunately,
\cref{cex:nonconvexity-asymmetric} seems to imply the answer is no.
\begin{counterexample}[Nonconvexity, Asymmetric Example]
  \label{cex:nonconvexity-asymmetric}
  One may apply \cref{prop:supset-better} to $\fmap_0$, $\fmap_{.5}$,
  $\fmap_1$ having the following images:
  \begin{figure}[H]
    \centering
    \begin{subfigure}{.32\linewidth}
      \centering
      \includegraphics{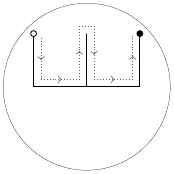}
      \caption{$\fmap_0$}
    \end{subfigure}
    \begin{subfigure}{.32\linewidth}
      \centering
      \includegraphics{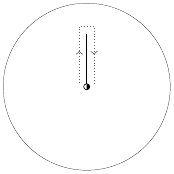}
      \caption{$\fmap_{.5}$}
    \end{subfigure}
    \begin{subfigure}{.32\linewidth}
      \centering
      \includegraphics{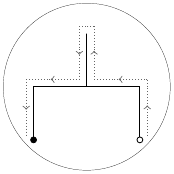}
      \caption{$\fmap_1$}
    \end{subfigure}
    \caption{Example $\fmap_0$, $\fmap_{.5}$ and $\fmap_1$ for which
      $\iset_0 \not\equiv \iset_1$.}
    \label{fig:non-symm-cex}
  \end{figure}
\end{counterexample}
Note in \cref{fig:non-symm-cex} the unfilled/filled circles represent
the starting/ending points in a $1$-d parametrization \`a la
$\varphi(\dpt)$ in \cref{cex:convexity-rn}. Using an argument similar
to \cref{lem:arbitrarily-bad-reparametrizations} one may show that
$\varphi$ can be made smooth whence $\fmap_0$, $\fmap_{.5}$, and
$\fmap_1$ are all admissible.

In any case, the intuition is that nonconvexity tends to arise from
cancellation, while nonconcavity (when $\objp >1$) tends to arise from
the convexity of $x \mapsto \abs{x}^\objp$. Both appear to be
fundamental features of the problem.

\subsection{Continuity Properties of $\objfp$}
\label{sec:continuity}
$\objfp(\fmap;\cmeas)$ has a very helpful continuity property that
will be used in proving existence of optimizers and also in
establishing the consistency results later on. Note this result does
not rely on \ref{hyp:p}.

\begin{theorem}[Joint Sequential Continuity]
  \label{thm:joint-continuity}
  Take \ref{hyp:cset} and let $\set{\fmap_M}$, $\set{\cmeas_N}$ be
  convergent in $C(\dset; \cset)$ and $\pmeas(\cset)$, respectively,
  with the former given the topology of uniform convergence and the
  latter given the usual weak-\textasteriskcentered{} topology. Then
  \[
    \lim_{M,N \to \infty} \objfp(\fmap_M; \cmeas_N) = \objfp(\finf;
    \cmeas_\infty).
  \]
  In particular, if one also takes \ref{hyp:dset} and \ref{hyp:kqm},
  then by \cref{cor:wconv-uconv} if $\set{\fmap_M}$ is any sequence
  converging weakly in $\wkpsets$, the same result holds.
\end{theorem}
\begin{proof}
  We employ the Moore-Osgood theorem for double limits. It suffices to
  show
  \begin{enumerate}[label=(\roman*)]
    \item $\objfp(\fmap_M; \cmeas_N) \to \objfp(\fmap_\infty;
      \cmeas_N)$ uniformly in $N$, and
    \item for each $M$, $\objfp(\fmap_M; \cmeas_N) \to \objfp(\fmap_M;
      \cmeas_\infty)$.
  \end{enumerate}
  For (i): For all $\dpt, \cpt$, we have $\abs[big]{\abs{\fmap_M(\dpt)
      - \cpt} - \abs{\fmap_\infty(\dpt) - \cpt}} \leq
  \abs{\fmap_M(\dpt) - \fmap_\infty(\dpt)}$, and so $\fmap_M \uconv
  \fmap_\infty$ gives $d(\fmap_M, \cpt) \uconv d(\fmap_\infty, \cpt)$
  in $\cpt$. Then since $\cset$ is compact \ref{hyp:cset},
  \[
    \inf_{\dpt \in \dset} \abs{\fmap_M(\dpt) - \cpt}^\objp \uconv
    \inf_{\dpt \in \dset} \abs{\fmap_\infty(\dpt) - \cpt}^\objp
  \]
  as functions of $\cpt$. Finally since the $\cmeas_N$ are probability
  measures we have $\objfp(\fmap_M; \cmeas_N) \to \objfp(\fmap_\infty;
  \cmeas_N)$ uniformly in $N$, as desired.

  For (ii): Note that $\inf_{\dpt \in \dset} \abs{\fmap_M(\dpt) -
    \cpt}^\objp$ is a bounded, continuous function of $\cpt$, so by
  weak convergence of the $\cmeas_N$ we immediately get
  \[
    \int_\cset \inf_{\dpt \in \dset} \abs{\fmap(\dpt) - \cpt}^\objp
    \dd \cmeas_N \to \int_\cset \inf_{\dpt \in \dset} \abs{\fmap(\dpt)
      - \cpt}^\objp \dd \cmeas,
  \]
  as desired. Thus the double limit exists, namely,
  \[
    \lim_{M,N\to\infty} \objfp(\fmap_M; \cmeas_N) = \lim_{M\to\infty}
    \lim_{N\to\infty} \objfp(\fmap_M; \cmeas_N) = \lim_{M\to\infty}
    \objfp(\fmap_M; \cmeas_\infty) = \objfp(\fmap_\infty;
    \cmeas_\infty)
  \]
  as desired.
\end{proof}

\section{Existence of Optimizers and Properties of $J$}
\label{sec:existence-and-J}

We now have everything we need to show existence of solutions to
\xref{prob:hard-constraint}. But first we address a small technical
detail when $0 \not \in \cset$. We remind the reader that when writing
$\cflb$ it should be understood that we are considering only $\fmap
\in \wkpsets$. Define
\begin{equation}
  \budgmin = \inf \set{\budg \MID \cflb \neq
  \varnothing}. \label{eq:budgmin-def}
\end{equation}
Recall our global assumptions \ref{hyp:cset} and \ref{hyp:sobp}, which
stipulate that $\cset$ is compact and convex and that $\sobp > 1$. So
in general $\budgmin$ is achieved by a unique constant function
\begin{equation}
  \fmin = \argmin_{\cpt \in \cset} \cost(\cpt) = \argmin_{\cpt \in
    \cset} \abs{\cpt}^\sobp. \label{eq:fmin-def}
\end{equation}
Clearly $0 \in \cset$ iff $\reffmin = 0$ iff $\refbudgmin = 0$. Not
much would be lost by assuming $0 \in \cset$; nonetheless, we treat
the general case since doing so adds little complexity and the case
$\refbudgmin \neq 0$ might be of interest in applications (for
example, if a package delivery service has its central warehouse
located outside a city modeled by $\cset$).

\subsection{Existence \& Nonuniqueness of Optimizers}
\label{sec:existence-and-nonuniqueness}
\begin{theorem}[Existence of Optima]
  \label{thm:existence}
  Take \ref{hyp:dset}, \ref{hyp:cset}, \ref{hyp:sobp}, and
  \ref{hyp:kqm}. Fix $\budg \geq \refbudgmin$. Then $\objfp$ attains
  its optima on $\cflb$. In particular \xref{prob:hard-constraint} has
  solutions.
\end{theorem}
\begin{proof}
  Using \ref{hyp:cset}, \ref{hyp:sobp}, the definition of $\refbudgmin$
  implies $\cflb$ is nonempty, and by \cref{lem:weak-compactness},
  $\cflb$ is weakly compact and weakly metrizable. From weak
  metrizability we get that weak continuity and weak sequential
  continuity are equivalent. By \ref{hyp:dset}, \ref{hyp:cset}, and
  \ref{hyp:kqm}, \cref{thm:joint-continuity} implies $\objfp$ is
  weakly sequentially continuous, hence it is weakly continuous on
  $\cflb$. So continuity-compactness gives that $\objfp$ achieves its
  optima on $\cflb$.
\end{proof}
Note that \cref{thm:existence} relies on the fact that $\cmeas$ is a
probability measure. However, it does not require \ref{hyp:p},
\ref{hyp:ac}. In any case, we will refer to solutions of
\xref{prob:hard-constraint} as \emph{$\budg$-optimizers}. We now
discuss some trivial cases of nonuniqueness arising when the optimizer
is nonconstant. We first show that for all $\budg$ sufficiently large,
the optimizers \emph{must} be nonconstant. To that end we define
another special constant function:
\begin{equation}
  \fnc \in \argmin_{\cpt \in \cset} \objfp(\cpt). \label{eq:fnc-def}
\end{equation}
Here $\reffnc$ (known as the \emph{$\objp$-mean}; see
\cite{cuesta1987,Cuesta1988Jul}) corresponds to the geometric median
of $\cmeas$ when $\objp = 1$ and the $\cmeas$-mean when $\objp = 2$.
It is unique for $\objp > 1$ as well as when $\objp=1$ as long as
$\cmeas$ is not essentially 1-dimensional \cite{Cuesta1988Jul}.
\begin{boxedremark}
  The theory becomes cleaner if one assumes $\reffmin = \reffnc$. In general
  $\reffmin = 0$ or $\reffmin \in \partial \cset$; on the other hand,
  \ref{hyp:ac} gives $\reffnc \in \cset^\circ$. So the only way to force
  $\reffmin = \reffnc$ is to make the translation $\cset \mapsto \cset -
  \reffnc$. Doing so is a nontrivial simplification that will
  qualitatively change $\budg$-optimizers for small $\budg$;
  nevertheless as $\budg \to \infty$ the differences between $\cset$
  and $\cset - \reffnc$ are no longer so important as the
  $0$\textsuperscript{th}-order term in $\cost$ becomes negligible. So
  the assumption is not severe if one is only treating asymptotics.
\end{boxedremark}
Let us define
\begin{equation}
  \bnc \coloneqq \cost(\reffnc). \label{eq:bnc-def}
\end{equation}
The subscript ``nc'' was chosen for ``non-constant'' in light of the
following:

\begin{proposition} \label{prop:nontrivial-optimizers-nonconstant}
  Suppose $\cset$ is convex (\eg\ via \ref{hyp:cset}) and $\cost(\fmap)$
  is a seminorm that is finite on smooth functions (\eg\ via
  \ref{hyp:sob}). Then if $\cmeas \neq \delta_{\reffnc}$, then for all
  $\budg > \refbnc$, $\budg$-optimizers are nonconstant.
\end{proposition}

Notice that $\budg$-optimizers can be non-unique, and one may obtain
trivial counterexamples for $\sobp = 2$ when $\supp(\cmeas)$ has
symmetries. To that end first recall that if $\varphi : A \subseteq
\RR^N \to A$ is an isometry then $\varphi(a) = Ta + b$ where $T$ is an
orthogonal linear transformation; thus $\abs{D\varphi} = 1$ and
$\abs{D^\mind \varphi} = 0$ for higher $\mind$.
\begin{counterexample}[Nonuniqueness]
  \label{cex:nonuniqueness}
  Suppose $\cmeas \neq \delta_{\reffnc}$ and for some $\budg >
  \refbnc$ let $\fbudg$ be an $\budg$-optimizer. Suppose $\psi : \cset
  \to \cset$ is a nontrivial isometry with $\psi_\#(\cmeas) = \cmeas$,
  and let $\wt \fbudg = \psi \circ \fbudg$. Then $\objfp(\wt \fbudg) =
  \objfp(\fbudg)$ and (provided $\sobp = 2$) $\cost(\wt \fbudg) =
  \cost(\fbudg)$, but $\wt \fbudg \neq \fbudg$. Note that $\fbudg$
  being nonconstant is important for getting $\wt \fbudg \neq \fbudg$;
  otherwise $\fbudg$ could be the fixed point of $\psi$.
\end{counterexample}

\subsection{Properties of $\objell$}
Recall $\objell : \bp{\refbudgmin, \infty} \to \RR^{\geq 0}$ given in
\eqref{eqn:J-l}. In view of \cref{thm:joint-continuity}, the following
result is straightforward:
\begin{theorem}
  \label{thm:objell-properties}
  Take \ref{hyp:dset}, \ref{hyp:cset}, \ref{hyp:sobp}, and
  \ref{hyp:kqm}. Then $\objell$ is nonincreasing, bounded, and
  continuous.
\end{theorem}
\begin{proof}
  The fact that $\objell$ being nonincreasing is immediate: For all
  $\refbudgmin \leq \budg_0 \leq \budg_1$, $\set{\cost(\fmap) \leq
    \budg_0} \subseteq \set{\cost(\fmap) \leq \budg_1}$ so
  $\objell(\budg_1) \leq \objell(\budg_0)$. For boundedness: The value
  $\objell(\budg)$ is trivially bounded below by $0$ and above by
  $\objell(\refbudgmin) = \objfp(\reffmin)$.

  Now we show $\objell$ is right continuous. Fix $\bz \geq \refbudgmin$
  and let $\set[]{\bjp}_{j=1}^\infty$ with $\bjp \decto \bz$. Since
  $\objell(\budg)$ is nonincreasing in $\budg$, $\objell(\bjp)$ is
  nondecreasing in $j$ and also bounded above by $\objell(\bz)$. Thus
  $\objell(\bjp)$ converges to some $\objellinfp \leq \objell(\bz)$.
  We claim $\objellinfp \geq \objell(\bz)$ too. For each $j$ let
  $\fjp$ be an $\bjp$-optimizer. Note $\set[]{\fjp}_{j=1}^\infty
  \subseteq \cfl{\bop}$, which by \cref{lem:weak-compactness}
  \ref{hyp:cset}, \ref{hyp:sobp} is weakly compact. Thus there exists
  a subsequence $\set[]{\fmap_{j_i}^+}_{i=1}^\infty$ converging weakly
  to some $\fmap_\infty^+ \in \cfl{\bop}$; by
  \cref{thm:joint-continuity} \ref{hyp:dset}, \ref{hyp:cset},
  \ref{hyp:sobp}, \ref{hyp:kqm} we see $\objfp(\fmap_\infty^+) =
  \objellinfp$. Finally since $\cost$ is defined via a norm it is
  weakly-l.s.c.\ whence $\cost(\fmap_\infty^+) \leq \bz$. In
  particular since $\objell(\bz)$ is the minimum value over
  $\set{\cost(\fmap) \leq \bz}$ we get $\objell(\bz) \leq
  \objfp(\fmap_\infty^+) = \objellinfp$. So $\objell$ is
  right-continuous as desired.

  Now we show $\objell$ is left continuous. Fix $\bz > \refbudgmin$
  and let $\set[]{\bjm}_{j=1}^\infty$ with $\bjm \incto \bz$. Let
  $\fz$ be a $\bz$-optimizer and for all $j$ define $t_j = \pn[]{\bjm
    - \refbudgmin}/\pn{\bz - \refbudgmin}$ and $\fjm = t_j \fmap_0 +
  (1-t_j) \reffmin$. Note that since $\cset$ is convex \ref{hyp:cset}
  and $\fjm$ is constructed via a convex combination, $\fjm(\dset)
  \subseteq \cset$. Also note that
  \[
    \cost(\fjm) \leq t_j \bz + (1-t_j) \refbudgmin = \bjm,
  \]
  whence $\objfp(\fjm) \geq \objell(\bjm) \geq \objell(\bz) =
  \objfp(\fmap_0)$. Since $\fjm \to \fmap_0$ strongly, as in the
  previous case weak continuity implies $\objfp(\fjm) \to
  \objfp(\fmap_0)$, whence $\objell(\bjm) \to \objell(\bz)$ as well.
  So $\objell(\budg)$ is both left-continuous and right-continuous;
  thus it is continuous.
\end{proof}

Supposing the constraint is \emph{effective} (\ie\ ``$\cost(\fmap) <
\infty$ implies $\objfp(\fmap) \neq 0$;'' see
\cref{prop:effectiveness}) we expect $\objell$ to be \emph{strictly
  decreasing} rather than just nonincreasing. This was shown for
$\cost(\fmap) = \arclen(\fmap)$ in \cite{delattre2020}, but as usual
the local modification argument does not carry over when $\sobk > 1$.
It would suffice to prove optimizers can occur only on
$\set{\cost(\fmap) = \budg}$; we were unable to find a proof of this,
but show in \cref{cor:saturating-optimizers} that at least
$\set{\cost(\fmap) = \budg}$ is always guaranteed to contain an
optimizer. In fact, this is a simple consequence of reparametrization
and does not depend on \ref{hyp:dset}--\ref{hyp:kqm}.

\begin{lemma}[Arbitrarily Bad Reparametrizations]
  \label{lem:arbitrarily-bad-reparametrizations}
  Let $\fmap$ be nonconstant with $\budg = \cost(\fmap)$. Then for all
  $\wt \budg > \budg$ there exists $\varphi : \dset \to \dset$ such
  that
  \begin{enumerate}[label=\roman*)]
    \item $\ftilde \coloneqq \fmap \circ \varphi \in \wkpsets$,
    \item $\ftilde(\dset) = \fmap(\dset)$, and
    \item $\cost(\wt \fmap) = \wt \budg$.
  \end{enumerate}
\end{lemma}
\begin{proof}
  We have two cases: $\sobk > 1$ and $\sobk = 1$. For $\sobk > 1$:
  First, we exhibit a family of smooth surjections
  $\varphi_\varepsilon : B_1(0; \RR^m) \to B_1(0; \RR^m)$ where
  $\varepsilon \in \pb{0,1}$ such that $\varphi_1 = \id(\dpt)$ and for
  all multiindices $\abs{\mind} > 1$, $\lim_{\varepsilon \to 0^+}
  \norm{D^\mind \varphi_\varepsilon}_{L^\sobp} = \infty$. Let
  \[
    \smoothstep(t) =
    \begin{cases}
      0 & t \in (-\infty, 0] \\
      \displaystyle \frac{e^{-1/t}}{e^{-1/t} + e^{-1/(1-t)}}
        & t \in \pn{0,1} \\
      1 & t \in [1, \infty).
    \end{cases}
  \]
  For all $\varepsilon \in \pb{0,1}$ let $\psi_\varepsilon(\dpt) =
  \hat \dpt \smoothstep(\abs{\dpt}/\varepsilon)$.
  $\psi_\varepsilon(\dpt)$ is a smooth surjection and for $\abs{\mind}
  > 1$, $\lim_{\varepsilon \to 0^+} \norm{D^{\mind}
    \psi_\varepsilon}_{L^\sobp} = \infty$. Now let
  $\varphi_\varepsilon(\dpt) = (1-\varepsilon)\psi_\varepsilon(\dpt) +
  \varepsilon\id(\dpt)$. By construction $\varphi_\varepsilon$ is
  smooth. To see $\varphi_\varepsilon$ it is surjective write
  $\varphi_\varepsilon(\dpt) = \hat \dpt \pn[Big]{(1-\varepsilon)
    \abs{\psi_\varepsilon(\dpt)} + \varepsilon \abs{\dpt}}$ and note
  that for all $\varepsilon \in \pb{0,1}$, the scalar term is $0$ when
  $\abs{\dpt} = 0$, $1$ when $\abs{\dpt} =1$, and strictly increasing
  in $\abs{\dpt}$. Now we use the $\varphi_\varepsilon$ to define
  $\ftilde$. Let $\dpt \in \dset$ such that $\fmap$ is nonconstant at
  $\dz$. Fix $r> 0$ and define
  \[
    \ftilde_\varepsilon(\dpt) =
    \begin{cases}
      \fmap(\dpt) & \dpt \not \in B_r(\dz) \\
      \fmap\pn{\dz + r\varphi_\varepsilon\pn{\frac{\dpt -
      \dz}{r}}} & \dpt \in B_r(\dz).
    \end{cases}
  \]
  By smoothness and surjectivity of $\varphi_\varepsilon$ we have
  $\ftilde_\varepsilon \in \wkpsets$ and $\ftilde_\varepsilon(\dset) =
  \fmap(\dset)$. Note $\cost(\ftilde_\varepsilon)$ is continuous in
  $\varepsilon$ and $\lim_{\varepsilon\to 0^+}
  \cost(\ftilde_\varepsilon) = \infty$. Hence there exists a
  particular $\varepsilon$ yielding $\cost(\ftilde_\varepsilon) = \wt
  \budg$ as desired.

  For the case $\sobk = 1$, the same flavor of construction works if
  one modifies $\psi_\varepsilon$ such that its values oscillate
  between $0$ and $\partial B_1$ arbitrarily many times as
  $\varepsilon \to 0^+$.
\end{proof}
\begin{corollary}[Existence of Saturating Optimizers]
  \label{cor:saturating-optimizers}
  For all $\budg > \refbnc$, there exists $\fmap \in \set{\cost(\fmap)
    = \budg}$ with $\objfp(\fmap) = \objell(\budg)$.
\end{corollary}

\begin{proof}
  Let $\fbudg$ be an $\budg$-optimizer. If $\cost(\fbudg) < \budg$
  then we can use \cref{lem:arbitrarily-bad-reparametrizations} to
  reparametrize to get $\fbt$ such that $\cost (\fbt) = \budg$ and
  $\fbt(\dset) = \fbudg(\dset)$, whence $\objfp(\fbt) = \objfp(\fbudg)
  = \objell(\budg)$.
\end{proof}

\subsection{Asymptotics for $\objell$}
\label{sec:asymptotics}
We discuss how $\objell (\budg)$ behaves as $\budg \to \infty$.

A first question is whether it is possible to achieve $\objell(\budg)
= 0$ with finite $\budg$. \cref{prop:effectiveness} gives sufficient
conditions for guaranteeing this does not occur. Note \ref{hyp:kqm} is
used only in getting \cref{prop:gen-sob}, while \ref{hyp:dset} appears in
the proof on its own.
\begin{proposition}[Effectiveness of $\cost(\fmap)$]
  \label{prop:effectiveness} Recall (see \eg\ \cite{Mattila2000})
  that we may extend the notion of Hausdorff dimension from subsets of
  $\rcdim$ to Borel measures by letting
  \[
    \hdim(\cmeas) = \inf \set{\hdim(A) \MID A \text{ is Borel and
      }\cmeas(A) > 0}.
  \]
  Take \ref{hyp:dset}, \ref{hyp:kqm} and let $C^{\holda, \holdr}(\dset;
  \cset)$ be the H\"older space guaranteed by \cref{prop:gen-sob}.
  Then if
  \begin{enumerate}
    \item $\holda \geq 1$ and  $\hdim(\cmeas) > \ddim$, or
    \item $\holda = 0$ and $\hdim\pn{\cmeas} > \ddim/\holdr$,
  \end{enumerate}
  then the constraint is effective (\ie\ for all $\fmap \in
  \wkpsets$, $\objfp(\fmap) > 0$) and as a consequence, for all $\budg
  \in \bp{\refbudgmin, \infty}$ we have $\objell(\budg) > 0$.
\end{proposition}
\begin{proof}
  Fix an arbitrary $\fmap \in \wkpsets$. In the first case, since
  $\fmap \in C^1$ we see $\hdim(\fmap(\dset)) \leq \ddim$, hence
  $\cmeas(\supp(\cmeas) \setminus \fmap(\dset)) > 0$ (otherwise,
  $\hdim(\cmeas) \leq \ddim$ a contradiction) and thus $\objfp(\fmap)
  > 0$. For the second case, denote the $d$-dimensional Hausdorff
  measure by $\haus^d$. Then a standard argument shows that there
  exists a constant $C$ depending only on $d, \holdr$ such that
  \[
    \haus^{d/\holdr} (\fmap(\dset)) \leq C \haus^d(\dset).
  \]
  In particular, \ref{hyp:dset} gives $\hdim(\dset) = \ddim$, whence for
  all $d > \ddim$ we have $\haus^{d/\holdr}(\fmap(\dset)) \leq C
  \haus^d(\dset) = 0$. Thus $\hdim(\fmap(\dset)) \leq \ddim/\holdr$
  and so $\cmeas(\supp(\cmeas) \setminus \fmap(\dset)) > 0$, whence
  $\objfp(\fmap) > 0$. Lastly, fixing $\budg \in \bp{\refbudgmin,
    \infty}$ and an $\budg$-optimizer $\fbudg \in \cflb$, the above
  gives $\objell(\budg) = \objfp(\fbudg) > 0$.
\end{proof}
When the constraint is effective, it's meaningful to ask about the
rate at which $\objell(\budg) \decto 0$ as $\budg \to \infty$.
\cref{prop:asymptotics} below provides a very coarse bound in this
direction. Our argument centers around constructing a degenerate
$\fmap$ with $\hdim(\fmap(\dset)) = 1$ independent of $\ddim$, hence
we imagine there should exist a tighter bound highlighting the
relationship between $\hdim(\fmap(\dset))$ and $\hdim(\supp(\cmeas))$
similarly to \cref{prop:effectiveness}. In any case, note that when
$\sobk = \objp = 1$ and $\sobp \to \infty$ \cref{prop:asymptotics}
recovers the bound \cite[Theorem~3.16]{buttazzo2002}, which inspired
the proof below.
\begin{proposition}
  \label{prop:asymptotics}
  Take \ref{hyp:dset}, \ref{hyp:cset}, and \ref{hyp:kqm}. Then there
  exists a constant $C$ depending only on $\dset$, $\cset$ such that
  for all $\varepsilon > 0$, taking $\budg > C
  \varepsilon^{(1-\cdim(\sobk + \sobp^{-1}))/\objp}$ yields
  $\objell(\budg) \leq \varepsilon$.
\end{proposition}
\begin{proof}
  Fix some $\varepsilon > 0$. We provide a sketch for the construction
  of an $\fmap$ with $\cost(\fmap) < \infty$ such that $\objfp(\fmap)
  \leq \varepsilon^\objp$. To that end let $M_\varepsilon$ be the
  $\varepsilon$-covering number of $\cset$ and let
  $\set[]{\cpt_j^\varepsilon}$ be an $\varepsilon$-covering of $\cset$
  by $M_\varepsilon$ points. We may abstract
  $\set[]{\cpt_j^\varepsilon}$ as a graph $G_\varepsilon$ with
  $M_\varepsilon$ vertices and edges between those $\cpt_j, \cpt_{j'}$
  with $d(\cpt_j, \cpt_{j'}) < 2\varepsilon$. Call a walk on
  $G_\varepsilon$ \emph{spanning} if it visits every vertex at least
  once. Since $\cset$ is connected \ref{hyp:cset}, $G_\varepsilon$ is
  connected, so there exists a minimal spanning tree $T_\varepsilon$.
  A simple inductive argument then shows that for $M_\varepsilon \geq
  2$, $T_\varepsilon$ has a spanning walk of length at most
  $N_\varepsilon \leq 2M_{\varepsilon} - 3$; denote this walk $\{v_1,
  v_2, \cdots, v_{N_\varepsilon}\}$.

  As in the proof of \cref{lem:arbitrarily-bad-reparametrizations} let
  $\smoothstep$ denote a smooth step function; note that $\smoothstep
  \in \wkp([0,1]; \RR)$. Now subdivide $[0,1]$ into $N_\varepsilon$
  equal subintervals $I_j$ and define a piecewise-smooth curve
  $\gamma$ by
  \[
    \gamma(t) =
      v_{j} + (v_{j+1} - v_j) \smoothstep(N_\varepsilon(t - t_j)) \ \
      \text{for } t \in I_j, \ \ j=1, \cdots, N_\varepsilon.
  \]
  Then since $\cset$ is convex \ref{hyp:cset}, $\gamma \in W^{\sobk,
    \sobp}([0,1]; \cset)$ with
  \begin{align*}
    \cost(\gamma)
    &\leq \pn[]{\diam(\cset)^\sobp + N_\varepsilon
      \varepsilon^\sobp \norm{\smoothstep(N_\varepsilon t)}_{W^{\sobk,
      \sobp}([0,1])}^\sobp}^{1/\sobp}.
      \shortintertext{Note that
      $\norm{\smoothstep(at)}_{W^{\sobk, \sobp}([0,1]; \RR)}$ is
      $O(a^\sobk)$ as $a \to \infty$; denote the constant by $C_0$ and
      note it is independent of $\dset, \cset$. As
      $\varepsilon \to 0$ we get $N_\varepsilon \to \infty$, whence}
    &\leq\pn[]{\diam(\cset)^\sobp + C_0 \varepsilon^{\sobp}
      N_\varepsilon^{\sobk \sobp + 1}}^{1/\sobp} \\
    &\leq C_1 \varepsilon N_\varepsilon^{\sobk + 1/\sobp},
  \end{align*}
  where in the last line we have used the fact that the second term
  dominates as $\varepsilon \to 0$. Note $C_1$ depends only on
  $\cset$. Since $N_\varepsilon \leq 2M_\varepsilon - 3$, a standard
  covering number bound gives $\cost(\fmap) \leq C_2 \varepsilon^{1-
    \cdim(\sobk + (1/\sobp))}$, with $C_2$ depending only on $\cset$.
  Define a smooth function $\varphi : \dset \to [0,1]$ such that
  $\norm{\varphi}_{\wkp(\dset; [0,1])}$ depends only on $\dset$. Then
  $\fmap \coloneqq \gamma \circ \varphi \in \wkpsets$ with
  $\objfp(\fmap) \leq \varepsilon^\objp$ and for some $C$ depending
  only on $\dset, \cset$ we have $\cost(\fmap) \leq C\varepsilon^{1 -
    \cdim(\sobk + \sobp^{-1})}$, from which the claim follows.
\end{proof}
As an alternative approach to the sort of argument in
\cref{prop:asymptotics}, one might consider starting with an $\fmap
\not \in \wkpsets$ but with $\objfp(\fmap) = 0$ and finding ways to
approximate it with elements of $\wkpsets$.
\begin{proposition}[Smoothing $\cmeas$-Fillings]
  \label{prop:mollification}
  Take \ref{hyp:dset}, \ref{hyp:cset}, and \ref{hyp:kqm}. Let $\finf :
  \dset \to \cset$ such that $\supp(\cmeas) \subseteq \fmap(\dset)$.
  Let $\varphi : \dset \to \RR$ be a mollifier and define
  $\varphi_\theta(\dpt) \coloneqq \theta^{-\ddim}
  \varphi(\dpt/\theta)$. Suppose $\fmap_N \uconv \fmap_\infty$ (note
  that one could take $\fmap_N = \finf$ for all $N$) and for all
  $\theta > 0$, define $\fthetn = \fn * \mthet$. Then $\fthetn \in
  \wkpsets$ and
  \[
    \lim_{\substack{\theta \to 0 \\ N \to \infty}} \objfp(\fthetn) =
    \objfp(\finf) = 0.
  \]
\end{proposition}
The proof is a direct application of the Moore-Osgood theorem similar
to the argument for \cref{thm:joint-continuity}.

\begin{figure}[H]
  \centering
  \begin{subfigure}[t]{.24\linewidth}
    \centering
    \includegraphics[scale=.25]{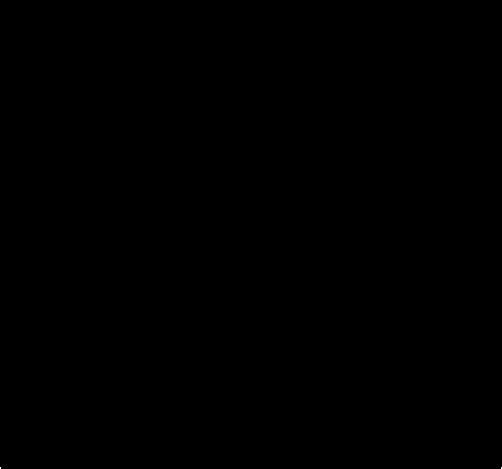}
    \caption{$\theta = 0$}
  \end{subfigure}
  \begin{subfigure}[t]{.24\linewidth}
    \centering
    \includegraphics[scale=.25]{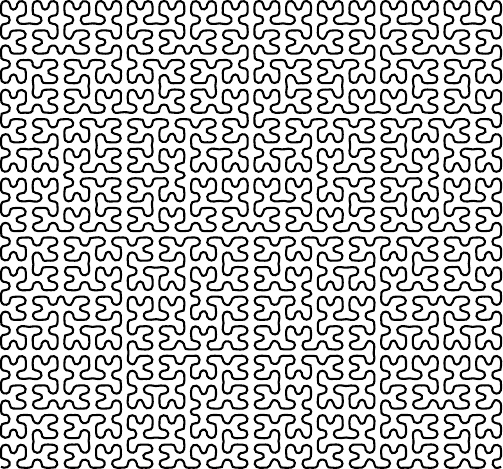}
    \caption{$\theta = 2^4/2^{16}$}
  \end{subfigure}
  \begin{subfigure}[t]{.24\linewidth}
    \centering
    \includegraphics[scale=.25]{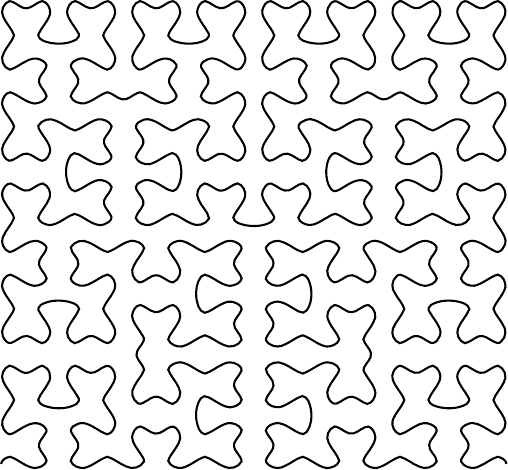}
    \caption{$\theta = 2^7/2^{16}$}
  \end{subfigure}
  \begin{subfigure}[t]{.24\linewidth}
    \centering
    \includegraphics[scale=.25]{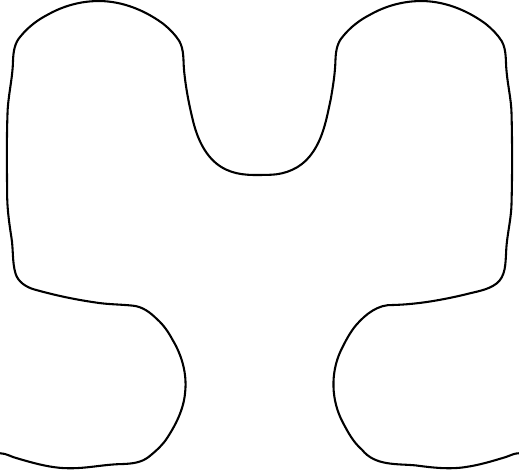}
    \caption{$\theta = 2^{12}/2^{16}$}
  \end{subfigure}
  \caption{Mollifying an $8$\textsuperscript{th}-order Hilbert curve
    by bump functions of varying width $\theta$.}
  \label{fig:mollifying-hc}
\end{figure}

\section{The ``Gradient'' of $\objfp$: The Barycenter Field.}
\label{sec:local-strategies}

We now turn to the ``gradient'' of the functional $\objfp$, which is
relevant, in particular, to local improvement of $\objfp(\fmap)$. For
this purpose, we define a vector field we call the \emph{barycenter
  field}, which is essentially a succinct way of encoding the first
variation of $\objfp$. Note that for now we will denote the barycenter
field simply by $\fbsymb$; later notation will include subscripts to
reflect dependencies on how one chooses to tiebreak ambiguities in the
closest-point projection. In any case, we discuss relationships
between $\fbsymb$ and existing work in
\cref{sec:barycenter-field-comparison} below; for now we summarize the
main takeaways about $\fbsymb$.

In general, $\fbsymb$ is defined on $\dset$ and depends on the choice
of variational perturbation $\pert$. The first main result of the
section, \cref{thm:fbary-grad-J}, verifies under mild hypotheses that
this $\fbsymb$ indeed encodes the first variation of $\objfp(\fmap)$.
Importantly, \cref{thm:fbary-grad-J} does not require any regularity
of $\fmap$, $\pert$ beyond continuity, and when $\objp > 1$, we need
no hypotheses on $\cmeas$ beyond $\cmeas \in \pmeas(\cset)$. When
$\objp = 1$ we must add a hypothesis that a set of ``problem points''
in the image set $\iset = \fmap(\dset)$ is $\cmeas$-null. This is
necessary due to the way we have defined $\fbsymb$, reflecting the
lower regularity of the $\objp = 1$ case.

The $\dset$-parametrized form of $\fbsymb$ used in
\cref{thm:fbary-grad-J} is an important theoretical tool. However, the
generality of the hypotheses necessitates defining $\fbsymb$ through
multiple implicit minimizations, some of which can make simulations
computationally intractable. Thus we also consider a simplified case:
When the set of self-intersections of $\fmap$ are null with respect to
a certain measure we may instead parametrize $\fbsymb$ by the image
$\iset = \fmap(\dset)$, and a further hypothesis that the
closest-point projection is $\cmeas$-a.e.\ well-defined allows us to
reduce $\fbsymb$ to the Frech\'et derivative of $\objfp$
(\cref{cor:fbary-grad-J-y}).

However, from a theoretical standpoint, even in the
$\iset$-parametrized case it is nontrivial to use $\fbsymb$ to derive
local improvements to a given $\fmap$. This is because $\fbsymb$ often
lacks regularity---it can be discontinuous even when $\iset$ is a
$C^1$-submanifold of $\cset$ (\cref{cex:Y-c1-fbary-discont})---and so
$\fmap + \varepsilon \fbsymb$ might not remain in the Sobolev class.
Nonetheless, in the second main result of the section
(\cref{thm:smooth-improvement}) we show that as long as $\fbsymb$ is
``nontrivial'' in some sense then there exists a smooth, local
perturbation improving $\objfp(\fmap)$. We conjecture that every
$\budg$-optimizer has this nontriviality property, whence the
monotonicity in \cref{thm:objell-properties} would become strict.
However, we were not able to find a proof.

The layout of the remainder of the section is as follows. In
\cref{sec:barycenter-field-comparison} we discuss similar ideas in the
literature, with a particular eye toward motivating the relaxed
hypotheses of \cref{thm:fbary-grad-J}. In \cref{sec:barycenter-field}
we introduce some measurable selection tools that are necessary for
defining $\fbsymb$ in this relaxed case and also list some basic
properties of $\fbsymb$, including its lack of regularity.
\cref{sec:fbary-negative-gradient} establishes the relationship
between $\fbsymb$ and the first variation of $\objfp$, and
\cref{sec:local-improve} shows that although $\fbsymb$ might lack
regularity, when $\fbsymb$ is ``nontrivial'' we may still use it to
obtain smooth, local improvements to $\objfp(\fmap)$.

\subsection{Motivation for The Barycenter Field and Comparison to
  Previous Techniques}
\label{sec:barycenter-field-comparison}

First, one might wonder why it is even necessary to appeal to the
first variation of $\objfp(\fmap)$ in order to obtain local
improvements. For example, suppose $\fmap$ is a curve (\ie\ $\ddim =
1$): Is it not possible to simply ``make the curve longer?''

If the special constant function $\reffmin$ defined in
\eqref{eq:fmin-def} is in $(\supp(\cmeas))^\circ$ and if $\budg \ll
\diam(\cset)$ then this will work. However, the general case is more
difficult. For $\budg$ sufficiently large it becomes possible that
$\iset \cap \partial \supp(\cmeas) \neq \varnothing$ in such a way
that extending $\fmap$ past its endpoints might yield no improvement
to $\objfp$ (for example, take $\cset = B_1(0)$, $\cmeas =
\mrm{Unif}_\cset$, and $\fmap$ a parametrization of a diameter).
Importantly, in such cases, the inclusion of higher-order terms in our
Sobolev constraint might prevent us from adding a ``sharp turn'' to
send the extended $\fmap$ back inside $\supp(\cmeas)$. Thus, to
analyze the Sobolev case we turn to the barycenter field $\fbsymb$.

As stated earlier, $\fbsymb$ is essentially just a succinct way of
encoding the first variation of $\objfp$ at $\fmap$ with respect to a
perturbation $\pert$; denote this by $\firstvar$. Note that
$\firstvar$ has been considered before in the literature, see \eg\
\cite{buttazzo-mainini-stepanov2009,gerber_dimensionality_2009,kirov-slepcev-2017,delattre2020}
and the references therein. However, we believe the importance of
studying $\firstvar$ in detail (and in particular, examining how it
may be used to construct local modifications) has gone largely
overlooked, as such analysis was unnecessary when using simpler
constraints like $\cost(\fmap) = \arclen(\fmap)$.

A common challenge in analyzing $\firstvar$ comes from the implicit
dependence of $\objfp$'s integrand $\inf_{\dpt} \abs{\cpt -
  \fmap(\dpt)}^\objp$ on the ``domain-projected'' closest-point map
$\pfd : \cset \to \dset$ given by
\begin{align}\label{eqn:pi-X-hat}
  \pfd(\cpt) \in \argmin_{\dpt \in \dset} \abs{\cpt - \fmap(\dpt)}.
\end{align}
This is problematic because $\pfd$ is not always well-defined due to
possible multiple minima. Denoting, as usual, $\iset = \fmap(\dset)$,
the two failure modes are made apparent by writing $\pfd = \fmap^{-1}
\circ \pfi$, where
\begin{align}\label{eqn:pi-Y-hat}
  \pfi(\cpt) \in  \argmin_{\ipt \in \iset} \abs{\cpt - \fmap(\ipt)}.
\end{align}
Namely,
\begin{enumerate}[leftmargin=4.5em, label=(FM\arabic*)]
  \item \label{fm:ambf} $\pfi$ fails to be uniquely-determined,
    and/or
  \item \label{fm:nonf} $\fmap^{-1}$ fails to be uniquely-determined.
\end{enumerate}
While these cases cause no problems when computing $\objfp(\fmap)$
itself, $\firstvar$ turns out to be sensitive to both.

To avoid \ref{fm:ambf} it is common to make some simplifying
assumptions: For example \ref{hyp:ac} \cite{kirov-slepcev-2017}; $\cmeas$
being null on $\set{A \subseteq \cset \MID \hdim(A) \leq \cdim - 1}$
\cite{buttazzo-mainini-stepanov2009}; or (via a clever invariant)
taking $\objp = 2$ and considering $\firstvar$ only for optimal
$\fmap$ \cite{delattre2020}. The first two simplifications have the
disadvantage of precluding analyzing $\firstvar$ when $\cmeas$ is
discrete (as is often the case in applications) and the third
precludes analyzing local modifications to suboptimal solutions (which
is useful in computational methods if one starts with a random
``seed'' function $\fmap$ and wants to ``evolve'' it in the direction
of the barycenter field to produce a better solution).

Similarly, to avoid \ref{fm:nonf} one can try to first establish that
optimizers must be injective and then restrict one's subsequent
analysis to just optimizers. However, in addition to yielding the same
difficulties as the third case above, the arguments in this vein are
nontrivial. Indeed, in the case $(\cost(\fmap), \ddim) =
(\arclen(\fmap), 1)$, the best result we are aware of establishes
injectivity only when $\objp = 2$ \cite{delattre2020}, so a similar
argument with the more-complicated Sobolev constraint does not seem
promising.

To address these concerns, in this section we show that (except
perhaps for certain configurations in the $\objp = 1$ case) one may
use the tools of measurable selections to encode $\firstvar$
succinctly even when \ref{fm:ambf} and \ref{fm:nonf} occur
simultaneously. Compared with previous results, this allows us to
(again except perhaps when $\objp=1$) extend the tools of calculus to
the whole constraint set $\set{\cost(\fmap) < \infty}$, and
importantly, to explicitly construct improvements to $\objfp$ at
nonstationary points. We reiterate that while our motivation came from
the peculiarities of the Sobolev constraint, our analysis of
$\firstvar$ in \cref{thm:fbary-grad-J} requires only continuity of
$\fmap$, $\pert$, and hence applies to common costs like $\cost(\fmap)
= \arclen(\fmap)$ as well; similarly, the local improvement result
\cref{thm:smooth-improvement} holds for any cost that satisfies
$\cost(\fmap + \pert) \lesssim \cost(\fmap) + \norm{\pert}_{C^\infty}$
for $\pert \in C^\infty$.

\subsection{The Barycenter Field}
\label{sec:barycenter-field}
First we introduce some preliminary concepts that will be essential to
our analysis.

\subsubsection{Measurable Selection and Disintegration}
\label{sec:disintegration-details}

As usual, we use the shorthand $\iset = \fmap(\dset)$. Recall
closest-point projection maps to $\dset$ and $\iset$, respectively,
given in \eqref{eqn:pi-X-hat} and \eqref {eqn:pi-Y-hat}. As detailed
in \ref{fm:ambf}, \ref{fm:nonf}, $\pfd$ and $\pfi$ generally fail to
be well-defined in some parts of $\cset$; thus we begin by
constructing measurable extensions. Define the set-valued functions
$\pfds : \cset \to 2^\dset$ and $\pfis : \cset \to 2^\iset$ by
\[
  \pfds(\cpt) = \argmin_{\dpt \in \dset}  \abs{\cpt - \fmap(\dpt)}
  \qquad \text{and} \qquad
  \pfis(\cpt) = \argmin_{\ipt \in \iset}  \abs{\cpt - \ipt}
\]
Following the terminology of \cite{hastie1989} we define the
\emph{ambiguity set} of $\fmap$ to be
\begin{equation}
  \ambf = \set{\cpt \in \cset \MID \pfis(\cpt) \text{ is not a
      singleton}}. \label{eq:ambf-def}
\end{equation}
Note, $\refambf$ is exactly the set of points $\cpt$ where \ref{fm:ambf}
occurs.

The set $\refambf$ is closely related to the notion of the \emph{medial
  axis}, which is commonly defined as follows: Given an open $\medset
\subseteq \rcdim$ with $\partial \medset \neq \varnothing$, the
\emph{medial axis} of $E$ is
\[
  \textstyle \medax(\medset) = \set{\medpt \in \medset \MID
    \argmin_{\medpt' \in \partial \medset} \abs{\medpt - \medpt'}
    \text{ is not unique}}.
\]
Thus we see $\refambf = \cset \cap \medax(\rcdim \setminus \iset)$.
$\medax(\medset)$ is also called the \emph{ambiguous locus} or
\emph{skeleton} \cite{choi1997,Hajlasz2021,miura2016}. In some
applications $\medax(\medset)$ is treated interchangeably with its
closure $\ol{\medax(\medset)}$ (called the \emph{ridge set} or
\emph{cut locus}), though this is technically improper in the abstract
setting as $\ol{\medax(\medset)}$ can behave much more pathologically
than $\medax(\medset)$ in cases where $\partial \medset$ lacks
regularity \cite{miura2016}. In any case, $\medax(\rcdim \setminus
\iset)$ is guaranteed to be Lebesgue-null \cite{Hajlasz2021}, which is
important in the second part of \cref{cor:pf-meas} below.

We now recall a result that will serve as the basic tool for our
measurable selection statements. A proof may be found in \cite[Prop.\
7.33]{stochoptcont}.

\begin{proposition}[Measurable selection]
  \label{prop:measurable-selection}

  Suppose $\cset$ is a metrizable space, $\optset$ a compact
  metrizable space, $\seldom$ a closed subset of $\cset \times
  \optset$, and $\lscf : \seldom \to \RR \cup \set{-\infty, \infty}$ a
  lower semicontinuous function. For all $\cpt$ let $\seldomc =
  \set{\optpt \MID (\cpt, \optpt) \in \seldom}$ and let $h^* :
  \mrm{proj}_{\cset}(D) \to \RR \cup \set{-\infty, \infty}$ be given
  by
  \[
    \lscfm(\cpt) = \min_{\optpt \in \seldomc} h(\cpt, \optpt).
  \]
  Then $\projc(\seldom)$ is closed in $\cset$,  $\lscfm$ is lower
  semicontinuous, and there exists a Borel-measurable function
  $\varphi : \projc(\seldom) \to \optset$ such that $\graph(\varphi)
  \subseteq \seldom$ and for all $\cpt \in \projc(\seldom)$
  \[
    \lscf (\cpt, \varphi(\cpt)) = \lscfm(\cpt).
  \]
  In this case we call $\varphi$ a \emph{measurable selection} of the
  set-valued function $\cpt \mapsto \seldomc$.
\end{proposition}

Applying this we can find measurable selections for the projections
$\pfds$, $\pfis$.

\begin{corollary}
  \label{cor:pf-meas}

  Suppose $\dset$ is compact \ref{hyp:dset} and $\fmap \in C(\dset;
  \cset)$ (\eg\ via \cref{prop:gen-sob}). Then there exist
  Borel-measurable selections $\pfd$ of $\pfds$ and $\pfi$ of $\pfis$
  that are well-defined on all of $\cset$. Furthermore, under
  \ref{hyp:ac}, $\pfi$ is $\cmeas$-essentially unique.
\end{corollary}
\begin{proof}
  Existence of $\pfd$ follows immediately from
  \cref{prop:measurable-selection} with $\seldom = \cset \times \dset$
  and $\lscf(\cpt, \dpt) = \abs{\cpt - \fmap(\dpt)}$; we then obtain
  $\pfi$ by taking $\pfi(\cpt) \coloneqq \fmap(\pfd(\cpt))$. For the
  second part, noting $\refambf \subseteq \medax(\rcdim \setminus \iset)$
  we see $\refambf$ is Lebesgue null \cite{Hajlasz2021}, whence under
  \ref{hyp:ac} we get $\cmeas(\refambf) = 0$. So $\pfi$ is
  $\cmeas$-essentially unique.
\end{proof}
In light of \cref{cor:pf-meas} we define
\begin{equation}
  \label{eq:pfdset-def}
  \pfdset = \set{\pfd \MID \pfd \text{ is a measurable selection of
    } \pfds}
\end{equation}
and
\begin{equation}
  \label{eq:pfiset-def}
  \pfiset = \set{\pfi \MID \pfi \text{ is a measurable selection of
    } \pfis}.
\end{equation}
With this we now apply disintegration of measures to decompose
$\cmeas$ by $\pfd \in \refpfdset$ into $(\set[]{\cmdpthm}_{\dpt \in
  \dset}, \cpushdh)$, and by $\pfi \in \refpfiset$ into
$(\set[]{\cmipthm}_{\ipt \in \iset}, \cpushih)$, where
$\cpushdh$-a.s.\ $\supp(\cmdpthm) = \pfdinv(\dpt)$ and
$\cpushih$-a.s.\ $\supp(\cmipthm) = \pfiinv(\ipt)$. Thus we may
express $\objfp$ as
\begin{align}
  \objfp(\fmap)
  &= \int_{\dset}
    \bk[bigg]{\int_{\pfd^{-1}(\dpt)} \abs{\cpt - \fmap(\dpt)}^\objp \dd
    \cmdpthm(\cpt)} \dd \cpushdh(\dpt) \label{eq:disintegration-dset} \\
  &= \int_{\fmap(\dset)} \bk[bigg]{\int_{\pfi^{-1}(\ipt)} \abs{\cpt -
    \ipt}^\objp \dd \cmipthm(\cpt)} \dd
    \cpushih(\ipt), \label{eq:disintegration-iset}
\end{align}
with the representation in \eqref{eq:disintegration-iset} being unique
if one takes \ref{hyp:ac}.

\subsubsection{The Barycenter Field and Its Properties}
\label{sec:barycenter-field-and-properties}
\begin{definition}[Barycenter Field]
  \label{def:fbary}
  Take \ref{hyp:p}. For fixed $\pfd \in \refpfdset$ we define
  \[
    \fbdh(\dpt) = \int_{\pfdinv(\dpt)} \objp \abs{\cpt -
      \fmap(\dpt)}^{\objp - 2} (\cpt - \fmap(\dpt)) \dd \cmdpt(\cpt),
  \]
  taking the convention that $\abs{v}^{\objp - 2} v = 0$ when $v = 0$.
  Similarly, for fixed $\pfi \in \refpfiset$ we let $\fbih(\ipt) =
  \int_{\pfiinv(\ipt)} \objp \abs{\cpt - \ipt}^{\objp -2} (\cpt -
  \ipt) \dd \cmipt(\cpt)$.
\end{definition}
\begin{boxedremark}
  1) $\fbdh$, $\fbih$ are Borel-measurable.
  2) Note that the integrands of $\fbdh$ and $\fbih$ are just the
  derivatives of the integrands in \eqref{eq:disintegration-dset},
  \eqref{eq:disintegration-iset}; this is the standard outcome we
  expect in the context of gradient flows.
  3) Whenever $\fbdh$, $\fbih$ are the same independent of the choice
  of $\pfd, \pfi$ (see \cref{cor:fbary-grad-J-y}), we will simply
  write $\fbd$, $\fbi$, respectively to denote the parametrizations
    in terms of $\dset$, $\iset$.
\end{boxedremark}
We now examine some basic geometric properties of the barycenter
field. In general, these are most easily visualized for $\fbi$. One
should also note that some statements like
\cref{prop:fiber-characterization} hold only for $\fbi$ and lack
analogues for $\fbd$ on account of the fact that $\fmap^{-1}$ can be
discontinuous when it is not well-determined.

We recall two definitions that will be helpful. Let $\tangset
\subseteq \rcdim$ and let $\tangpt \in \tangset$. Then the
\emph{tangent cone} at $\tangpt$ (denoted via the abuse of notation
$\taa$) is given by
\[
  \taa = \set{v \in \rcdim \MID \textstyle v = 0 \text{ or } \exists
    \text{ a seq.\ } \set{a_j} \subseteq A \setminus \set{a} \st a_j
    \to a \text { and } \frac{a_j - a}{\abs{a_j - a}} \to \widehat v}
\]
where $\widehat v$ denotes the unit vector $\frac{v}{|v|}$. Now let
$\pcset \subseteq \rcdim$. Then the \emph{polar cone} of $\pcset$ is
given by
\[
  \pcone(\pcset) = \set{v \in \rcdim \MID \forall \pcpt \in \pcset,\
    \ip{v, \pcpt} \leq 0}.
\]
With these in hand, we can characterize the fibers of $\pfi$ more
fully. Note we denote Minkowski addition by $+$.
\begin{proposition}
  \label{prop:fiber-characterization}
  Let $\ipt \in \iset$ be arbitrary. Then $\pfiinv(\ipt)$ is convex,
  and $\pfiinv(\ipt) \subseteq \ipt + \pcone(\tyy)$.
\end{proposition}
\begin{proof}
  For all $\ipt_0, \ipt_1 \in \iset$ with $\ipt_0 \neq \ipt_1$, define
  the half-space $\halfspace = \{u \in \rcdim \MID \abs{u - \ipt_0} <
  \abs{u - \ipt_1}\}$, changing the $<$ to an $\leq$ as needed to
  facilitate tiebreaking when $u \in \refambf$ and $\ipt_0, \ipt_1 \in
  \pfis(u)$. Then we have $\pfiinv(\ipt_0) = \bigcap_{\ipt_1 \neq
    \ipt_0} \halfspace \cap \cset$, which is an intersection of convex
  sets and hence convex.

  For the second part, let $\cpt \in \pfiinv(\ipt_0)$ and $v_0 \in
  T_{\ipt_0} \iset$ be arbitrarily chosen. We want to show $\ip{\cpt -
    \ipt_0, v_0} \leq 0$. If $v_0 = 0$ then the claim is trivial.
  Hence suppose $v_0 \neq 0$. Then there exists a sequence $\ipt_j \to
  \ipt_0$ such that $\frac{\ipt_j - \ipt}{\abs{\ipt_j - \ipt}} \to
  \hat v_0$. Let $\halfptj = (\ipt_j + \ipt_0)/2$ and $\vjhat =
  (\ipt_j - \ipt_0)/\abs{\ipt_j - \ipt_0}$. One may rewrite
  $\halfspacej = \set{u \in \rcdim \MID \ip[]{u - \halfptj, \vjhat} <
    0}$, again replacing the $<$ with $\leq$ as needed to facilitate
  tiebreaking. For each $j$, since $\cpt \in \halfspacej$ we get
  $\ip[]{\cpt - \halfptj, \vjhat} < 0$. As $j \to \infty$ we see
  $\halfptj \to \ipt_0$ and $\vjhat \to \hat v_0$, thus $\ip[]{\cpt -
    \ipt_0, \hat v_0} \leq 0$. Since $v_0$ was arbitrary it follows
  that $\cpt - \ipt_0 \in \pcone(T_{\ipt_0} \iset)$ from which the
  claim follows.
\end{proof}
In particular, \cref{prop:fiber-characterization} implies that for all
$\dpt$, $\fbd(\dpt) \in \fmap(\dpt) + \pcone(T_{\fmap(\dpt)} \iset)$;
similarly, for all $\ipt$, $\fbi(\ipt) \in \ipt + \pcone(T_\ipt
\iset)$. However, note that if $\fmap$ fails to be injective, then
unlike $\pfiinv(\ipt)$, the fiber $\pfdinv(\dpt)$ could fail to be
convex. This is because a priori we can only say $\pfiinv(\ipt) =
\bigsqcup_{\dpt \in \finv(\ipt)} \pfdinv(\dpt)$; there is no structure
enforced on exactly how the $\cpt \in \pfiinv(\ipt)$ are partitioned
between the $\pfdinv(\dpt)$.

Next, one might ask whether we can expect any regularity from $\fbd,
\fbi$. Generally, the answer is no. Trivial examples arise when
$\cmeas$ is singular, for example $\cmeas = \delta_\cpt$ for some
$\cpt \not \in \iset$. When $\cmeas$ is absolutely continuous, easy
examples can be obtained when the fibers $\pfdinv(\dpt)$
(respectively, $\pfiinv(\ipt)$) have nonconstant Hausdorff dimension,
such as when $\fmap(\partial \dset) \subseteq \supp(\cmeas)^\circ$ and
$\fmap^{-1}$ is well-defined on $\fmap(\partial \dset)$. For instance,
consider $\cmeas$ uniform on $\ol{B_1(0; \rcdim)}$ with $\fmap :
\ol{B_1(0; \rddim)} \to \ol{B_1(0; \rcdim)}$ given by $\fmap(\dpt) =
\frac{9}{10} \iota(\dpt)$ where $\iota$ is the inclusion map. Then
$\fbd \equiv 0$ on $B_1(0; \rddim)$ but points radially outward on
$\partial \ol{B_1(0; \rddim)}$; analogously for $\fbi$ on $B_{.9}(0;
\rcdim)$ and $\partial \ol{B_{.9}(0; \rcdim)}$.

However, there are actually trivial counterexamples with $\fmap \in
C^\infty$ even when $\partial \fmap(\dset) \subseteq \partial
\supp(\cmeas)$. This might sound surprising at first but becomes less
so once one remembers that smoothness of $\fmap$ as a map is not the
same thing as smoothness of $\fmap(\dset)$ as a manifold. An
elementary counterexample can be constructed by taking a smooth
parametrization of the graph of $\dpt \mapsto \abs{\dpt}$ as we now show.
\begin{counterexample}[$\fmap \in C^\infty$; Barycenter Field discontinuous]
  \label{cex:f-smooth-fbary-discont}
  Let $\cset$ be a diamond in $\RR^2$ with vertices $(0,-2), (-2, 0),
  (0,2)$, and $(2,0)$ endowed with $\cunif$. Recall that in the proof
  of \cref{lem:arbitrarily-bad-reparametrizations} we used
  $\smoothstep(\dpt)$ to denote a smooth step function. We similarly
  define a smooth ``dip'' function $\smoothdip : \RR \to [0,1]$ by
  \[
    \smoothdip(\dpt) =
    \begin{cases}
      -\smoothstep(\dpt + 1) + 1 & \dpt \in (-\infty, 0] \\
      \ph \smoothstep(\dpt) & \dpt \in [0, \infty).
    \end{cases}
  \]
  Observe $\smoothdip \equiv 1$ on $\RR \setminus [-1,1]$, and that
  its derivatives of all orders vanish at $0$. Finally, let $\fmap :
  [-1,1] \to \cset$ be given by $\fmap(\dpt) = (\smoothdip(\dpt)
  \sgn(\dpt),\ \smoothdip(\dpt))$. Then $\fmap \in C^\infty$ but
  $\fbd$ (respectively, $\fbi$) is discontinuous at $0$ (respectively,
  $(0,0)$).
\end{counterexample}
Much more surprising is that even $C^1$-ness of $\fmap(\dset)$ as a
manifold is insufficient to guarantee continuity of $\fbd$, $\fbi$.
\begin{counterexample}[$\fmap(\dset)$ a $C^1$-manifold; Barycenter
  Field discontinuous]
  \label{cex:Y-c1-fbary-discont}
  Let $\dset = [-1,1]$, $\cset = [-1,1]^2$, and $\sobk=1$.  Define
  \[
    \fmap_2(t) =
    \begin{cases}
      t^3 \sin(1/t) & t < 0 \\
      0 & t \geq 0
    \end{cases}
  \]
  and let $\fmap : \dset \to \cset$ be given by $\fmap(t) = (t,
  \fmap_2(t))$. Observe that $\fmap$ is injective and $C^1$ with
  $\fmap'(t) \neq 0$, so $\fmap([-1,1])$ is a $C^1$-manifold.
  \begin{figure}[H]
    \centering
    \includegraphics[scale=.8]{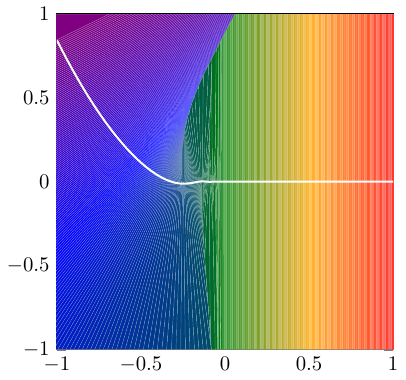}
    \caption{$\fmap$ and the $\pfiinv(\ipt)$ visualized.}
    \label{fig:topsin-fibs-visualized}
  \end{figure}
  We claim $\fbi$ is discontinuous at $(0,0)$; in particular $\lim_{t
    \to 0} \abs{\fbi(t)}$ does not exist. To that end note $\cset
  \setminus \iset$ has two connected components; call the upper
  $\cset_1$ and the lower $\cset_2$. For each $\ipt \in \fmap(\dset)$
  let $L_1(t) = \mrm{Length}(\pfiinv(\ipt) \cap \cset_1)$ and $L_2(t)
  = \mrm{Length}(\pfiinv(\ipt) \cap \cset_2)$. For concision, we use
  the notation $\ipt_t = \fmap(t)$.

  From now on $t_0 \in \pn{-1,0}$ will always denote a local maximum
  of $\fmap_2$. Note that because $\fmap_2''(t_0) < 0$,
  $\cmeas_{\ipt_{t_0}}$ is nonuniform on $\pfiinv(\ipt_{t_0})$, and in
  particular $\abs{\fbi(\ipt_{t_0})} > L_1(t_0) - L_2(t_0)$. So, to
  show $\fbi$ is discontinuous at $0$ it suffices to show that as $t_0
  \to 0$, $L_2(t_0) \to 0$ while $L_1(t_0)$ is bounded below by some
  $c_0 > 0$. To treat $L_2(t_0)$: Let $R(t)$ denote the radius of
  curvature, where defined. We have
  \begin{align*}
    R(t) = \abs{\frac{\pn[]{1 + \dot \fmap_2(t)^2}^{3/2}}{\ddot
    \fmap_2(t)}}
    &= \abs{\frac{\pn[]{1 + \pn{t \pn{3t\sin(1/t) -
      \cos(1/t)}}^{2}}^{3/2}}{\frac{6t^2-1}{t} \sin(1/t) - 4
      \cos(1/t)}}.
  \end{align*}
  $\iset$ is locally a $C^2$-manifold at $\fmap(t_0)$, so by
  \cite[Thm.\ 1.2]{miura2016}, $L_2(t_0) \leq R(t_0)$. Since $\dot
  \fmap_2(t_0) = 0$, for small $t_0$ one may verify $R(t_0) \approx
  \abs{\frac{t_0}{\sin(1/t_0)}} \approx \abs{t_0}$, so $\lim_{t_0 \to
    0} L_2(t_0) = 0$.

  Now we claim $\lim_{t_0 \to 0} L_1(t_0) > 0$; numerically it appears
  the value is $\approx 0.6671$. Hence for simplicity let $c_0 = 1/2$
  and consider $\cpt = (t_0, c_0)$. We claim $\pfi(\cpt) =
  \fmap(t_0)$.
  \begin{figure}[H]
    \centering
    \includegraphics[scale=.35]{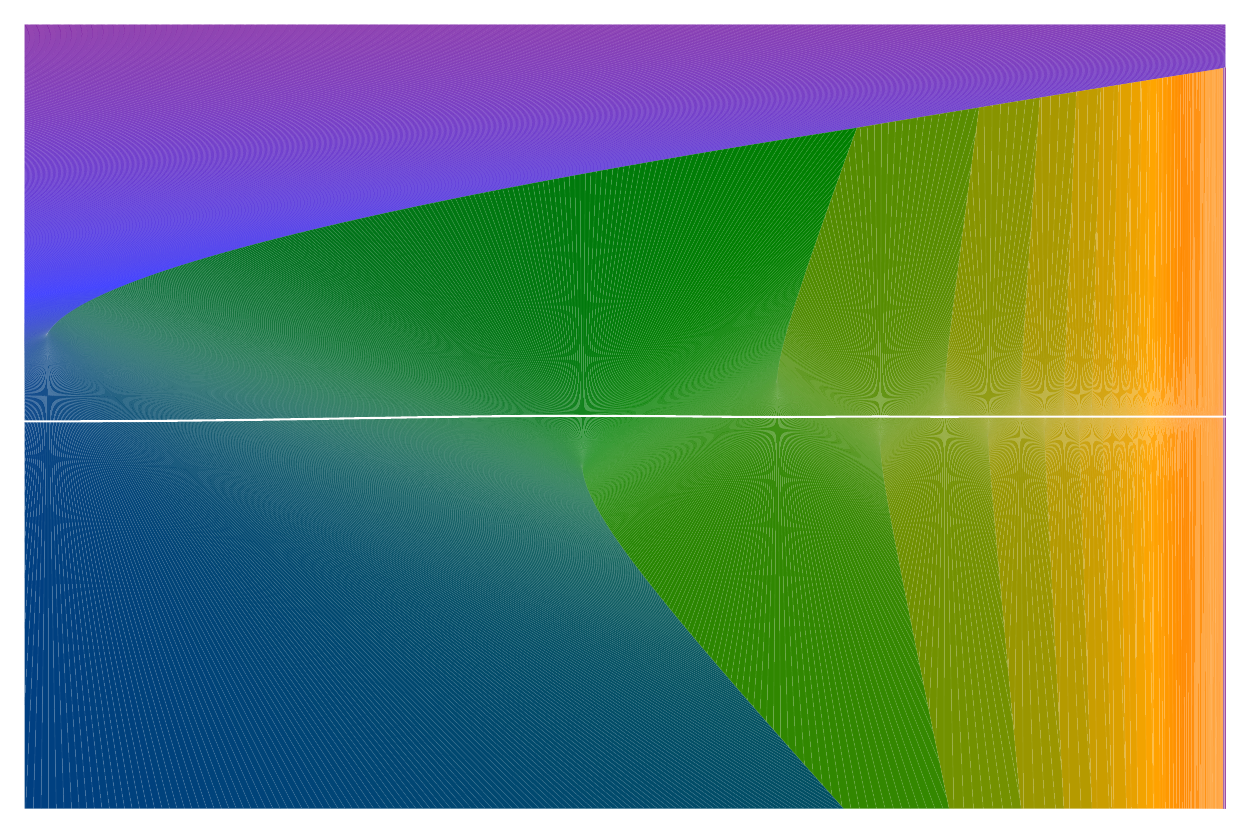}
    \caption{Closeup of $\pfiinv(\ipt)$ for $-1 \ll t < 0$ (axes not
      scaled equally).}
  \end{figure}
  To see this note $\pfi(\cpt)$ minimizes $d^2(\cpt, \fmap(t)) = (t -
  t_0)^2 + (\fmap_2(t) - c_0)^2$. Thus either $t=\pm 1$ or $0 = 2(t -
  t_0) + t \cos(1/t) + O(t^2)$. By direct comparison, one may verify
  that we can rule out $t = \pm 1$ as well as $t > 0$. Next one may
  show that for small $t_0$ the $2(t - t_0)$ term dominates, whence
  critical points will occur in some neighborhood $N_{t_0}$ of $t_0$
  with $\diam(N_{t_0}) = O(\abs{t_0})$. For some $\varepsilon$
  sufficiently small, $t_0$ is the only critical point in
  $B_\varepsilon(t_0)$. Now suppose $t'$ is a critical point not in
  $B_\varepsilon(t_0)$. We must have $t' < t_0$ or else $d^2(\cpt,
  \fmap(t')) > d^2(\cpt, \fmap(t_0))$. Define $\delta = \abs{t' -
    t_0}$; observe $\delta \in O(\abs{t_0})$ since $\diam(N_{t_0})$
  is. Now
  \begin{align*}
    \abs{\fmap(t') - \fmap(t_0)}
    &= \abs{t'^3 \sin(1/t') - t_0^3 \sin(1/t)} \\
    &\leq \abs{(t_0 - \delta)^3 \sin(1/t') - t_0^3 \sin(1/t_0)} \\
    &= O(t_0^3).
  \end{align*}
  Thus
  \begin{align*}
    d^2(\cpt, \fmap(t'))
    &= (t' - t_0)^2 + (\fmap(t') - c_0)^2 \\
    &= \delta^2 + (\fmap(t') - \fmap(t_0) + \fmap(t_0) - c_0)^2 \\
    &= \delta^2 + (\fmap(t_0) - c_0)^2 + 2 (\fmap(t_0) - c_0)
      (\fmap(t') - \fmap(t_0)) + (\fmap(t') - \fmap(t_0))^2 \\
    &= \delta^2 + d^2(\cpt, \fmap(t_0)) + O(t_0^3).
  \end{align*}
  So as $t_0 \to 0$ we see $d^2(\cpt, \fmap(t')) > d^2(\cpt,
  \fmap(t_0))$. Finally, comparing to $d^2(\cpt, \fmap(-1))$ and
  $d^2(\cpt, \fmap(1))$ we see $\fmap(t_0)$ achieves the global
  minimum.

  Thus we see for points of local maxima $t_0$, $\fbi(\ipt_{t_0}) \not
  \to 0$ as $t_0 \to 0^-$. By contrast, $\lim_{t\to 0^+} \fbi(\ipt_t)
  = 0$. So $\fbi$ is discontinuous at $(0,0)$, despite the fact that
  $\iset$ is a local $C^1$-manifold there. Since $\fmap$ is injective
  and $\cmeas \ll \leb$, $\fbi$ and $\fbd$ are equivalent and so
  $\fbd$ is similarly discontinuous.
\end{counterexample}

\subsection{$F_{\hat{\pi}_X}$ as a (Negative) Gradient}
\label{sec:fbary-negative-gradient}
We now move to proving the first main result of this section. As
mentioned previously the $\objp = 1$ case is special, which reflects
the fact that the integrand $\abs{\cdot}$ of $\objf_1$ is only
Lipschitz, as compared to the $C^1$ integrand $\abs{\cdot}^\objp$ of
$\objfp$ for $\objp > 1$.
\begin{theorem}
  \label{thm:fbary-grad-J}
  Take \ref{hyp:dset}, \ref{hyp:cset}, and \ref{hyp:p}. Fix $\fmap \in
  C(\dset; \cset)$ and $\pert \in C(\dset; \rcdim)$, and for all
  $\varepsilon > 0$ let $\feps \coloneqq \fmap + \varepsilon \pert$
  and $\iset = \fmap (\dset)$, $\iseps=\feps(\dset)$. If $\objp = 1$,
  let $\iset_{\rm out} = \set{\ipt \in \iset \MID \liminf_{\varepsilon
      \to 0} 1_{\iseps}(\ipt) = 0}$, and take the extra hypothesis
  $\cmeas(\iset_{\rm out}) = 0$. Recall the definition of $\refpfdset$
  from \eqref{eq:pfdset-def}. Then
  \begin{equation}
    \label{eq:grad-J-x}
    \lim_{\varepsilon \to 0}
    \frac{\objfp(\feps) - \objfp(\fmap)}{\varepsilon} =
    \min_{\pfd \in \refpfdset} \int_\dset
    \ip[]{-\fbdh(\dpt), \pert(\dpt)} \dd
    \cpushdh(\dpt).
  \end{equation}
\end{theorem}

\begin{proof}[Proof of \cref{thm:fbary-grad-J}]
  For concision write
  \begin{align*}
    \frac{ (\objfp(\feps) - \objfp(\fmap))}{\varepsilon}
    =& \int_\cset \delta_\varepsilon (\cpt) \dd \cmeas (\cpt), \text{
       where} \\
    \delta_\varepsilon (\cpt)
    \coloneqq& \frac{1}{\varepsilon}\left[
               \inf_{\dpt_{\varepsilon} \in \dset} \abs{\cpt -
               \feps(\dpt_{\varepsilon})}^\objp - \inf_{\dpt \in
               \dset} \abs{\cpt - \fmap(\dpt)}^\objp\right].
  \end{align*}
  Also, let
  \[
    \lscf(\cpt, \dpt) =
    \begin{cases}
      \ip[]{-\objp \abs{\cpt - \fmap(\dpt)}^{\objp -
      2} (\cpt - \fmap(\dpt)),\ \pert(\dpt)}
      & \text{if } \cpt \neq \fmap(\dpt), \text{ and} \\
      0
      & \text{if } \cpt = \fmap(\dpt).
    \end{cases}
  \]
  As we will show, the formula for $\lscf$ is related to
  $\ip[]{-\fbdh(\dpt), \pert(\dpt)}$ in \eqref{eq:grad-J-x}. Notice
  that $\lscf$ depends on $\pert$.

  \textbf{Case 1:} Suppose $\objp > 1$. We will apply
  \cref{prop:measurable-selection} to obtain a measurable selection of
  $\pfds$. To that end, define $\seldom = \set[]{(\cpt, \dpt) \MID
    \dpt \in \pfds(\cpt)}$, $\seldomc = \set[]{\dpt \MID (\cpt, \dpt)
    \in \seldom}$, and let $$\lscfm(\cpt) = \min_{\seldomc}
  \lscf(\cpt, \dpt).$$ Note that since $\objp > 1$, $\lscf$ is
  continuous, and so l.s.c. Moreover, it is easy to verify that
  $\seldom$ is closed. Thus \cref{prop:measurable-selection} gives a
  measurable selection $\pfdm$ of $\pfds$ with
  \[
    \lscf(\cpt, \pfdm(\cpt)) = \lscfm(\cpt).
  \]
  Notice that by the definition of $\pfdm(\cpt)$, disintegrating
  $\int_\cset \lscf(\cpt, \pfdm(\cpt) ) \dd \cmeas$ gives the
  right-hand side of\eqref{eq:grad-J-x}:
  \begin{align}\label{eq:h-*-min-pi}
   \int_\cset \lscfm(\cpt) \dd\cmeas =   \int_\cset \lscf(\cpt, \pfdm(\cpt)) \dd \cmeas =
    \min_{\pfd \in \refpfdset} \int_\dset
    \ip[]{-\fbdh(\dpt), \pert(\dpt)} \dd
    \cpushdh(\dpt).
  \end{align}

  For the rest it suffices to show that
  \begin{equation*}
    \text{for $\cmeas$-a.e.\ $\cpt$, }
    \lim_{\varepsilon\to 0} \delta_\varepsilon (\cpt) =  \lscfm(\cpt).
  \end{equation*}
  (Then, we can apply the dominated convergence theorem to get
  \eqref{eq:grad-J-x}.)

  Analogously to the construction of $\pfdm$, define the set-valued
  map $\pfdepss(\cpt) = \argmin_{\dpt \in \dset} \abs{\cpt -
    \feps(\dpt)}$ and let $\seldomeps = \set{(\cpt, \dpt) \MID \dpt
    \in \pfdepss(\cpt)}$. Then \cref{prop:measurable-selection} yields
  a selection $\pfdepsm$ of $\pfdepss$ with $\lscf(\cpt,
  \pfdepsm(\cpt)) = \lscfepsm(\cpt)$.

  For concision define
  \begin{align*}
    \gmap(\cpt) = \cpt - \fmap(\pfdm(\cpt))
    \quad \text{and} \quad
    \geps(\cpt) = \cpt - \fmap(\pfdepsm(\cpt)).
  \end{align*}
  Since $\pfds$, $\pfdepss$ were defined via closest-point projection
  onto $\iset$, $\iseps$, and since $\pfdm$, $\pfdepsm$ are selections
  of $\pfds$, $\pfdepss$, respectively, it follows that
  \begin{align*}
    \abs{\gmap} \le \abs{\geps} \quad \text{and} \quad \abs{\geps -
    \varepsilon \pert(\pfdepsm)}\le \abs{\gmap - \varepsilon
    \pert(\pfdm)}.
  \end{align*}
  Therefore,
  \begin{equation}
    \underbrace{\frac{\abs[]{\geps - \varepsilon
          \pert(\pfdepsm)}^{\objp} -
        \abs{\geps}^{\objp}}{\varepsilon}}_{\textrm{(I)}}
    \leq \delta_\varepsilon \leq
    \underbrace{\frac{\abs{\gmap - \varepsilon \pert(\pfdm)}^\objp -
        \abs{\gmap}^{\objp}}{\varepsilon}}_{\textrm{(II)}}.
    \label{eq:lower-upper-bounds-delta}
  \end{equation}

  We now rewrite the (I) and (II) in terms of (respectively)
  $\lscfepsm(\cpt)$ and $\lscfm(\cpt)$, up to an error term that
  vanishes as $\varepsilon \to 0$. Note that for all $a,b \in \RR^n$
  with $a \neq 0$, Taylor expanding $\varepsilon \mapsto \abs{a -
    \varepsilon b}^\objp = \pn{\sum_i (a_i - \varepsilon
    b_i)^2}^{\objp/2}$ about $\varepsilon = 0$ yields
  \begin{align*}
    \abs{a - \varepsilon b}^{\objp}
    = \abs{a}^\objp + \varepsilon \cdot {\ip[big]{-\objp
      \abs{a}^{\objp - 2} a, b}} +
      o(\varepsilon).
  \end{align*}
  Thus if $\geps(\cpt)\neq 0$ then by the definition of $\lscf$,
  $\textrm{(I)} = \lscf(\cpt, \pfdepsm) + o(1)$, and similarly if
  $\gfunc(\cpt) \neq 0$ then $\textrm{(II)} = \lscf(\cpt, \pfdm) +
  o(1)$.

  On the other hand when $a = 0$ we get $\abs{a - \varepsilon
    b}^{\objp} = \varepsilon^\objp \abs{b}^{\objp}$; thus if
  $\geps(\cpt) = 0$ then $\textrm{(I)} = \varepsilon^{\objp - 1}
  \abs[]{\pert(\pfdepsm)}^\objp$ and if $\gfunc(\cpt) = 0$ then
  $\textrm{(II)} = \varepsilon^{\objp - 1} \abs{\pert(\pfdm)}^\objp$.
  Observe that since $\pert$ is bounded \ref{hyp:dset} and $\objp > 1$,
  both of these expressions are $o(1)$.

  So, in either case, $\textrm{(I)} = \lscf(\cpt, \pfdepsm) + o(1)$
  and $\textrm{(II)} = \lscf(\cpt, \pfdm) + o(1)$. Thus from
  \eqref{eq:lower-upper-bounds-delta} we get
  \begin{equation}
    \lscfepsm(\cpt) + o(1)
    \leq \delta_\varepsilon (\cpt)
    \leq \lscfm(\cpt) + o(1). \label{eq:pre-squeeze-theorem}
  \end{equation}
  To conclude it suffices to show that for each $\cpt$, $\lscfm(\cpt)
  \leq \liminf_{\varepsilon \to 0} \lscfepsm(\cpt)$. But this follows
  from the continuity of $(\cpt, \dpt) \mapsto \lscf(\cpt, \dpt)$ and
  $\varepsilon\mapsto \feps$. This concludes the $\objp > 1$ case.

  \medskip

  \textbf{Case 2:} Suppose $\objp = 1$. In this case, $\lscf$ is no
  longer globally l.s.c.

  Consider $\cset_{\rm out} = \cset \setminus \iset$. Define
  $\seldom$, $\seldomeps$, $\lscf$, and so on as in the $\objp > 1$
  case but replacing $\cset$ with $\cset_{\rm out}$. Note that while
  $\cset_{\rm out}$ is not a closed subset of $\cset$, we nevertheless
  see that $\seldom$, $\seldomeps$ are closed with respect to the
  subspace topology on $\cset_{\rm out} \times \dset$. Since we also
  have that $\lscf$ is continuous on $\cset_{\rm out} \times \dset$,
  the hypotheses of \cref{prop:measurable-selection} are satisfied and
  as in the $\objp > 1$ case we obtain measurable selections $\pfdm$,
  $\pfdepsm$ such that
  \begin{equation*}
    \text{for all } \cpt \in \cset_{\rm out}, \quad
    \lscfm(\cpt) = \lscf(\cpt, \pfdm)
    \quad \text{and} \quad
    \lscfepsm(\cpt) = \lscf(\cpt, \pfdepsm).
  \end{equation*}
  With this we recover \eqref{eq:h-*-min-pi} as follows. Extend
  $\pfdm$ to also be defined on $\iset$ via an arbitrary measurable
  selection of the (potentially) set-valued map $\cpt \mapsto
  \fmap^{-1}(\cpt)$. Regardless of our choice, $\lscfm \equiv 0$ on
  $\iset$, so the $\cset$ integrals in \eqref{eq:h-*-min-pi} are
  independent of the extension. Similarly, since we took the
  convention $\widehat{\bm 0} = \bm 0$ in the definition of $\fbdh$
  (\cref{def:fbary}), the right-hand side is also independent of the
  extension. Thus we obtain \eqref{eq:h-*-min-pi} as in the $\objp >
  1$ case.

  As in the $\objp > 1$ case it now suffices to show
  \begin{equation*}
    \text{for $\cmeas$-a.e.\ $\cpt$, }
    \lim_{\varepsilon\to 0} \delta_\varepsilon (\cpt) = \lscfm(\cpt)
  \end{equation*}
  Recall that by hypothesis $\cmeas(\iset_{\rm out}) = 0$, so in fact
  it suffices to just verify the limit by casework on $\cset_{\rm
    out}$ and $\iset_{\rm in} \coloneqq \iset \setminus \iset_{\rm
    out}$. Hence, let $\cpt \in \cset_{\rm out}$ be arbitrarily
  chosen. As in the $\objp > 1$ case we obtain
  \eqref{eq:lower-upper-bounds-delta}, that is,
  \begin{align*}
    \underbrace{\frac{\abs[]{\geps - \varepsilon
    \pert(\pfdepsm)}^{\objp} -
    \abs{\geps}^{\objp}}{\varepsilon}}_{\textrm{(I)}}
    \leq
    \delta_\varepsilon   \leq
    \underbrace{\frac{\abs{\gmap - \varepsilon \pert(\pfdm)}^\objp -
    \abs{\gmap}^{\objp}}{\varepsilon}}_{\textrm{(II)}} .
  \end{align*}
  Since $\pert$ is bounded \ref{hyp:dset}, taking $\varepsilon$
  sufficiently small yields $\cpt \not\in \iset_\varepsilon$. Thus we
  may Taylor expand (I) and (II) as in the $\objp > 1$ case to obtain,
  respectively,
  \begin{equation*}
    \textrm{(I)} = \lscfepsm(\cpt) + o(1)
    \qquad \text{and} \qquad
    \textrm{(II)} = \lscfm(\cpt) + o(1).
  \end{equation*}
  Note that unlike in the $\objp > 1$ case the $\varepsilon^{\objp -
    1} \abs{\pert}^\objp$ error terms do not appear, as these
  estimates were relevant only when $\cpt \in \iseps$ or $\cpt \in
  \iset$. In any case, $(\cpt, \dpt) \mapsto \lscf(\cpt, \dpt)$ is
  continuous on $\cset_{\rm out} \times \iset$ and $\varepsilon\mapsto
  \feps$ is continuous everywhere, which implies
  \begin{align*}
    \lscfm(\cpt) \leq \liminf_{\varepsilon \to 0} \lscfepsm(\cpt)
  \end{align*}
  Since $\cpt \in \cset_{\rm out}$ was arbitrary we thus have
  \begin{align*}
    \text{for all $\cpt \in \cset_{\rm out}$, }
    \lim_{\varepsilon\to 0} \delta_\varepsilon (\cpt) = \lscfm(\cpt)
  \end{align*}
  which completes the $\cset_{\rm out}$ case.

  For $\iset_{\rm in}$, observe that we may write $\iset_{\rm in} =
  \set{\cpt \in \iset \ | \ \liminf_{\varepsilon \to 0}
    1_{\iseps}(\cpt) = 1} $. So, fixing an arbitrary $\cpt \in
  \iset_{\rm in}$, we see there exists $\varepsilon_0>0$ such that for
  all $0 < \varepsilon \leq \varepsilon_0$ we have $\cpt \in
  \iset_\varepsilon$. Recalling that we defined $\delta_\varepsilon
  (\cpt) = \frac{1}{\varepsilon} [\inf_{\dpt_{\varepsilon} \in \dset}
  \abs{\cpt - \feps(\dpt_{\varepsilon})}^\objp - \inf_{\dpt \in \dset}
  \abs{\cpt - \fmap(\dpt)}^\objp]$, it follows that for all $0 <
  \varepsilon \leq \varepsilon_0$, $\delta_\varepsilon (\cpt)=0$. In
  particular, since $\cpt \in \iset_{\rm in}$ was arbitrarily chosen
  we get
  \begin{equation*}
    \text{for all } \cpt \in \iset_{\rm in},\
    \lim_{\varepsilon\to 0 } \delta_\varepsilon (\cpt) = 0,
  \end{equation*}
  thus completing the $\iset_{\rm in}$ case. Therefore, combining the
  above results we have the desired limit:
  \begin{align*}
    \lim_{\varepsilon\to 0} \delta_\varepsilon (\cpt) = \lscfm(\cpt)
    \hbox{ for $\cmeas$-a.e. $\cpt$.}
  \end{align*}
  This completes the $\objp=1$ case, and so the theorem.
\end{proof}

\begin{boxedremark}
  \label{rem:p=1-special}
  To highlight why the $\objp=1$ case is special in
  \cref{thm:fbary-grad-J}, fix some $\dset$, $\cset$ and let $\cdvz
  \in \cset^\circ$, $\cmeas = \delta_{\cdvz}$, and $\fmap \equiv
  \cdvz$. Note that regardless of the choice of selection $\pfd(\cpt)$
  we have $\fbdh \equiv \bm 0_{\rcdim}$. Fix $\cdvh \in
  \mathbb{S}^{\cdim -1}$ and let $\pert(\dpt) \equiv \cdvh$. Then
  $\objfp[1](\feps) = \abs{\varepsilon \hat v} = \varepsilon$ and
  $\objfp(\fmap) = 0$, so
  \[
    \lim_{\varepsilon \to 0} \frac{\objfp[1](\feps) -
      \objfp[1](\fmap)}{\varepsilon} = 1 \neq 0 = \int_\dset
    \ip[]{-\fbdh(\dpt), \pert(\dpt)} \dd \cpushdh(\dpt).
  \]
  In this particular example, we can at least get an integral
  representation of the first variation by defining the modified
  function $\lscft(\cpt, \dpt)$ to be $\lscf(\cpt,\dpt)$ on $\set{\cpt
    \neq \fmap(\dpt)}$ and $\abs{\pert(\dpt)}$ on $\set{\cpt =
    \fmap(\dpt)}$, whence $\displaystyle \lim_{\varepsilon \to 0}
  \frac{\objfp[1](\feps) - \objfp[1](\fmap)}{\varepsilon}= \int_\cset
  \lscftm \dd \cmeas$.

  $\quad$ However, this does not always work. Fixing a particular
  $\varepsilon > 0$ and comparing the functions in
  \eqref{eq:lower-upper-bounds-delta} at an $\cpt \in \iset_{\rm
    out}$, we see $\fmap(\pfdm) = \cpt$ and so we may write
  $\delta_\varepsilon$ as
  \begin{equation}
    \textstyle \frac{1}{\varepsilon} {\abs{\feps(\pfdepsm) - \cpt}}
    = \frac{1}{\varepsilon} \abs{\feps(\pfdepsm) - \fmap(\pfdm)} =
    \abs{\pert(\pfdepsm) + \frac{1}{\varepsilon} ({\fmap(\pfdepsm) -
        \fmap(\pfdm)})}, \label{eq:gap-control}
  \end{equation}
  while (II) reduces to $\abs{\pert(\pfdm)}$ by virtue of $\gfunc =
  0$. The result is that the inequalities in
  \eqref{eq:lower-upper-bounds-delta} can become strict, with the gap
  in some sense controlled by the ``angle'' $\pert$ makes relative to
  $\iset$ at $\cpt$.

  As an elementary example, let $\dset = \bk{-1,1}$ and consider
  $\fmap(\dpt) = \frac{1}{\sqrt 2} (\dpt, \dpt)$ and $\pert(\dpt) =
  \hat v \in \mathbb{S}^1$. Then $\iset_{\rm out} = \iset$, so we may
  pick $\cpt = (0,0)$, whence $\pfdm(\cpt) = 0$. If $\hat{v} =
  \frac{1}{\sqrt 2} (-1,1)$ then $\pfdepsm = 0 = \pfdm$, so
  \eqref{eq:gap-control} reduces to just $\abs{\pert(0)}$, the exact
  form of (II). So in this case the upper bound in
  \eqref{eq:lower-upper-bounds-delta} is tight. On the other hand, if we
  choose $\hat v = \frac{1}{\sqrt 2} (1,1)$ then for small
  $\varepsilon$ we get $\pfdepsm(\cpt) = -\varepsilon$ whence
  \eqref{eq:gap-control} gives $\delta_\varepsilon(\cpt) = 0 < 1 =
  \abs{\pert(0)} = ({\rm II})$. So in this case the modification
  $\lscft$ is not sufficient to recover the limit in
  \cref{thm:fbary-grad-J}.

  $\quad$ It is also instructive to repeat similar analysis for
  $\fmap(\dpt) = (\dpt, -\abs{\dpt})$ and $\pert(\dpt) = (0,1)$. From
  these examples it seems likely that in the $\objp = 1$ case, on
  $\iset_{\rm out}$ the relevant quantity for $\lscft$ is not
  $\abs{\pert(\pfdm)}$ but rather something like $\inf_{v \in
  \pcone(T_\cpt \iset)} \ip{v, \pert(\pfdm)}$. However, as the
  original $\fmap(\dpt) \equiv v_0$ shows, this too can fail at
  ``boundary'' points of $\iset$.
\end{boxedremark}

\cref{thm:fbary-grad-J} can be written in terms of an integral over
$\iset=\fmap(\dset)$:
\begin{corollary}
  \label{cor:fbary-grad-J-y}
  Take the hypotheses of \cref{thm:fbary-grad-J} and let $\pfim =
  \fmap(\pfdm)$, $\cpushihm = \fmap_\#(\cpushdhm)$, and
  \[
    \ninjiset = \set{\ipt \in \iset \MID \finv(\ipt) \text{ is
        non-unique}}.
  \]
  Suppose $\cpushihm(\ninjiset) = 0$ and let $\perty(\ipt) =
  \pert(\finv(\ipt))$. Then
  \begin{equation}
    \lim_{\varepsilon \to 0} \frac{\objfp(\feps) -
      \objfp(\fmap)}{\varepsilon} = \int_\iset \ip[]{-\fbihm(\ipt),
      \perty(\ipt)} \dd \cpushihm(\ipt). \label{eq:grad-J-ystar}
  \end{equation}
  In particular, if $\cmeas(\refambf) = 0$ then $\fbi$, $\cpush$ are
  independent of both the choice of $\pert \in C(\dset; \rcdim)$ and
  the choice of $\pfi \in \refpfiset$, so we may simply write
  \begin{equation}
    \lim_{\varepsilon \to 0} \frac{\objfp(\feps) -
      \objfp(\fmap)}{\varepsilon} = \int_\iset \ip[]{-\fbi(\ipt),
      \perty(\ipt)} \dd \cpushi(\ipt). \label{eq:grad-J-y}
  \end{equation}
\end{corollary}
\begin{proof}
  Let $\ninjdset = \finv(\ninjiset)$ and note $\cpushdhm(\ninjdset) =
  0$. So
  \begin{align*}
    \int_\dset \ip[]{-\fbdhm(\dpt), \pert(\dpt)} \dd \cpushdhm(\dpt)
    &= \int_{\dset \setminus \ninjdset} \ip[]{-\fbdhm(\dpt),
      \pert(\dpt)} \dd \cpushdhm(\dpt).
  \end{align*}
  By definition, for all $\dpt \in \dset \setminus \ninjdset$ we have
  $\pfdminv(\dpt) = \pfiminv(\fmap(\dpt))$ whence expanding $\fbdhm$,
  applying the change-of-variables $\ipt = \fmap(\dpt)$, and using
  $\cpushihm(\iset) = \cpushihm(\iset \setminus \ninjiset)$ gives the
  result.
\end{proof}
\begin{boxedremark}
  Verifying the $\cpushihm(\ninjiset) = 0$ hypothesis does not always
  require precise knowledge of $\cpushihm$. For example, suppose
  $\fmap(\dset)$ is a $C^1$-manifold with all self-intersections
  transverse. Then for each $\ipt \in \ninjiset$
  \cref{prop:fiber-characterization} gives $\pfiminv(\ipt) =
  \set{\ipt}$, whence if $\cmeas$ has no atoms we are guaranteed
  $\cpushihm(\ninjiset) = 0$.
\end{boxedremark}

Before continuing, it is worth discussing an important distinction
between \cref{thm:fbary-grad-J} and \cref{cor:fbary-grad-J-y}. In a
nutshell, \cref{cor:fbary-grad-J-y} is a weaker form of
\cref{thm:fbary-grad-J} that is easier to work with visually (and
hence numerically as well): Provided $\cmeas(\iset_{\rm in}) = 0$, one
may compute the first variation by simply computing the fibers
$\pfiinv(\ipt)$ and, if necessary, assigning points in $\refambf$ to the
$\ipt$ where $\abs{\perty}$ is the largest. In this way,
\cref{cor:fbary-grad-J-y} also discretizes nicely; see
\cref{sec:discretizing-fmap}.

The tradeoff is that \cref{cor:fbary-grad-J-y} is insufficient to
analyze an important class of perturbations, namely those for which
the parametrization of $\iset$ is key. In loose terms, these are the
perturbations that change the topology of $\iset$ (or at least appear
to locally) on a non-null set. Simple examples arise when $\cmeas(\wt
\iset) > 0$; \eg\ $\dset = [-1,1]$, $\cmeas$ uniform on $[-1,1]^2$,
$\fmap(\dpt) = (1 - \abs{\dpt}, 0)$, and $\pert(\dpt) = (0,t)$.
Pictorially, $\varepsilon\pert$ has the effect of ``splitting'' a line
segment into a narrow wedge shape, which cannot be encoded via a
$\cpushihm$-a.e.\ well-defined $\perty$, hence we must use
\cref{thm:fbary-grad-J}.

For the case $\cpushihm(\wt \iset) \neq 0$ but $\cmeas(\wt \iset) =
0$, an illustrative example is $\dset = [-1,1]$, $\cmeas = \frac{1}{2}
\delta_{(0,1)} + \frac{1}{2} \mrm{Unif}_{[-1,1] \times \set{0}}$, and
$\fmap$ is a parametrization of $[-1,1] \times \set{0}$ that is
identically $(0,0)$ on $[-1/2, 1/2]$ and linear elsewhere. Here,
\cref{thm:fbary-grad-J} allows perturbations like $\pert(\dpt) =
\chi_{[-1/2, 1/2]}(\dpt)\ \pn{0, 1-2\abs{\dpt}}$ (which adds a
vertical spike at $(0,0)$) and correctly predicts that $\feps$ gives
an $O(\varepsilon)$ improvement. However, as before this cannot be
encoded via a $\cpushihm$-a.e.\ well-defined $\perty$.

\subsection{Local Improvements and First-Order Conditions}\label{sec:local-improve}

\cref{thm:fbary-grad-J} shows that if there exist $\pert, \pfd$
yielding $\cpushdh(\set{\fbdh \neq 0}) > 0$ then it should be possible
to improve $\objfp(\fmap)$ by ``following'' $\fbdh(\dpt)$, as long as
the modification still stays in the same Sobolev class $\wkpsets$.
This regularity condition requires some care because as seen in
Counterexamples \ref{cex:f-smooth-fbary-discont} and
\ref{cex:Y-c1-fbary-discont}, $\fbdh$ often lacks regularity, so we
typically cannot guarantee $\fmap + \varepsilon \fbdh \in \wkpsets$.
Nevertheless, our second main theorem of this section,
\cref{thm:smooth-improvement} below shows that we can at least find a
smooth local approximant to $\fbd$ around a density point of
$\set{\fbdh \neq 0}$. The argument is essentially measure-theoretic
and relies in the last step on a simple fact about Radon measures
(\cref{cor:radon-shrinking-bound}). Though
\cref{lem:lebesgue-shrinking} and \cref{cor:radon-shrinking-bound} are
surely well-known, we are not aware of a reference, and so for
completeness have included two short proofs we learned of from Pablo
Shmerkin.
\begin{lemma}
  \label{lem:lebesgue-shrinking}
  Let $\dmeas$ be a Radon measure on $\rsld$. Define
  \[
    \slpts = \set{\dpt \in \rsld \MID \liminf_{\smallr \to 0}
      \frac{\dmeas(B_r(\dpt))}{\smallr^{\sld}} = 0}.
  \]
  Then $\dmeas(\slpts) = 0$.
\end{lemma}
\begin{proof}
  If $\dmeas(\set{\dpt}) > 0$ then immediately $\dpt \not \in \slpts$,
  hence we may exclude these points. We first treat the case where
  $\dmeas$ has bounded support. Let $\varepsilon > 0$ and define
  \begin{align*}
    \slptseps
    &= \set{\dpt \MID \dmeas(\set{\dpt}) = 0 \text{ and }
      \liminf_{\smallr \to 0}
      \frac{\dmeas(B_{\smallr}(\dpt))}{\smallr^{\sld}} < \varepsilon}
      \shortintertext{ and }
      \sballseps
    &= \set{B_\smallr(\dpt) \MID \dpt \in \slptseps,\ \smallr \leq 1,
      \text{ and } \frac{\dmeas(B_\smallr(\dpt))}{\smallr^{\sld}} <
      \varepsilon}.
  \end{align*}
  Note $\slpts \subseteq \slptseps$ and that for all $\dpt \in
  \slptseps$ we get $\inf \set{r \MID B_r(\dpt) \in \sballseps} = 0$.
  By Vitali's covering theorem for Radon measures (see, \eg,
  \cite[Thm 2.8]{mattila}) there exists a pairwise-disjoint,
  at-most-countable set $\set{B_{\smallr_i}(\dpt_i)} \subseteq
  \sballseps$ with \(\dmeas\pn{\slptseps \setminus \bigcup_i
    B_{\smallr_i}(\dpt_i)} = 0\). So
  \[
    \dmeas(\slpts) \leq \dmeas(\slptseps) \leq
    \sum_i \dmeas(B_i)
    \leq \varepsilon \sum_i \smallr_i^d \leq \varepsilon
    \frac{\leb(\supp(\dmeas))}{\leb(B_1(0))},
  \]
  and taking $\varepsilon \to 0$ yields the desired result.

  Now suppose $\dmeas$ has unbounded support. For each $j \in \NN$
  define $\dmeas_j$ to be the restriction of $\dmeas$ to $B_j(0)$.
  Then each $\slpts_j \coloneqq \set{\dpt \MID \liminf_{\smallr \to 0}
    {\dmeas_j(B_\smallr(\dpt))} / {\smallr^\sld} = 0}$ is
  $\dmeas_j$-null and so $\dmeas$-null. Writing $\slpts = \bigcup_j
  \slptsj$ then yields the result.
\end{proof}
\begin{corollary}
  \label{cor:radon-shrinking-bound}
  Let $\dmeas$ be a Radon measure on $\rsld$. Then for all $\dbfac \in
  (0,1)$, for $\dmeas$-a.e.\ $\dpt$,
  \begin{equation}
    \limsup_{\smallr \to 0}
    \frac{\dmeas(B_{\dbfacr}(\dpt))}{\dmeas(B_\smallr(\dpt))} \geq
    \dbfac^{\sld + 1} > 0. \label{eq:limsup-bound}
  \end{equation}
\end{corollary}
\begin{proof}
  By \cref{lem:lebesgue-shrinking}, for $\dmeas$-a.e.\ $\dpt$ there
  exists $\constx > 0$ and $R_\dpt > 0$ such that for all $0 < r \leq
  R_\dpt$, $\dmeas(B_\smallr(\dpt)) \geq \constx \smallr^\sld$. Fix
  some such $\dpt$ and for all $j \in \NN$ sufficiently large define
  $a_j = \dmeas(B_{\dbfac^j}(\dpt))$; note $a_j \geq \constx
  \dbfac^{\sld j}$. Suppose, to obtain a contradiction, that
  \[
    \limsup_{j \to \infty} \frac{a_{j+1}}{a_j} < \dbfac^{\sld + 1}.
  \]
  Then there exists $j_0$ such that for all $j \geq j_0$ one has
  $a_{j+1} < \lambda^{\sld + 1} a_j$, whence
  \[
    a_j < a_{j_0} \dbfac^{(\sld + 1)(j - j_0)} = \bk[]{a_{j_0}
      \dbfac^{-(\sld + 1)j_0} \cdot \dbfac^{j}} \dbfac^{\sld j}.
  \]
  Taking $j$ sufficiently large one gets $a_{j_0} \dbfac^{-(\sld +
  1)j_0} \cdot \dbfac^j < \constx$ whence $a_j < \constx \dbfac^{\sld
  j}$, a contradiction.
\end{proof}
With this we can prove our local modification result:
\begin{theorem}[Local Improvement]
  \label{thm:smooth-improvement}
  Take the hypotheses of \cref{thm:fbary-grad-J} and suppose that
  \begin{align}\label{eqn:nontrivial-F}
  \hbox{for some $\pfd \in \pfdset$ we have $\cpushdh(\set[]{\fbdh
    \neq 0}) > 0$.}
  \end{align}
  Then for all $\varepsilon > 0$ there exists $\perteps \in
  C^\infty(\dset; \rcdim)$, depending on $\fbdh$, such that
  \[
    \objfp(\fmap + \perteps) < \objfp(\fmap) \qquad \text{and} \qquad
    \cost(\fmap + \perteps) < \cost(\fmap) + \varepsilon .
  \]
  Furthermore, if $\cpushdh(\finv(\partial \cset)) = 0$, then
  $\perteps$ can be chosen so that $\fmap + \perteps$ takes values
  only in $\cset$, whence $\fmap$ is not a local minimum of $\objfp$
  in $\wkpsets$.
\end{theorem}
This result says that the barycenter field, under the condition
\eqref{eqn:nontrivial-F}, allows local modifications that reduce
$\objfp$ with only small additional budget. It also implies that the
function $J(\budg)$ is strictly decreasing at those $\budg$ where
there exists an $\budg$-optimizer $\fbudg$ having nontrivial
barycenter field in the sense of \eqref{eqn:nontrivial-F}.
\begin{proof}[Proof of Theorem~\ref{thm:smooth-improvement}]
  Let $\varepsilon > 0$ be given. Let $\fddset =
  \set{\fibdir_j}_{j=1}^\infty \subseteq \rcdim \setminus \set{0}$ be
  dense and fix $\fibconst \in (0,1)$. For all $j$ define
  \[
    \dset_j = \set[big]{\dpt \in \dset \MID \abs[]{\fbdh(\dpt) - v_j} <
      \fibconst \abs{v_j}}
  \]
  Observe $\set[]{\fbdh \neq 0} = \bigcup_j \dset_j$, and that because
  $\fbdh$ is Borel-measurable, each $\dset_j$ is measurable. Since
  $\set{\dset_j}$ is countable and $\cpushdh(\set[]{\fbdh \neq 0})
  >0$, for at least one $\specind$ we have $\cpushdh(\dsspec) > 0$.
  Now because $\cpushdh$ is Radon, the density theorem (see, \eg,
  \cite[Cor. 2.14(1)]{mattila}) and \cref{cor:radon-shrinking-bound}
  imply $\cpushdh$-a.e.\ $\dpt \in \dsspec$ is a density point of
  $\cpushdh$ for which \eqref{eq:limsup-bound} holds. Fix one and call
  it $\dptspec$.

  Let $\bump$ be a bump function on $B_{1}(0; \rcdim)$ and define
  \[\pertdens(\dpt) = \vspec \bump \pn{(\dpt - \dptspec)/\densr}, \]
  which is compactly supported in the ball $\bdensr(\dptspec)$.
  Applying \cref{thm:fbary-grad-J}, since $\pfd$ might not be optimal
  for $\pertdens$ we get the inequality
  \begin{equation*}
    (*) = \lim_{\pertamp \to 0} \frac{\objfp(\fmap + \pertamp
      \pertdens) - \objfp(\fmap)}{\pertamp} \leq
    -\int_{\bdensr(\dptspec)} \ip{\fbdh, \pertdens} \dd
    \cpushdh(\dpt).
  \end{equation*}
  Let $\idz = \bdensr(\dptspec) \cap \dsspec$ and $\ido =
  \bdensr(\dptspec) \setminus \idz$. Then we can rewrite the above as
  \begin{equation*}
    (*) \le -\int_{\idz} \ip{\fbdh, \pertdens} \dd \cpushdh(\dpt) -
    \int_{\ido} \ip{\fbdh, \pertdens} \dd \cpushdh(\dpt).
  \end{equation*}
  For the $\idz$ integral: By the definition of $\dsspec$, for all
  $\dpt \in \idz$,
  \[
    -\ip[]{\fbdh(\dpt), \vspec} < -\frac{1}{2}
    \pn{\abs[]{\fbdh(\dpt)}^2 + (1-c^2) \abs[]{\vspec}^2}.
  \]
  We ignore the $\abs[]{\fbdh(\dpt)}^2$ term and define
  \[
    \constz = \frac{1-c^2}{2} \abs[]{\vspec}^2 \quad \text{ (note
      $\constz > 0$)}.
  \]
  For the $\ido$ integral, the global bound $\abs[]{\fbdh} \leq
  \diam(\cset)$ gives
  \[
    -\ip[]{\fbdh, \vspec} \leq \diam(\cset) \abs[]{\vspec} \eqqcolon
    \consto \, \hbox{ (note $\consto > 0$).}
  \]
  Thus we have
  \begin{equation}
    (*)  < - \constz \int_{\idz} \bump\pn[Big]{\frac{\dpt -
        \dptspec}{\densr}} \dd \cpushdh + \consto \int_{\ido}
    \bump\pn[Big]{\frac{\dpt - \dptspec}{\densr}} \dd
    \cpushdh. \label{eq:two-ints}
  \end{equation}
  Now fix some $\dbfac \in (0,1)$. On the one hand,
  \begin{align*}
    \int_{\idz} \bump\pn[Big]{\frac{\dpt - \dptspec}{\densr}} \dd
    \cpushdh
    &\geq \cpushdh(\idz \cap B_{\dbfac \densr}(\dptspec))
      \inf_{\idz \cap B_{\dbfac \densr}(\dptspec)} \bump\pn[Big]{\frac{\dpt -
      \dptspec}{\densr}} \\
    &\geq \cpushdh(\idz \cap B_{\dbfac \densr}(\dptspec))
      \exp(-1/(1-\dbfac^2)),
  \end{align*}
  while on the other hand \(\int_{\ido} \bump\pn[big]{\frac{\dpt -
      \dptspec}{\densr}} \dd \cpushdh \leq C \cpushdh(\ido) \) with
  $C$ depending on $\eta$ only. For concision, we now suppress writing
  the center $\dptspec$ of the balls $B_{\densr}(\dptspec)$,
  $B_{\dbfac \densr}(\dptspec)$. Letting $\constzt = \constz
  e^{-1/(1-\dbfac^2)} > 0$, factoring $\cpushdh(\idz)$ out of the
  right side of \eqref{eq:two-ints} thus gives
  \begin{align*}
    \lim_{\pertamp \to 0} \frac{\objfp(\fmap + \pertamp \pertdens) -
    \objfp(\fmap)}{\pertamp}
    &\leq \cpushdh(\idz) \bk{-\constzt \frac{\cpushdh(\idz \cap B_{\dbfac
      \densr})}{\cpushdh(\idz)} + C \consto
      \frac{\cpushdh(\ido)}{\cpushdh(\idz)}}.
  \end{align*}
  Since $\dptspec$ is a $\cpushdh$-density point of $\dsspec$, the
  second term vanishes as $\densr \to 0$. For the first term,
  recalling $\idz = B_\densr \cap \dsspec$ we may rewrite the fraction
  as
  \begin{align*}
    \frac{\cpushdh(\dsspec \cap B_{\dbfac
    \densr})}{\cpushdh(\dsspec \cap B_\densr)}
    &= \frac{\cpushdh(\dsspec \cap B_{\dbfac
      \densr})}{\cpushdh(B_{\dbfac \densr})} \cdot
      \frac{\cpushdh(B_{\dbfac
      \densr})}{\cpushdh(B_\densr)} \cdot
      \frac{\cpushdh(B_\densr)}{\cpushdh(\dsspec \cap
      B_\densr)}.
  \end{align*}
  Since $\dptspec$ is a $\cpushdh$-density point of $\dsspec$ for
  which \eqref{eq:limsup-bound} holds, we thus see that
  \[
    \limsup_{\densr \to 0}\frac{\cpushdh(\idz \cap B_{\dbfac
        \densr})}{\cpushdh(\idz)} \geq 1 \cdot \dbfac^{\cdim +1} \cdot
    1 >0.
  \]
  Hence there exists a choice of $\densr > 0$ such that
  \[
    \lim_{\pertamp \to 0} \frac{\objfp(\fmap + \pertamp \pertdens) -
      \objfp(\fmap)}{\pertamp} \leq -\constzt \cpushdh(\idz)
    \dbfac^{\cdim + 1} + C \consto \cpushdh(\ido) < 0.
  \]
  In particular, for $\pertamp >0 $ sufficiently small, taking
  $\perteps = \pertamp \pertdens$ yields $\objfp(\fmap + \perteps) <
  \objfp(\fmap)$ as well as $\cost(\pertamp \pertdens)< \varepsilon$.

  Finally, if $\cpushdh(\finv(\partial \cset)) = 0$ then we may assume
  $\dptspec$ is chosen so that $\fmap(\dptspec) \in \cset^\circ$.
  Taking $\densr$ sufficiently small then yields
  $d(\fmap(\bdensr(\dptspec)), \partial \cset) > 0$ and taking
  $\pertamp < d(\fmap(\bdensr(\dptspec)), \partial \cset) /
  \abs[]{\vspec}$ guarantees $\fmap + \pertamp \pertdens$ takes values
  only in $\cset$.
\end{proof}

\section{The Discrete Problem \& Simulations}
\label{sec:the-discrete-problem}
In this section we discuss some theoretical justifications for two
discretization schemes for the problem.

\subsection{Discretizing $\cmeas$}
\label{sec:discretizing-cmeas}

A natural idea is to discretize $\cmeas$ by representing it with an
empirical measure $\cmeas_N$. We then obtain a consistency result for
the $\cmeas_N$ formulation for \xref{prob:hard-constraint} directly
from \cref{lem:weak-compactness} and \cref{thm:joint-continuity}.
\begin{corollary}[Consistency of \xref{prob:hard-constraint}]
  \label{cor:consistency}
  Take \ref{hyp:dset}, \ref{hyp:cset}, \ref{hyp:sobp}, and
  \ref{hyp:kqm}. Suppose that $\cmeas_N$ converges
  weak-\textasteriskcentered{} to $\cmeas$. Fix $\budg \geq
  \refbudgmin$ and for all $N$ let $\fmap_N$ be a $\budg$-optimizer
  for \xref{prob:hard-constraint} with respect to the objective
  $\objfp(\tinybox; \cmeas_N)$. Then every subsequence of
  $\set{\fmap_N}$ has a weakly-convergent (and thus, by
  \cref{cor:wconv-uconv}, uniformly-convergent) subsequence whose
  limit is an $\budg$-optimizer for \xref{prob:hard-constraint} with
  respect to the objective $\objfp(\tinybox; \cmeas)$.

  In particular, taking the $\cmeas_N$ to be empirical measures for
  $\cmeas$ yields the almost-sure equalities
  \[
    \lim_{N\to\infty} \objfp(\fmap_N;\cmeas_N) \eqas \lim_{N\to\infty}
    \objfp(\fmap_N; \cmeas) \eqas \objell(\budg).
  \]
\end{corollary}
In \cref{sec:generative-learning} we will shift focus toward the
\xref{prob:soft-penalty} formulation owing to its computational
benefits. Hence we now prove a consistency result for
\xref{prob:soft-penalty}.
\begin{corollary}[Consistency of \xref{prob:soft-penalty}]
  \label{cor:consistency-sp}
  Take the same hypotheses as \cref{cor:consistency}, only this time
  fixing $\lambda > 0$ and defining the $\fmap_N$ to be optimizers of
  \xref{prob:soft-penalty} with respect to the objective
  $\objfpl(\tinybox; \cmeas_N)$. Then every subsequence of
  $\set{\fmap_N}$ has a weakly-convergent (and thus, by
  \cref{cor:wconv-uconv}, uniformly-convergent) subsequence whose
  limit is a solution to \xref{prob:soft-penalty} with
  $\objfpl(\tinybox; \cmeas)$.
\end{corollary}

\begin{proof}
  First we show $\sup_N \cost(\fmap_N) < \infty$. To that end let
  $\cpt_0, \cpt_1$ such that $d(\cpt_0, \cpt_1) = \diam(\cset)$ and
  let $\fmap_{\rm max} = \cpt_0$, $\cmeas_{\rm max} =
  \delta_{\cpt_1}$; observe that they achieve $\sup_{\fmap, \cmeas}
  \objfp(\fmap; \cmeas)= \diam(\cset)^\objp$. Then for all $N$, taking
  $\fmap_{\rm max}$ as a competitor yields
  \[
    \lambda \cost(\fmap_N)
    \leq \objfpl(\fmap_N; \cmeas_N)
    \leq \objfpl(\fmap_{\rm max}; \cmeas_N)
    \leq \objfpl(\fmap_{\rm max}; \cmeas_{\rm max})
    = \diam(\cset)^\objp + \lambda \cost(\fmap_{\rm max}),
  \]
  whence we see $\sup_N \cost(\fmap_N) < \infty$ as desired (note that
  it was necessary to pass through $\objfpl(\fmap_{\rm max};
  \cmeas_N)$ since priori we cannot compare $\cost(\fmap_N)$ and
  $\cost(\fmap_{\rm max})$).

  In any case, after obtaining this uniform boundedness $\sup_N
  \cost(\fmap_N) < \infty$, the consistency result follows from
  \cref{lem:weak-compactness} and \cref{thm:joint-continuity}.
\end{proof}

\subsection{Representing $\fmap$ by Samples}
\label{sec:discretizing-fmap}

In this subsection we discuss some theoretical justifications for a
particular discretization scheme we use to simulate
\xref{prob:soft-penalty}. To that end we define discretized analogues
of $\fbdh$ (\cref{def:discrete-barycenter-field}) and
\cref{cor:fbary-grad-J-y}. We focus on discretizations related to
$\fmap$ and its perturbations by the barycenter field, but will not
consider simultaneous discretization of $\cmeas$ (though note $\cmeas$
might happen to be discrete on its own anyways).

The idea is to sample $\iset = \fmap(\dset)$, perturb the samples by a
discretized analogue of $\fbdh$, smooth any discontinuities,
reparametrize, and then repeat. We believe this scheme could be
particularly useful in real-world situations where $\fmap$ is being
modified in real-time; for example if we are continuously collecting
new samples from $\cmeas$.

To help distinguish with the continuous case we will endeavor to
denote all discretized objects with sans-serif letters.
\begin{definition}[Discrete Barycenter Field]
  \label{def:discrete-barycenter-field}
  Let $\isn$ be a collection of $\sfn$ points in $\fmap(\dset)$ and
  let $\psin$ be  a measurable selection of
  the closest-point projection onto $\isn$. Let $(\cmiptsfn,
  \cpushsf)$ denote the disintegration of $\cmeas$ by $\psin$. Then we
  define the \emph{discrete barycenter field} at $\iptsfj \in \isn$ to
  be
  \begin{equation}
    \fbarysf(\iptsfj) = \objp \int_{\psin^{-1}(\iptsfj)} (\cpt - \iptsfj)
    \abs{\cpt - \iptsfj}^{\objp - 2} \dd \cmeas_{\iptsfj,
      \psin}(\cpt), \label{eq:discrete-bary}
  \end{equation}
  again taking the convention that $\msf v \abs{\msf v}^{\objp -2} =
  0$ when $\msf v = 0$.
\end{definition}
Note, $\fbarysf(\iptsfj)$ is (up to the constant $\objp$) just the
$\objp-1$ barycenter of the Voronoi cell in coordinates with the base
point $\iptsfj$ at the origin.
\begin{boxedremark}
  \label{rem:fbary-consistency}
  For each $\sfn$ let $\dsn$ be an $\sfn$-sample empirical measure for
  some absolutely continuous $\mu \in \pmeas(\dset)$ and let $\isn =
  \fmap(\dsn)$. One may show that as $\sfn \to \infty$, $\psin \to
  \pfim$ a.s.\ on $\cset \setminus \refambf$, whence with \ref{hyp:ac} we
  get $\cpushsf \to \cpushihm$ weak-\textasteriskcentered{} a.s. While
  we have not investigated the matter in detail, we therefore suspect
  there could be a consistency result along the lines of $\fbarysf
  \wconv \fbihm$, in the sense that for all $\perty \in C(\iset;
  \rcdim)$ one should have something like
  \[
    \lim_{\sfn \to \infty} -\int_\iset \ip[]{\fbarysf(\ipt),
      \perty(\ipt)} \dd \cpushsf(\ipt) \eqas -\int_\iset
    \ip[]{\fbihm(\ipt), \perty(\ipt)} \dd \cpushihm(\ipt).
  \]
  The intuition would be that $\fbarysf$ is averaging $\fbihm$ over
  $\psin^{-1}(\iptsfn)$, up to error arising from treating $\iptsfn$
  as the base point instead of $\pfim(\cpt)$, which should be small in
  light of $\psin \to \pfim$.
\end{boxedremark}
In analogy with \cref{thm:fbary-grad-J} (more precisely with
\cref{cor:fbary-grad-J-y}) we have the following:
\begin{proposition}
  \label{prop:fbary-grad-J-discrete}
  Take \ref{hyp:cset}. Let $\pertsfn : \isn \to \rcdim$ such that for all
  $\varepsilon >0$ sufficiently small, $\isneps = \set{\iptsf +
    \varepsilon \pertsfn(\iptsf) \MID \iptsf \in \isn} \subseteq
  \cset$. Let $\isnout = \set{\iptsf \in \isn \MID \pertsfn(\iptsf)
    \neq 0}$ and suppose either $\objp > 1$, or that $\objp = 1$ and
  $\cmeas(\isnout) = 0$. Then
  \begin{align*}
    \lim_{\varepsilon \to 0} \frac{\objfp(\isneps) -
    \objfp(\isn)}{\varepsilon}
    &= -\sum_{\isn} \ip[]{\fbarysf(\iptsf), \pertsfn(\iptsf)}
      \cpushsf (\iptsf).
      \shortintertext{On the other hand, if $\objp
      = 1$ and $\cmeas(\isnout) > 0$ then the same result holds with
      the modification}
      \lim_{\varepsilon \to 0} \frac{\objf_1(\isneps) -
      \objf_1(\isn)}{\varepsilon}
    &= -\sum_{\isn} \ip[Big]{\fbarysf(\iptsf) - \chi_{\isn}(\iptsf)
      \hat{\msf{\pert}}_{\sfn}(\iptsf),\ \pertsfn(\iptsf)}
      \cpushsf (\iptsf).
  \end{align*}
  where $ \hat{\msf{\pert}}_{\sfn}(\iptsf)$ is the unit vector
  $\frac{1}{|{\msf{\pert}}_{\sfn}(\iptsf)| }
  {\msf{\pert}}_{\sfn}(\iptsf)$ (again, taking the convention
  $\hat{\bm 0} = \bm 0$).
\end{proposition}
\begin{proof}
  Define $\psineps$ to be a measurable selection of the closest point
  projection onto $\isneps$. Note that since $\isn$, $\isneps$ are
  discrete, the $\psineps$ can be chosen such that for any fixed $\cpt
  \in \cset$, there exists $\varepsilon_\cpt$ such that
  \[
    \text{for all } 0 < \varepsilon \leq \varepsilon_\cpt, \quad
    \psineps(\cpt) = \psin(\cpt) + \varepsilon \pertsfn(\psin(\cpt)).
  \]

  Hence fix $\varepsilon > 0$ and let
  \[
    \csetepssfj = \set{\cpt \ | \ \psineps(\cpt) = \psin(\cpt) +
      \varepsilon \pertsfn(\psin(\cpt)) }; \hbox{ note
      $\bigcup_{\varepsilon > 0}^\infty \csetepssfj = \cset$}.
  \]
  Since $\cset$ is bounded \ref{hyp:cset} and $\isneps \subseteq
  \cset$, we thus get
  \begin{align*}
 & \lim_{\varepsilon \to 0} \frac{\objfp(\isneps) -
      \objfp(\isn)}{\varepsilon}
       =
      \lim_{\varepsilon \to 0} \int_\cset \frac{d(\cpt,
      \isnepsj)^\objp - d(\cpt, \isn)^\objp}{\varepsilon} \dd \cmeas \\
    &= \lim_{\varepsilon \to 0} \left[\int_{\csetepssfj}
      \frac{\abs{\cpt - (\psin + \epssfj \pert(\psin))}^\objp -
      \abs{\cpt - \psin}^\objp}{\epssfj}
      \dd \cmeas + O(\cmeas(\cset \setminus \csetepssfj))\right],
  \end{align*}
  where we have suppressed writing the $\cpt$ in $\psin(\cpt)$ (notice
  we get the term $O(\cmeas(\cset \setminus \csetepssfj))$ due to
  $\displaystyle \frac{d(\cpt, \isnepsj)^\objp - d(\cpt,
    \isn)^\objp}{\varepsilon} \le C$ independent of $\varepsilon$).
  Further decomposing the domain of integration of the $\csetepssfj$
  integral into $\csetepssfj \setminus \isn$ and $\isn$, we see the
  $\isn$ integral reduces to $\sum_{\isn} \epssfj^{\objp -1}
  \abs{\pertsfn} \cmeas(\iptsf)$, which is nonzero iff
  $\cmeas(\isnout) > 0$. On the other hand, Taylor expanding the
  integrand of the $\csetepssfj \setminus \isn$ integral around
  $\epssfj = 0$ yields, as in \cref{thm:fbary-grad-J},
  \[
    \msf{h}(\cpt, \psin) = \ip[]{-\objp (\cpt - \psin)
      \abs{\cpt - \psin}^{\objp -2}, \pertsfn(\psin)} +
    O(\varepsilon).
  \]
  Taking the convention $\msf{h}(\cpt, \iptsf) = 0$ when $\cpt =
  \iptsf$ we may thus write
  \[
    \lim_{\varepsilon \to 0} \frac{\objfp(\isneps) -
      \objfp(\isn)}{\varepsilon} = \int_\cset \msf{h}(\cpt, \psin) +
    \lim_{\varepsilon \to 0} \varepsilon^{\objp - 1} \chi_{\isn}(\cpt)
    \abs{\pertsfn(\psin)} \dd \cmeas,
  \]
  the second term vanishing if $\objp \neq 1$ or if $\cmeas(\isnout) =
  0$. Finally, applying the disintegration theorem to the $\msf h$
  integral obtains the desired result.
\end{proof}
Now we discuss a discrete analogue of \cref{thm:smooth-improvement}.
To consider a discrete analogue of the cost $\cost(\fmap)$, start with
an arbitrarily-chosen discrete set $\isn = \set{\cpt_1, \ldots,
  \cpt_\sfn} \subseteq \cset$ and define the collection of admissible
parametrizations by
\begin{equation}
  \msf \Phi(\isn) = \set{\fmap \in \wkpsets \MID \isn \subseteq
    \fmap(\dset)}. \label{eq:admissible-parametrizations}
\end{equation}
A construction similar to \cref{prop:asymptotics} shows that there
always exists a smooth $\fmap : \dset \to \cset$ for which $\isn
\subseteq \iset$, whence $\Phi(\isn)$ is nonempty. So we may define a
discrete analogue of the cost $\cost(\fmap)$ by taking
\begin{equation}
  \costsf(\isn) = \inf_{\fmap \in \msf \Phi(\isn)}
  \cost(\fmap). \label{eq:discrete-cost}
\end{equation}
\begin{boxedremark}
  \label{rem:N-encodes-complexity}
  We note that the definition of $\msf C$ succinctly encodes that
  $\msf N$ is a proxy for the ``complexity'' of the approximation
  $\isn$: If $\msf{\iset_N} \subseteq \msf{\iset_{N+1}}$ then $\msf
  \Phi(\isn) \supseteq \msf \Phi(\msf{\iset_{N+1}})$, which implies
  $\msf{C}(\isn) \leq \msf{C(\iset_{N+1})}$.
\end{boxedremark}

We may state the analogue of \cref{thm:smooth-improvement} in these
terms. Note that in contrast to \cref{thm:smooth-improvement}, in this
case the proof is almost trivial: Because $\isn$ is discrete, we do
not have to treat the same regularity concerns for $\fbarysf$ that we
did for $\fbdh$ in \cref{thm:smooth-improvement}.
\begin{corollary} \label{cor:disc-improvement} Let $\isn
  =\set{\iptsf_1, \ldots, \iptsf_\sfn }\subseteq \cset^\circ$ be given
  and suppose there exists a measurable selection of the closest-point
  projection onto $\isn$ such that
  \[
    \cpushsf(\set[]{\fbarysf \neq 0}) > 0.
  \]
  Then if $\objp > 1$, or if $\objp = 1$ and $\cmeas(\isnout) = 0$,
  there exists $\pertsfn$ such for $\varepsilon$ sufficiently small,
  \[
    \objfp(\isneps) < \objfp(\isn) - C \varepsilon \qquad \text{and}
    \qquad \costsf(\isneps) < \costsf(\isn) + \varepsilon,
  \]
  where $C$ is a constant dependent on $\isn$ and $\cmeas$.
\end{corollary}
\begin{proof}
  By construction of $\fbarysf$, there exists an $\epssfn > 0$ such
  that for all $\iptsf \in \isn$ we have $\iptsf + \epssfn
  \fbarysf(\iptsf) \in \cset$. Let $\fepssf \in \msf \Phi(\isn)$ such
  that
  \[
    \cost(\fepssf) - \frac{\epssfn}{2} < \costsf(\isn).
  \]
  From now on write $\fmap = \fepssf$. Now for each $\msf j = 1,
  \ldots, \sfn$, let $\dptsfj \in \fmap^{-1}(\iptsfj)$ and pick $\dsfj
  > 0$ such that for all $\msf i \neq \msf j$,
  $\fmap(B_{\dsfj}(\dptsfj)) \cap \fmap(B_{\dsfi}(\dptsfi)) =
  \varnothing$. Finally, for some choice of coefficients $\msf c_1,
  \ldots, \msf c_\sfn > 0$ define $\pert : \dset \to \rcdim$ by
  $\pert(\dpt) = \sum_{\msf j = 1}^\sfn \msf c_{\msf j}
  \eta_{\dsfj}(\dpt - \dptsfj) \fbarysf(\iptsfj)$, where
  $\eta_{\dsfj}$ denotes a smooth nonnegative function compactly
  supported in the $\dsfj$-ball such that $\eta_{\dsfj} (0)=1$. Notice
  that although a priori $\dsfj$ depends on $\fmap=\fepssf$, we can
  make it depend only on $\isn$ for sufficiently small $\epssfn$. Then
  taking each
  \[
    \msf c_j = \frac{1}{2 \sfn \norm[]{\eta_{\dsfj}}_{\wkpsets}
      \abs[]{\fbarysf(\iptsfj)}}
  \]
  yields $\cost(\pert(\dpt)) \leq \frac{1}{2}$.

  Finally, let $\pertsfn(\iptsfj) = \pert(\dptsfj)$ and as before let
  $\iset_{\sfn, \epssfn} = \set{\iptsf + \varepsilon_{\sfn} \pert \mid
    \iptsf \in \isn}$. Note that since $\fmap + \epssfn \pert \in \msf
  \Phi(\iset_{\sfn, \epssfn})$,
  \[
    \costsf(\isneps) \leq \cost(\fmap) + \epssfn \cost(\pert) <
    \costsf(\isn) + \epssfn.
  \]
  By \cref{prop:fbary-grad-J-discrete}, we see that for $\epssfn$
  sufficiently small,
  \begin{equation}
    \objfp(\msf{\iset}_{\sfn, \epssfn}) - \objfp(\isn) = -\epssfn
    \sum_{\msf j} \frac{\abs[]{\fbarysf(\iptsfj)}}{\sfn
      \norm[]{\eta_{\dsfj}}_{\wkpsets}} \cpushsf (\iptsf) +
    O(\epssfn^2) \label{eq:discrete-update}
  \end{equation}
  where the summands on the right-hand side are independent of
  $\epssfn$. This completes the proof.
\end{proof}
Note that \eqref{eq:discrete-update} immediately allows us to see a
coarse upper bound on how small $\epssfn$ must be for the
approximation to hold: Since $0 \leq \objfp(\isneps) = \objfp(\isn) -
C \epssfn + O(\epssfn^2)$, we need at least
\begin{equation}\label{eq:coarse-epssfn}
  \epssfn \leq \objfp(\isn)/C.
\end{equation}

\subsection{Simulations Of \xref{prob:soft-penalty} In Two Dimensions
  With Uniform Target Measure}
\cref{cor:disc-improvement} provides theoretical motivation for a
numerical scheme for performing gradient descent on $\objfp(\isn) +
\lambda \cost(\fmap)$ by taking discrete steps along $\fbarysf +
\lambda \nabla \cost(\fmap)$. Similar ideas have been pursued many
times before: In the context of principal curves and surfaces, see
\eg\ \cite{kegl2000,kirov-slepcev-2017}, and particularly
\cite{gerber_dimensionality_2009}; in the context of spline fitting
see the references for \emph{unstructured} fitting listed in
\cref{sec:spline-smoothing}, particularly \cite{Wang2006Apr}, which
uses a penalty similar to ours for $\sobk = 2$.

However, there is a sense in which these algorithms are all primarily
designed for use in \emph{real-world} contexts rather than in studying
\emph{theoretical} properties of \xref{prob:hard-constraint},
\xref{prob:soft-penalty}. In \cref{sec:real-vs-theoretical} we explain
this distinction in greater detail, and particularly highlight the
motivation for developing algorithms tailored toward the
``theoretical'' case. Then in \cref{sec:algorithm} we
introduce such a routine for the special case where $\ddim=1$,
$\cdim=2$, $\cmeas$ is uniform, and $\partial\supp(\cmeas)$ is
piecewise linear (note there is no restriction on $\objp$ beyond
$\objp \geq 1$). Example outputs are shown in
\cref{sec:gallery-of-sims}. We hypothesize that the algorithm could be
adapted to other values of $\ddim, \cdim$; however, loosening the
restrictions on $\cmeas$ would be more difficult.

\subsubsection{\xref{prob:soft-penalty} in ``Real-world'' vs.\
  ``Theoretical'' Cases}
\label{sec:real-vs-theoretical}

We propose to draw a distinction between numerical algorithms for
\xref{prob:soft-penalty} in ``real-world'' vs.\ ``theoretical''
contexts. The difference is essentially one of data resolution and
hence only applies when $\cmeas$ has a non-atomic part; let us assume
this in the below.

\textbf{(Real-world case):} In real-world cases we typically do not
have access to $\cmeas$ itself but rather a finite set of $\msf M$
(possibly-noisy) observations $\cpt_1, \ldots, \cpt_{\msf M}$. In this
case, discretizing $\iset = \fmap(\dset)$ into $\msf N$ samples $\isn
= \set{\iptsf_1, \ldots, \iptsf_{\sfn}}$, we see that for each
$\iptsfj$ the integral expression for $\fbarysf(\iptsfj)$
(\ref{eq:discrete-bary}) reduces to a summation, making it simple to
compute. Then, we may take a gradient step \`a la
\cref{cor:disc-improvement} and iterate (note the $\msf{c_j}$'s in
\cref{cor:disc-improvement} may be chosen heuristically; see
\cref{sec:alg-discussion}).

In this scheme, we are free to choose $\msf N$ but not $\msf M$. This
begs the question: Given a fixed $\msf M$, what is a ``reasonable''
choice of $\msf N$? Observe that if $\msf N > \msf M$ then the
evolution behavior of $\isn$ will likely be highly distorted: On any
given iteration at most $\msf M$ of the $\iptsfj$ would have
$\fbarysf(\iptsfj) \neq 0$ (denote them $\msf{\iset_N^{\rm lead}}$),
while the remaining $\msf N - \msf M$ points have their behavior
dictated purely by $\lambda \nabla \cost(\fmap)$ (denote them
$\msf{\iset_N^{\rm follow}}$). Note that in the steady state we expect
the points of $\msf{\iset_N^{\rm follow}}$ to simply interpolate
$\msf{\iset_N^{\rm lead}}$ according to some $\msf{\fmap}$ achieving
$\msf C(\msf{\iset_N^{\rm lead}})$; thus the points of
$\msf{\iset_N^{\rm follow}}$ encode no meaningful information, and
might as well be omitted by choosing $\msf N \leq \msf M$ from the
start. Stated plainly, in the real-world case, our data sampling
resolution $\msf M$ is an upper bound on the maximum ``meaningful''
resolution of the approximant $\isn$.

\textbf{(Theoretical case):} Now suppose, by contrast, that we are
working in a context where we have direct access to the ground-truth
measure $\cmeas$ (since we assumed $\cmeas$ has nontrivial non-atomic
part, this typically occurs only in theoretical contexts). We may
obtain a Monte-Carlo scheme for $\fbarysf$ by simply drawing $\msf M$
samples from $\cmeas$ at each iteration and proceeding as in the
real-world case. Importantly, we are now free (modulo computational
constraints) to take as many samples as we want from $\cmeas$, \ie\
to make $\msf M$ arbitrarily large. This offers some important
benefits.
\begin{enumerate}[label=\arabic*.]
  \item Since there is no constraint on $\msf M$, we also no longer
    have a constraint on what constitutes a ``reasonable'' $\msf N$.
    In particular, it becomes possible to take $\msf N \to \infty$,
    which (see \cref{rem:N-encodes-complexity}) opens the door to
    studying an approximate asymptotic relationship between $\objfp$
    and $\msf C$.
  \item By taking $\msf M$ to be large, we can also improve the
    precision of our approximation of $\fbarysf$. This is noteworthy
    because the scheme can exhibit instabilities with respect to
    $\fbarysf$. Loosely speaking, the problem is that $\fbarysf$ is
    dependent on the $\psin^{-1}(\iptsfj)$, which themselves depend
    on the \emph{global} structure of $\isn$ (note, \eg\ that
    updating the position of one of the $\iptsfj$ typically modifies
    \emph{all} of the cells sharing a face with
    $\psin^{-1}(\iptsfj)$). Thus, a distortion in $\fbarysf$ at one
    iteration typically results in a distortion to $\isn$ at the
    next iteration, thereby distorting the subsequent $\fbarysf$,
    and so on. (In our experiments this was partially mitigated by
    the $\lambda \nabla \cost(\fmap)$ term but nonetheless caused
    some problems).
\end{enumerate}
In practice, however, a na\"ive implementation of this Monte-Carlo
scheme can require intractable values of $\msf M$ to achieve these
effects. We illustrate this in the special case $\ddim = 1$, $\cdim =
2$. For each $\msf j$, define
\[
  \sigma_{\msf j} = \Var_{\cmeas_{\iptsfj, \psin}}[(\cpt - \iptsfj)
  \abs{\cpt - \iptsfj}^{\objp -2}].
\]
Let $\msf N = 200$ and suppose we want to guarantee with $\geq 99\%$
confidence that we approximate each $\fbarysf(\iptsfj)$ within, say,
$\frac{\sigma_{\msf j}}{100}$ of the true value. At first glance this
degree of precision might appear overly demanding, but again recall
that since each step of the simulation depends on the previous steps,
small initial errors in $\fbarysf$ may propagate to significant
end-state differences.

In any case, denote the number of sample points $\cpt_{\msf j} \in
\psin^{-1}(\iptsfj)$ by $\msf{M_j}$; it suffices to choose the
$\msf{M_j}$ such that the probability of being within
$\frac{\sigma_{\msf j}}{100}$ of $\fbarysf$ is $\geq \sqrt[\msf
N]{.99} \approx .9999$. Via a standard $Z$-score argument we find that
choosing $\msf{M_j} \geq \ceil{(100 \cdot 3.89)^2} = 151274$ suffices.
Applying this to each $\msf{j}$, we see that we'd need at
\emph{minimum} $\msf M \approx \msf M_{\msf j} \times \msf N \approx 3
\times 10^7$ samples at each iteration to get the desired precision.

Actually, even this is not enough, as a priori there is no guarantee
the sample points will be uniformly distributed across the
$\psin^{-1}(\iptsfj)$. For example, if all the $\cmeas_{\iptsfj,
  \psin}$ have identical total mass, then one may compute that in
order to get $\geq 99\%$ probability that randomly assigning samples
to each $\psin^{-1}(\iptsfj)$ yields each $\msf{M_j} \geq 3 \times
10^7$, the number of points $\msf M$ to draw is lower bounded by $2.92
\times 10^{13}$, which is prohibitively large.

For some choices of $\cmeas$ this could be avoided by employing
\emph{low-discrepancy sequences}. It is essentially trivial to do so
for simple $\supp(\cmeas)$ (\eg, the unit square), but if $\partial
\supp(\cmeas)$ is sufficiently complex, this too could become
computationally burdensome. In any case, there remains a ``hard'' step
of assigning each of the $\msf M$ sample points to the corresponding
closest $\iptsfj$. Na\"ively this can be done in $O(\msf{MN})$ time;
however, this was prohibitively slow on our hardware.

A faster approach can be obtained by preprocessing the Voronoi cells
into an efficient \emph{point location} datastructure (see \eg\
\cite[Ch.\ 6]{Berg2008}), which reduces the complexity to
$O(\msf{M}\log(\msf N) + \msf N)$. However, as it turns out, when
$\ddim=1$, $\cdim = 2$, and $\cmeas$ is uniform, the preprocessing
required to compute these point location datastructures (typically
some form of triangulation procedure) is sufficient for running a fast
algorithm that computes $\fbarysf$ via approximating an
\emph{analytic} formula, thus obviating the need for a Monte Carlo
procedure. We discuss this algorithm now.

\subsubsection{A Deterministic Algorithm For Theoretical
  Applications} \label{sec:algorithm}

We now demonstrate that in at least one special case, for all $\objp
\geq 1$ one may obtain a fast algorithm for computing $\fbarysf$. Note
that some of the precise implementation details and runtime analysis
are beyond the scope of this paper, and hence we postpone them for a
later work.

The particular requirements for our algorithm are
\begin{itemize}
  \item $\ddim=1$,
  \item $\cdim=2$, and
  \item $\cmeas$ is uniform with $\partial \supp(\cmeas)$ piecewise
    linear (PL). \label{term:PL}
\end{itemize}
With these hypotheses we are able to represent $\fbarysf$ analytically
rather than by approximation via Monte Carlo methods. As far as we are
aware this has not been done before. Performance is discussed briefly
in \cref{sec:alg-discussion}; for now, we focus on the high-level
consequences. Essentially, by removing this source of randomness, our
method appears to offer better stability for properties of theoretical
interest such as symmetry, self-similarity, and so on, as can be seen
comparing the curves in \cref{fig:voronoi-sims} against the curve
obtained via stochastic gradient descent in \cite[Fig.\
1]{delattre2020}. We believe this stability is highly desirable when
using the algorithm for purposes such as formulating conjectures,
\eg\ ``\emph{if $\cmeas$ is uniform, then as $\objp \to \infty$, to
  what extent does the regularity of a prototypical `good solution'
  tend towards that of $\partial \cset$?}''

We choose to represent \(\fmap\) by a cubic spline, which we
parametrize by its knots. We initialize with some seed function
$\fmap_0$ (either chosen randomly or set by the user) and then iterate
as follows. In the first step of the $j$\textsuperscript{th}
iteration, we sample $\sfnj = \ceil{\frac{\arclen(\fmap_j)}{\delta}}$
points from \(\fmap_j\) uniformly with respect to arc length, where
$\delta$ is a parameter controlling the maximum distance between two
adjacent points. Call the sampled points $\isnj$. To compute $\isnj$
efficiently, we give a numerically stable method to calculate the
roots of the ``squared speed'' quartic, which we can substitute into
the routine of \cite{henkel2014calculating}. That routine then
calculates the arc length of a cubic parametric curve in two
dimensions by reduction to an elliptic integral, which was previously
solved in \cite{carlson1992table}. Using this we can efficiently
compute an approximate arc length parametrization of $\fmap_j$ using
Newton's method since the speed of the curve is easily computed. Thus
$\isnj$ is given by evaluating evenly-spaced points of
$[0,\arclen(\fmap_j)]$.

Once we have $\isnj$, we compute the Voronoi cells
$\psin^{-1}(\iptsfj)$ using a standard algorithm. Then for each cell,
we compute $\fbarysfj$ by decomposing the cell into triangles and
computing the contribution of each triangle individually. If
$\supp(\cmeas)$ is convex, then $\supp(\cmeas) \cap
\psin^{-1}(\iptsfj)$ is a convex polygon and we obtain a triangulation
by enumerating the vertices $\set{\msf{v}_1, \ldots, \msf{v}_{M}}$ and
iterating through the triangles $\triangle\msf v_{1} \msf v_{i} \msf
v_{i+1}$ where $i = 2, \ldots, M-1$. Otherwise, we run a
Weiler-Atherton routine \cite{Weiler1977Jul} to compute $\supp(\cmeas)
\cap \psin^{-1}(\iptsfj)$ and apply a standard triangulation algorithm
to the result.

Since $\cmeas$ is uniform, any two of the resulting triangles
intersect on at most a $\cmeas$-null set, whence we avoid
double-counting in \eqref{eq:discrete-bary}. Again since $\cmeas$ is
uniform, each triangle integral has a ``nice'' expression for general
$\objp$ in terms of the Gaussian hypergeometric function, which can be
evaluated efficiently. Here, ``nice'' is in scare quotes because while
the expressions are numerically efficient to evaluate, writing them
down by hand for general $\objp$ involves some slightly-messy
casework. Hence, we simply note that in the $\objp = 2$ case, the
contribution of triangle $\triangle \msf{uvw}$ can be computed instead
via
\[
  \textstyle 2 \cdot \pn{\frac{1}{3} \pn{\msf u + \msf v + \msf w} -
    \iptsfj} \cmeas(\triangle \msf{uvw}),
\]
where $\cmeas(\triangle \msf{uvw})$ can be computed \eg\ via Heron's
formula. Summing over all triangles then yields $\fbarysfj$ precisely.
Note that in practice we found that performance sometimes improved
significantly when we reweighted $\fbarysfj$ by
$1/\pn{\cmeas(\psin^{-1}(\iptsfj))}^\kappa$, where $\kappa \in
\bk{0,1}$ is a parameter chosen by the user.

Next, we approximate the gradient $\nabla \cost$ of the Sobolev
constraint with respect to the points $\isnj$ as follows. First note
that since $f$ is a piecewise cubic polynomial, it will generally lack
regularity higher than third differentiability, which is insufficient
for computing $\mathscr C(f)$ for general $k$. Hence we first compute
a new, sufficiently high-order spline $\widetilde f$ interpolating the
arclen-evenly-spaced samples of $f$, and represent it via B-Splines.
Then, we approximate the Sobolev norm of $\widetilde f$ with the
vector $q$-norm of the derivatives of all orders evaluated at each of
the interpolated points (by the regularity properties of the splines,
this scheme is consistent).

We then take the gradient of this expression with respect to the
locations of the points $\isnj$, resulting in an approximation
$\msf{G_{\isnj}}$. This can be done efficiently because the
derivatives at the sample points depend linearly on the points being
interpolated and moreover the linear relationship has the form of the
inverse of a $(k+1)$-width banded matrix, which can be computed very
quickly. Finally, given a learning rate $\learningrate$ we make the
perturbation $$\msf{\iset_{N_{j+1}}} = \isn + \learningrate (\fbarysfj
- \lambda \msf{G_{\isnj}}).$$ These perturbed points then become the
knots of the next curve, and the algorithm terminates after reaching a
prescribed number of iterations.

Initially, \(\fmap\) does not cover \(\cset\) well and so the
barycenter field has very large magnitude. To counteract this, we
choose \(\learningrate\) to be small near the beginning of the
simulation and increase it as a function of the iteration index $i$.
We find that when we do this the arc length of the curve grows at a
steady rate. We show the resulting curves at several points in time in
\cref{fig:voronoi-sims}. The curves themselves are in blue, the
Voronoi cells are drawn with red edges and cell barycenters are shown
as green dots. Note that regions with dark-red shading represent the
convergence of many Voronoi cell boundaries while white regions
represent regions where few Voronoi cell boundaries converge. Hence
(with the exception of the white regions corresponding to the
endpoints) dark-red and bright-white regions tend to reflect base
points on the curve where curvature is high.

In any case, we see that as time progresses, the curve begins to
evenly fill the square in a fractal-like manner.

\begin{figure}[H]
  \centering
  \begin{subfigure}{.49\linewidth}
    \centering
    \includegraphics[scale=.15]{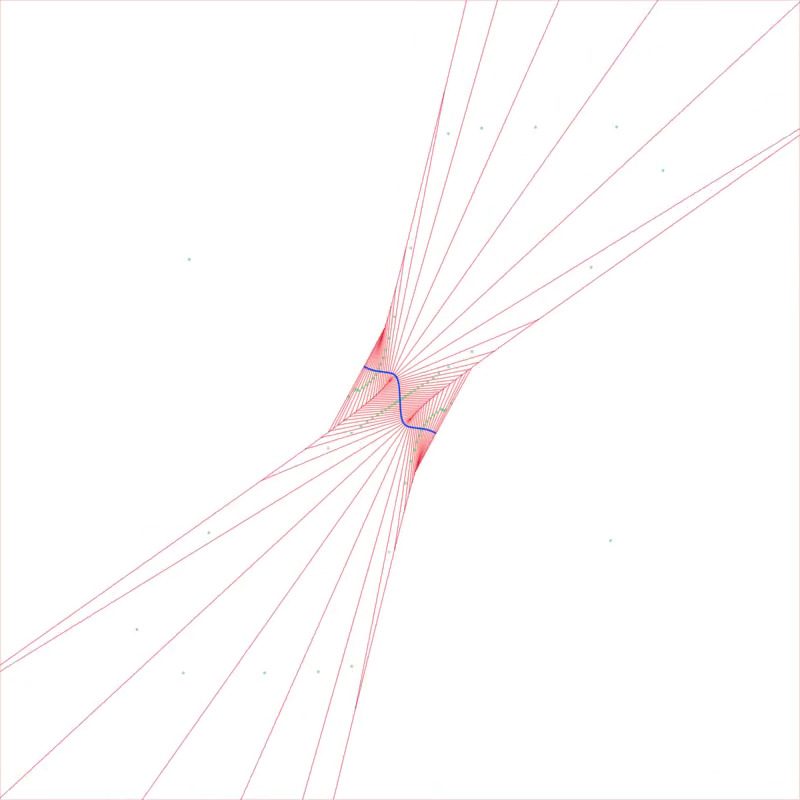}
    \caption{$t=0$}
  \end{subfigure}
  \begin{subfigure}{.49\linewidth}
    \centering
    \includegraphics[scale=.15]{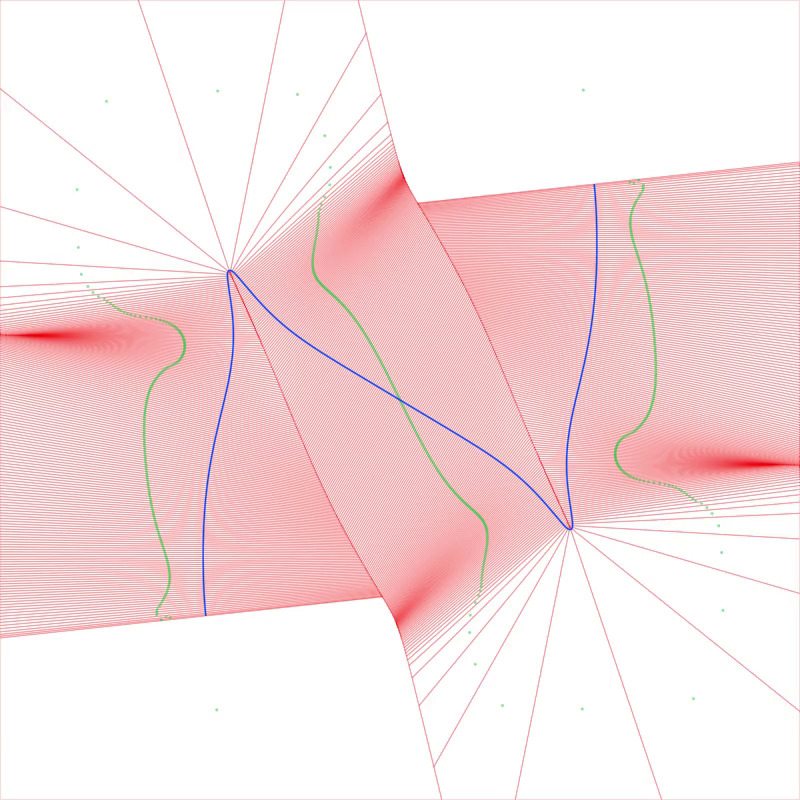}
    \caption{$t=3$}
  \end{subfigure}
\end{figure}\vspace{-1.5em}
\begin{figure}[H] \continuedfloat
  \begin{subfigure}{.49\linewidth}
    \centering
    \includegraphics[scale=.15]{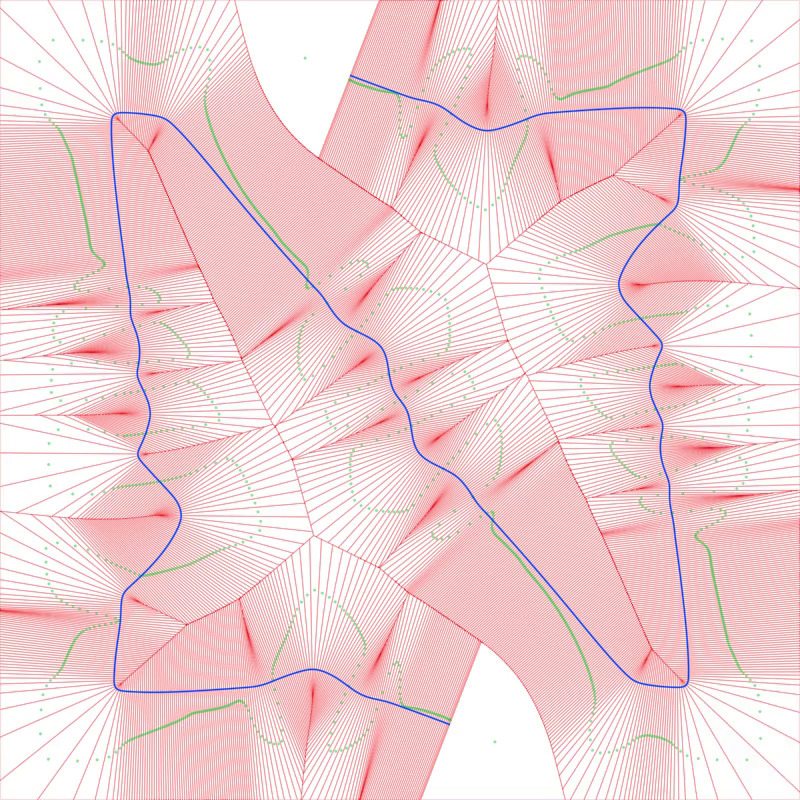}
    \caption{$t=6$}
  \end{subfigure}
  \begin{subfigure}{.49\linewidth}
    \centering
    \includegraphics[scale=.15]{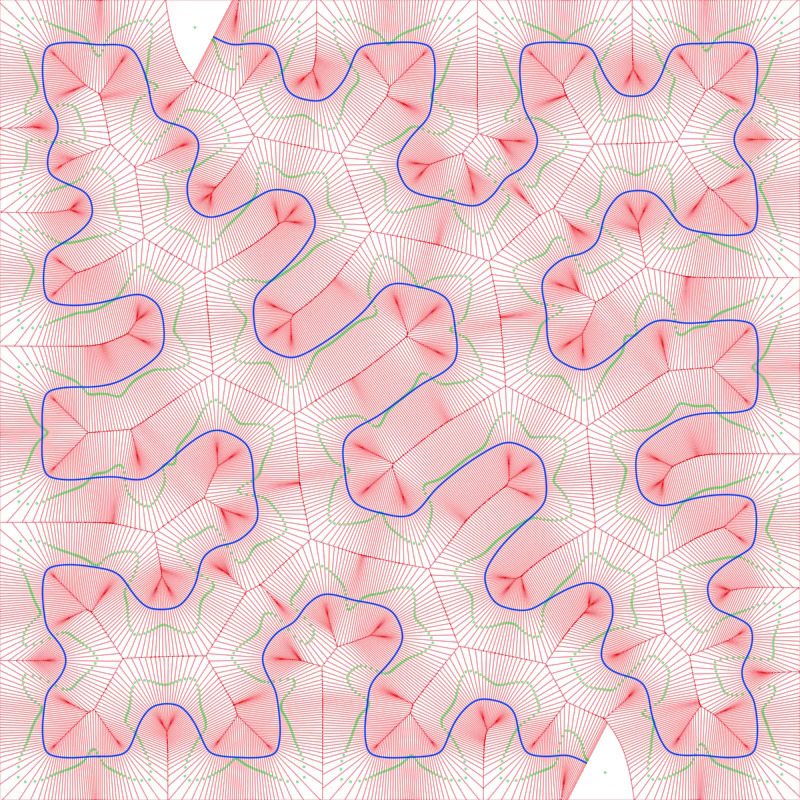}
    \caption{$t=9$}
  \end{subfigure}
\end{figure}\vspace{-1.5em}
\begin{figure}[H] \continuedfloat
  \begin{subfigure}{.49\linewidth}
    \centering
    \includegraphics[scale=.15]{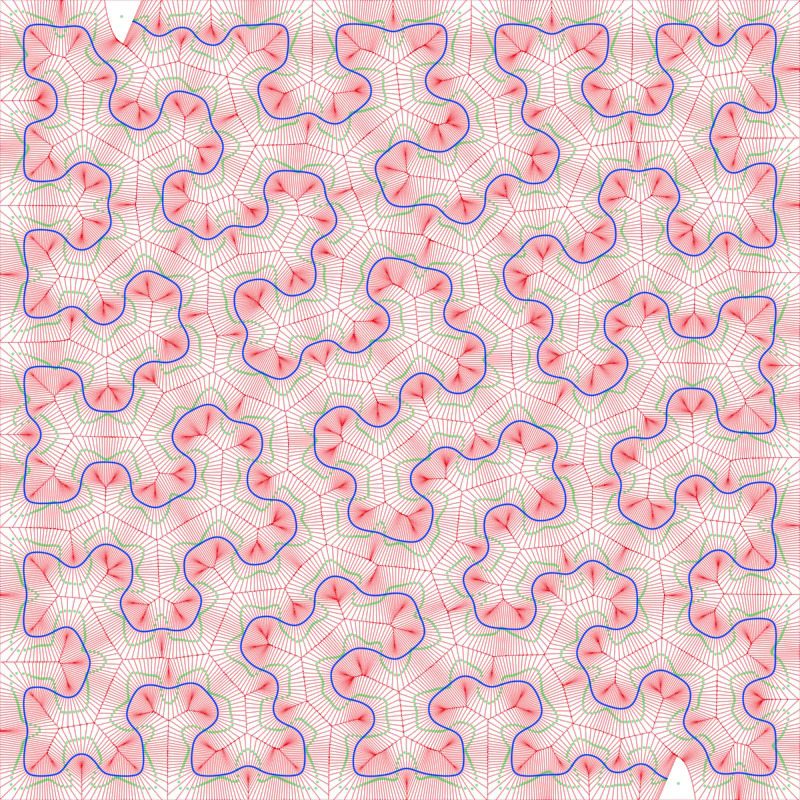}
    \caption{$t=11$}
  \end{subfigure}
  \begin{subfigure}{.49\linewidth}
    \centering
    \includegraphics[scale=.15]{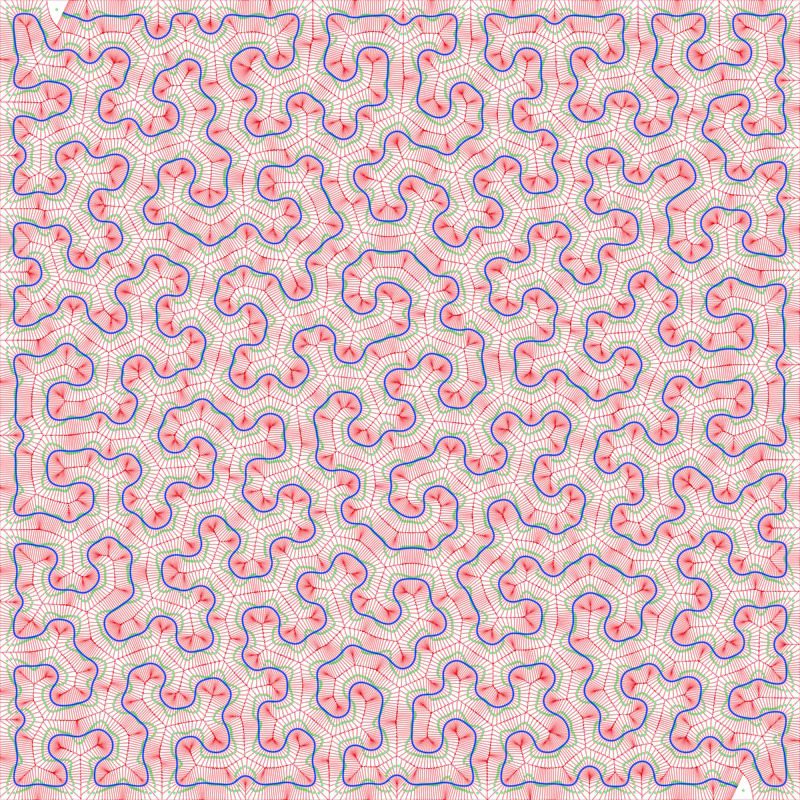}
    \caption{$t=13$}
  \end{subfigure}
  \caption{Evolving a curve according to $\fbarysf$.}\label{fig:voronoi-sims}
\end{figure}

We also experimented in cases where $\lambda$ decays as
$\frac{1}{\learningrate} - 1$, this also yielded good curves,
especially when the seed $\fmap_0$ was very messy, \eg\ interpolating
random points. See \cref{fig:regularized-curves} for an example.

\subsubsection{Discussion}
\label{sec:alg-discussion}

The algorithm is quite fast in practice when $\supp(\cmeas)$ is
convex, typically approaching a steady state after 10-20 seconds on a
Mid-2015 MacBook Pro. For simple nonconvex domains like the one in
\cref{sec:hexagonal-star-domain}, performance was similar. However, we
did see noticeable slowdowns when $\supp(\cmeas)$ is highly nonconvex,
\eg\ the branching shape in \cref{fig:regularized-curves}.

One drawback of the approach above is that it involves a long list of
somewhat-obscure parameters to tune (\eg\ the $\kappa$ in the
$\fbarysfj$ reweighting step), though in practice when $\supp(\cmeas)$
was only mildly-nonconvex we saw decent behavior over a wide range of
parameter values. Implementations that include a parameter search can
remove some of this hand-tuning; for example, the choices of
$\learningrate(i)$ and $\lambda(i)$ can be done away with by employing
a line search. In our experiments, this yielded greatly-improved
stability, though at a cost to runtime.

An alternative approach is to try and derive theoretical heuristics
that yield decent (though perhaps not optimal) behavior. As a case
study, let us consider the learning rate $\msf c$ that we use to scale
$\msf{F}_{\hat \pi_{\msf \iset}}$ for each update. What can the theory
tell us about how to choose a ``good'' value of $\msf c$?

Some guidance comes from examining the $\msf{c_j}$ in the proof of
\cref{cor:disc-improvement}. Note that the $\eta_{\dsfj}$ in the
definition were chosen for two reasons. First, the $\eta_{\dsfj}$
ensured that the effects of the perturbations at each $\iptsfj$ were
isolated from one another. In practice, we can guarantee this instead
by choosing our resampling resolution such that at each iteration we
have enough points to ensure $\inf_{\msf j \neq \msf{j'}} d(\iptsfj,
\iptsf_{\msf{j'}}) \ll \inf_{\msf{j}} \diam(\psin^{-1}(\iptsfj))$,
whence the effects of each $\pertsf$ at each $\iptsfj$ remain
isolated. Second, the controlled smoothness of the $\eta_{\dsfj}$'s
gives us an elementary way to ensure $\costsf(\isneps) < \costsf(\isn)
+ \epssfn$. But since large changes in the cost are penalized already
by the Sobolev gradient approximation $\lambda \msf{G}$ (and there is
little harm in \emph{undershooting} the proper $\msf{c_j}$), we may
choose to ignore the $\norm[]{\eta_{\dsfj}}_{\wkpsets}$ and take
$\msf{c_j} \propto \frac{1}{\sfn \abs[]{\fbarysf(\iptsfj)}}$. Applying
the coarse bound on $\epssfn$ \eqref{eq:coarse-epssfn} then gives
\begin{align*}
  \epssfn \msf{c_j} \leq \frac{\objfp(\isn)}{C \sfn
  \abs[]{\fbarysf(\iptsfj)}}
  &= \frac{\sum_{\msf i} \int_{\psin^{-1}(\iptsfi)} \abs{\cpt -
    \iptsfi}^\objp \dd \cmeas}{\abs[]{\fbarysf(\iptsfj)} \sum_{\msf i}
    \abs[]{\fbarysf(\iptsfi)}},
\end{align*}
whence taking the coarse approximation
\[
  \int_{\psin^{-1}(\iptsfi)} \abs{\cpt - \iptsfi}^{\alpha} \dd \cmeas
  \approx C \max_{\cpt \in \psin^{-1}(\iptsfi)} \abs{\cpt -
    \iptsfi}^{\alpha + 1} \approx C'
  \pn{\diam(\psin^{-1}(\iptsfi))}^{\alpha + 1}
\]
and supposing the $\diam(\psin^{-1}(\iptsfi))$ are relatively uniform
in $\msf i$ (not unreasonable on average since $\cmeas$ is uniform),
we obtain the \emph{very} coarse heuristic of selecting $\epssfn
\msf{c_j} = O(\diam(\psin^{-1}(\iptsfj))^{1-\objp})$. However, in
practice we do not compute $\diam(\psin^{-1}(\iptsfj))$, and we do not
have \emph{a priori} estimates on how quickly the
$\diam(\psin^{-1}(\iptsfi))$ decay. Obtaining such estimates could be
a useful direction for future work.

In any case, denoting the iteration index by $\msf i$ and the total
iterations by $N_{\rm tot}$, empirically we have observed good
performance choosing $\epssfn \msf{c_j} \propto (\msf{i}/N_{\rm
  tot})^\objp$, but as of yet we do not have good theoretical
justification.

\section{Toward Applications to Generative Learning}
\label{sec:generative-learning}
Note that in the below we consider the soft penalty problem
\xref{prob:soft-penalty} rather than the hard constraint
\xref{prob:hard-constraint}. However, as discussed in
\cref{sec:the-general-problem}, \xref{prob:soft-penalty} gives
solutions to \xref{prob:hard-constraint} while being easier to
simulate.

As mentioned in \cref{sec:unequal-dimensional-ot}, one possible
application of \xref{prob:soft-penalty} is in studying generative
modeling problems. In this section, we will expand on those remarks.
We begin by giving a brief overview of generative machine learning
problems (\cref{sec:brief-overview-of-ml}) and then discuss how our
framework can be applied to (unconditional) generative modeling
(\cref{sec:applying-to-ml}). In particular, we show that
\xref{prob:soft-penalty} is essentially a regularized, non-stochastic
analogue of a certain class of generative machine learning problems,
and propose mechanisms beyond the standard ``avoidance of
overfitting'' by which the regularization can offer training
improvements. In light of this connection, we propose that it could be
of interest to further investigate the theoretical properties of
\xref{prob:soft-penalty}, as well as those of the associated numerical
methods (\eg\ \cref{sec:algorithm}), with the hope that insights from
the \xref{prob:soft-penalty} context could help inform work on
generative learning problems.

To that end, in \cref{sec:continuous-opt} we offer some
proof-of-concept experiments demonstrating that adding $\lambda \cost$
to $\objfp$ for a simple feedforward neural network yields training
improvements in two generative tasks: learning the uniform measure on
a disk (\cref{sec:fitting-the-disk}) and learning to generate images
of handwritten digits (\cref{sec:fitting-mnist}). Finally, in
\cref{sec:discrete-optimization} we give a simple MNIST experiment
demonstrating that the regularization benefits we propose appear in
real data even when the data has been in some sense ``maximally
overfitted.''

More work would be required to understand whether the phenomena we
observe persist in more complicated settings.

\subsection{Brief Overview of Generative Learning Problems}
\label{sec:brief-overview-of-ml}

In generative modeling, it is typical to treat our training data as
being samples from some abstract distribution \(\cmeas\) and to view
the generative model itself as an estimating distribution
\(\sigma_\theta\) parametrized by \(\theta\). The goal is to find
\(\theta\) such that \(\sigma_\theta\) is close to \(\cmeas\) in some
statistical sense
\cite{arjovsky2017,goodfellow2020generative,fisher1925theory,myung2003tutorial}.
In many cases \(\cmeas\) is embedded in \(\rcdim\) with large
\(\cdim\), however in most applications it is reasonable to assume
\(\cmeas\) is concentrated on a set of local dimension \(\ddim \ll
\cdim\)
\cite{donoho2003hessian,fefferman2016testing,roweis2000nonlinear,tenenbaum2000global,hein2005intrinsic,cayton2008algorithms,carlsson2009topology,nakada2020adaptive,brahma2015deep}.

A further distinction can be drawn between \emph{unconditional} and
\emph{conditional} generative learning problems. In
\emph{unconditional} problems, the goal is essentially to approximate
a random sampling scheme for $\cmeas$ by randomly sampling
$\sigma_\theta$. To facilitate this, $\sigma_\theta$ is typically
constructed as the pushforward of a fixed, high-regularity
distribution $\mu$ (\eg\ a Gaussian) under a learned map
$\fmap_\theta$, whence sampling $\sigma_\theta$ is essentially trivial
(see \eg
\cite{goodfellow2020generative,papamakarios2021normalizing,ho2020denoising}).
Note that in this case, the user cannot specify a particular
\emph{part} of $\cmeas$ they want to sample from. For example, if
training a model to generate images of common pets, a user cannot
specify whether they want to see a cat or a dog.

By contrast, a \emph{conditional} model takes an input prompt from a
user (\eg\ ``a Labrador retriever'') and tries to sample from the
conditional distribution $\cmeas\vert_{Q}$, where $Q$ is the set of
outputs that can be described by the user query
\cite{saharia2022photorealistic,ramesh2021zero,mirza2014conditional}.
Although conditional modeling is generally more difficult, an
unconditional model can sometimes be extended to support conditioning
after the fact (see, \eg\ works on guided diffusion
\cite{kim2022diffusionclip,dhariwal2021diffusion,crowson2021clip,nichol2022glide},
where an unconditional diffusion model is extended to support
conditioning). Thus, we will leave conditional generation to future
work and restrict ourselves to examining unconditional generation.

A prototypical example of an unconditional generation problem is
generating images from a distribution. Supposing the images have
resolution \(512 \times 512\) it is typical to embed them into \(\cset
\subseteq \rcdim\) for \(\cdim = 512 \times 512 \times 3 = 786432\)
where each dimension encodes one of the RGB components of the color of
one pixel. We imagine that there is some real-world stochastic process
that randomly produces images in $\cset$; the measure $\cmeas$ then
describes how likely a given set of images is to contain the next one
produced. The model itself is a neural network \(\fmap_\theta\) with
parameters \(\theta\); $\fmap_\theta$ is used to transform a
distribution $\dmeas$ on \(\rddim\) (usually Gaussian) to the
distribution \(\sigma_\theta\) on \(\rcdim\) via
\[
  \sigma_\theta = (\fmap_{\theta})_{\#} \dmeas.
\]
In practice, the precise relationship between \(\theta\) and the shape
of \(\fmap\) may be nontrivial. In any case, the network
$\fmap_\theta$ is often taken to be either smooth or piecewise linear
so we may assume $\sigma_\theta$ has local dimension \(\ddim\) a.e.\
(\cf\ \cref{prop:effectiveness}). Since perturbing a natural-looking
image in a random direction is overwhelmingly likely to degrade its
appearance \cite{pope2021intrinsic}, it is typical to assume $\cmeas$
is concentrated on a very thin set in $\cset$. Thus we may take
$\ddim$ to be relatively small---for example \(\ddim = 128\)---and
still hope for $\sigma_\theta$ to approximate $\cmeas$ well.

\subsubsection{A Case Study: WGAN}\label{sec:WGAN-case}
As proposed in \cite{arjovsky2017}, when $\cmeas$ is concentrated on a
thin set, one good choice for the notion of statistical ``closeness''
of $\sigma_\theta$ and $\cmeas$ is the Monge-Kantorovich 1-distance.
Note that $\objp = 1$ is chosen essentially just because
$\wass_1(\cmeas, \sigma_\theta)$ can be represented
computationally-tractably via the dual formula
\begin{equation}
  \wass_1(\cmeas, \sigma_\theta) = \sup \set{\int \varphi \dd (\cmeas
    - \sigma_\theta) \MID \varphi \text{ is
  $1$-Lipschitz}}. \label{eq:KR-duality}
\end{equation}
See \eg\ \cite[Thm.~1.14]{villanitopics} for a rigorous proof of
\eqref{eq:KR-duality}. In any case, the authors of \cite{arjovsky2017}
use \eqref{eq:KR-duality} to propose the so-called \emph{Wasserstein
  GAN} (WGAN), which can offer notable improvements to training
stability over classical GAN designs.

To summarize, WGAN alternatingly trains two separate networks,
parametrized by $w$ and $\theta$, respectively: a \emph{critic} or
\emph{discriminator} $D_w$ and a \emph{generator} $\fmap_\theta$. The
critic plays the role of $\varphi$ in \eqref{eq:KR-duality} and thus
for fixed $\theta$ tries to learn to estimate $\wass_1(\cmeas,
\sigma_\theta)$. Then, the generator uses the trained critic to
estimate how best to adjust $\sigma_\theta$ to decrease
$\wass_1(\cmeas, \sigma_\theta)$. Explicitly, for fixed $\theta$ we
train $D_w$ by iteratively sampling $M$-point batches $\cpt^{(i)} \sim
\cmeas$ and $\dpt^{(i)} \sim \dmeas$ and updating $D_w$ via a small step
along
\[
  \textstyle \nabla_w \bk{\frac{1}{M} \sum_{i=1}^M \left[D_w(\cpt^{(i)}) -
    D_w(\fmap_\theta(\dpt^{(i)})) \right]} \approx \nabla_w \int D_w \dd (\cmeas -
  \sigma_\theta),
\]
using some regularization scheme to try and ensure $D_w$ remains
1-Lipschitz. After training the critic for some number of iterations
$N_{\rm critic}$, we then take another $M$-point batch $\dpt^{(i)} \sim
\dmeas$ and update the generator $\fmap_\theta$ via a small step along
\begin{equation}
  \textstyle
  \nabla_\theta \bk{\frac{1}{M} \sum_{i=1}^M
    D_w(\fmap_\theta(\dpt^{(i)}))} \approx
  \nabla_\theta \int D_w \circ \fmap_\theta \dd \mu =
  -\nabla_\theta \int D_w \dd(\cmeas -
  \sigma_\theta). \label{eq:generator-update}
\end{equation}
This process is then repeated until convergence.

In the original work \cite{arjovsky2017}, $D_w$ was regularized via
the (as they described it) ``clearly terrible'' method of \emph{weight
  clipping}. In \cite{gulrajaniImprovedTrainingWasserstein2017} it was
demonstrated that a better method is to instead augment the objective
\eqref{eq:KR-duality} with a \emph{gradient penalty} term
\begin{equation}
  \textstyle \frac \lambda M \sum_{i=1}^M
  \bk[]{\abs[]{\nabla_{u_{\theta}^{(i)}} D_w(u_{\theta}^{(i)})} -
    1}^2, \label{eq:wass-gp-penalty}
\end{equation}
where $u_{\theta}^{(i)} = s_i \cpt^{(i)} + (1-s_i)
\fmap_\theta(\dpt^{(i)})$ for $s_i \sim \mrm{Unif}[0,1]$. The
justification was that given an optimal coupling $\pi_\theta^*$ for
$\wass_1(\cmeas, \sigma_\theta)$, one can show
\cite[Prop.~1]{gulrajaniImprovedTrainingWasserstein2017} that for
$\pi_\theta^*$-a.e.\ pair $\cpt, \cpt' \in \cset^2$, if an optimal
critic $D_w^*$ is differentiable at $\cpt_t = t\cpt + (1-t) \cpt'$,
then $\abs{\nabla D_w^*(\cpt_t)} = 1$. However, in
\cite{Petzka2017Sep} it was demonstrated that this scheme has some
noteworthy problems, a key one being that the product measure $\cmeas
\otimes \sigma_\theta$ often gives positive mass to null sets of
$\pi_\theta^*$, and hence it is not guaranteed that sampling
$\cpt^{(i)} \sim \cmeas$, $\fmap_\theta(\dpt^{(i)}) \sim
\sigma_\theta$ will yield a point in $\supp(\pi_\theta^*)$. Therefore,
the authors of \cite{Petzka2017Sep} propose replacing the
${\abs{\nabla D_w} - 1}$ term in \eqref{eq:wass-gp-penalty} with just
the positive part $\pn{\abs{\nabla D_w} - 1}_+$, thus enforcing only
the 1-Lipschitz condition. An interpretation of this revised
formulation (called WGAN-GP) in terms of a congested transport problem
was introduced in \cite{Milne2022Jun}.

In any case, given the relationship between $\objfp$ and $\wpp(\cmeas,
\cdot)$ we established in \cref{prop:connection-to-ot}, the inclusion
of these penalization terms makes WGAN-GP evocative of
\xref{prob:soft-penalty}, the difference being that the penalty in
\xref{prob:soft-penalty} enforces regularization on the
\emph{generator} rather than on the critic.
We discuss this comparison now.

\subsection{WGAN and The Constrained Monge-Kantorovich Fitting
  Problem}
\label{sec:applying-to-ml}

By \cref{prop:connection-to-ot}, we see that when $\objp = 1$,
\xref{prob:soft-penalty} and WGAN essentially treat the same
objective, but there are two main distinctions between the
formulations.

First, \xref{prob:soft-penalty} depends solely on the geometry of the
image set $\iset = \fmap_\theta(\dset)$---in particular it is agnostic
with regards to the distribution $\dmeas$ on $\dset$---while WGAN
explicitly seeks to minimize $\wpp(\cmeas, \sigma_\theta)$. Thus, as
we proposed in \cref{sec:unequal-dimensional-ot}, we may loosely view
\xref{prob:soft-penalty} as the first step in a ``factorization'' of
WGAN into two parts:
\begin{enumerate}[label=\arabic*)]
  \item Learning the \emph{shape} of $\supp(\sigma_\theta^*)$, where
    $\sigma_\theta^*$ is an optimal approximating measure, and then
  \item Learning a reparametrization map $\varphi$ so that
    $\sigma_\theta^*$ can be expressed as the pushforward of a
    particular $\mu$.
\end{enumerate}
Intuitively, we expect step 1 might be the ``hard'' step (though we
have not investigated this carefully) as it involves learning an
$\ddim$-to-$\cdim$ map (typically $\ddim \ll \cdim$), while step 2
involves learning only an $\ddim$-to-$\ddim$ map. However, we note
that step 2 can still be nontrivial if the optimal $\varphi$ badly
lacks regularity.

The second main difference is that \xref{prob:soft-penalty} enforces
regularization on $\fmap$ (which plays the role of the WGAN generator)
rather than on the critic. To our knowledge, the closest analogous
idea in the literature is the recent work of \cite{Vardanyan2024Jul},
in which a ``left-invertibility'' penalty was used to regularize the
WGAN generator. The authors of that work found that such a
regularization scheme resulted in improved creativity of the model,
without sacrificing statistical optimality. (This is loosely analogous
to our discussion of generalization performance in the toy problem in
\cref{sec:discrete-optimization}).

By analogy with our ``factored'' WGAN problem, we now propose some
qualitative explanations for why such behavior might be expected.

\subsubsection{Classical Benefits of Regularization}
\label{sec:classical-regularization-benefits}
Classically, the training benefits of regularization terms like
$\lambda \cost$ terms are typically understood in terms of ``avoiding
overfitting.'' In our language we may understand this as follows:
Since we are in a ``real-world'' context
(\cref{sec:real-vs-theoretical}) we do not actually have access to
$\cmeas$, but rather a finite number of (possibly noisy) training
samples, which we denote by the empirical measure $\cmeas_{\msf M}$.
By \cref{cor:consistency-sp}, when $\msf M$ is large, an optimizer
$\fmap_{\msf M}$ of $\objfpl(\tinybox; \cmeas_{\msf M})$ approximates
an optimizer of $\objfpl(\tinybox; \cmeas)$ with high probability.

However, in practice, the threshold for what constitutes a
``sufficiently large'' $\msf M$ depends on $\lambda$. Essentially this
is because for any fixed $\msf M$ the constraint is not
\emph{effective} for $\cmeas_{\msf M}$ (see
\cref{prop:effectiveness}), and so if $\lambda$ is very small,
$\fmap_{\msf M}$ will more or less interpolate the atoms of
$\cmeas_{\msf M}$. Typically this yields poor generalization
performance because $\fmap_{\msf M}$ has ``fit'' random variations in
$\cmeas_{\msf M}$ that are not present in $\cmeas$. On the other hand,
we expect that for a proper choice of $\lambda$ (\ie\ not too small)
the $\lambda \cost$ term will force $\fmap_{\msf M}$ to ``average
out'' the noise in $\cmeas_{\msf M}$, thus helping to ensure
$\fmap_{\msf M}$ is a good approximant of \emph{both} $\cmeas$ and
$\cmeas_{\msf M}$, especially perhaps when the differentiability order
$\sobk$ in the Sobolev constraint is chosen based on the regularity of
$\cmeas$.

\subsubsection{Proposed Additional Benefits of Regularization}
\label{sec:additional-regularization-benefits}
In light of the two-step ``factorization'' of WGAN we proposed
earlier, we hypothesize that the inclusion of the regularization term
$\lambda \cost$ actually offers an additional benefit: Encouraging
solutions $\iset$ for step 1 such that the optimal reparametrization
$\varphi$ in step 2 has a small (local) Lipschitz constant at most
points (thus potentially improving training stability). In particular
we propose that $\lambda \cost$ encourages the following properties in
$\iset$:
\begin{enumerate}[label=\arabic*)]
  \item Avoidance of self-intersections,
  \item Few points outside $\supp(\cmeas)$, and
  \item Better preservation of locality, in the sense that the
    Lipschitz constant of $\fmap_\theta$ is relatively small.
\end{enumerate}
To illustrate the point in cases 1 and 2, suppose $\dmeas$, $\cmeas$
are uniform and $\fmap_\theta$ has constant speed, so that
$\sigma_\theta$ is uniform on $\iset$. Then near self-intersections,
the local density of $\sigma_\theta$ is approximately double that of
$\cmeas$, and so a reparametrization $\varphi$ would need to
compensate by moving faster near these points so as to lower the
density of $\sigma_\theta$. Similarly, near a point outside
$\supp(\cmeas)$, $\varphi$ should become discontinuous as it
approaches optimality so as to ensure the density of $\sigma_\theta$
vanishes. For case 3, suppose $\dmeas$ is uniform, $\cmeas$ is a
bimodal Gaussian, and $\fmap_\theta$ is constant speed. If $\iset$
contains many oscillations between the two peaks, then $\varphi$ will
have to be highly oscillatory to compensate; on the other hand, if
$\iset$ ``fills'' the region near the first peak first and only then
moves to the second, $\varphi$ can be taken to be smoother.

To give a tangible example of these effects, we initialize the
algorithm in \cref{sec:algorithm} on a highly nonconvex domain with an
initial condition given by interpolating 200 randomly-chosen points,
with learning rate $\eta(i) = (i/2000)^2$ and decaying regularization
parameter $\lambda(i) = 0.01(1 - \learningrate(i))/\learningrate(i)$
(\cref{fig:regularized-curves}). After just 20 iterations the curve
has a substantially simplified topology. By 110 iterations the curve
has no more self-intersections, and by 800 iterations it has learned
the basic shape of the domain, with only one arc outside
$\supp(\cmeas)$, though this arc later disappears at around iteration
1120. By iteration 1850, the behavior of the curve has stabilized, and
there are no self-intersections or external arcs, with relatively even
coverage of $\cmeas$. In this case, good performance can be obtained
by taking $\varphi$ approximately uniform.
\begin{figure}[H]
  \centering
  \begin{subfigure}[c]{.49\linewidth}
    \centering
    \includegraphics[scale=.15]{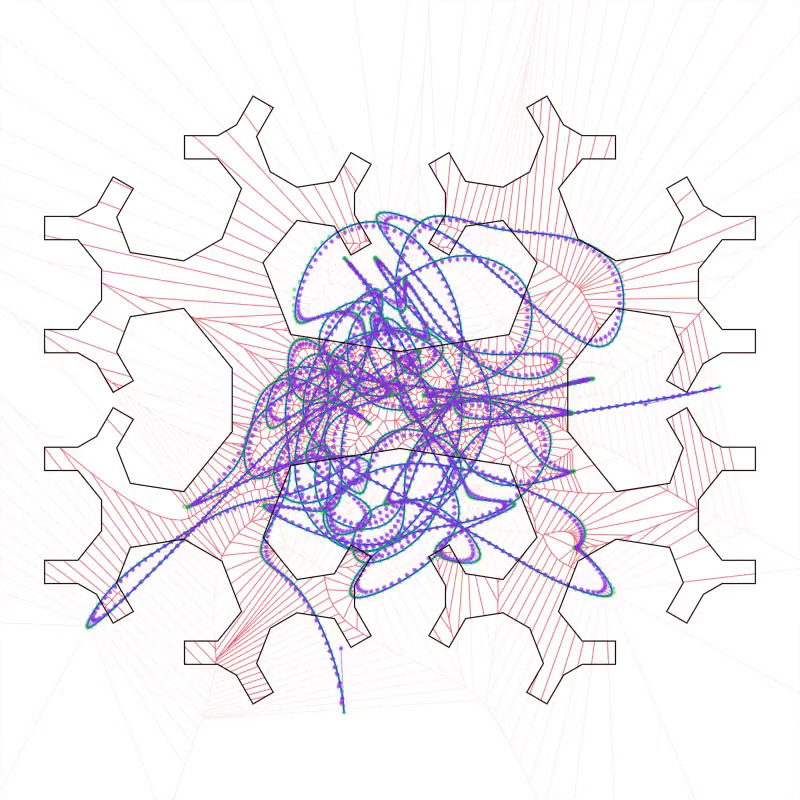}
    \caption{$i=1$}
  \end{subfigure}
  \begin{subfigure}[c]{.49\linewidth}
    \centering
    \includegraphics[scale=.15]{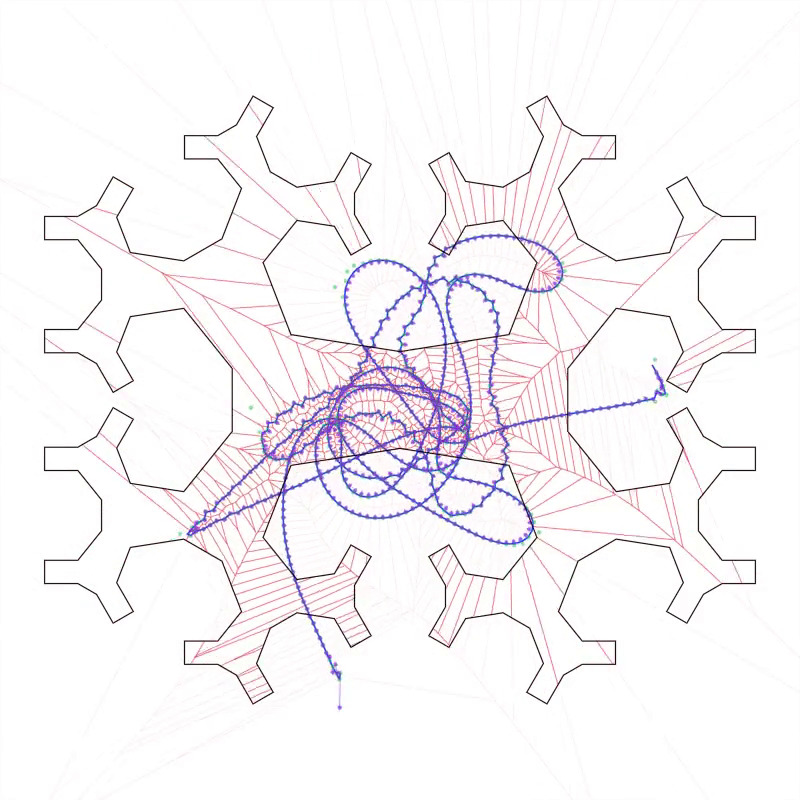}
    \caption{$i=20$}
  \end{subfigure}
  \vspace{-1.5em}
\end{figure}
\begin{figure}[H]
  \ContinuedFloat
  \begin{subfigure}[c]{.49\linewidth}
    \centering
    \includegraphics[scale=.15]{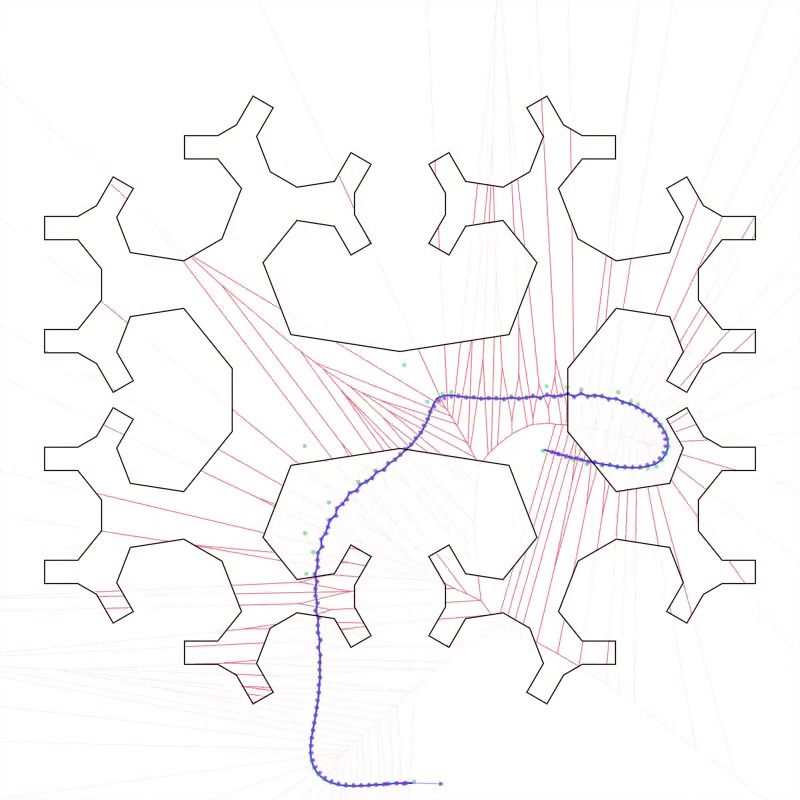}
    \caption{$i=110$}
  \end{subfigure}
  \begin{subfigure}[c]{.49\linewidth}
    \centering
    \includegraphics[scale=.15]{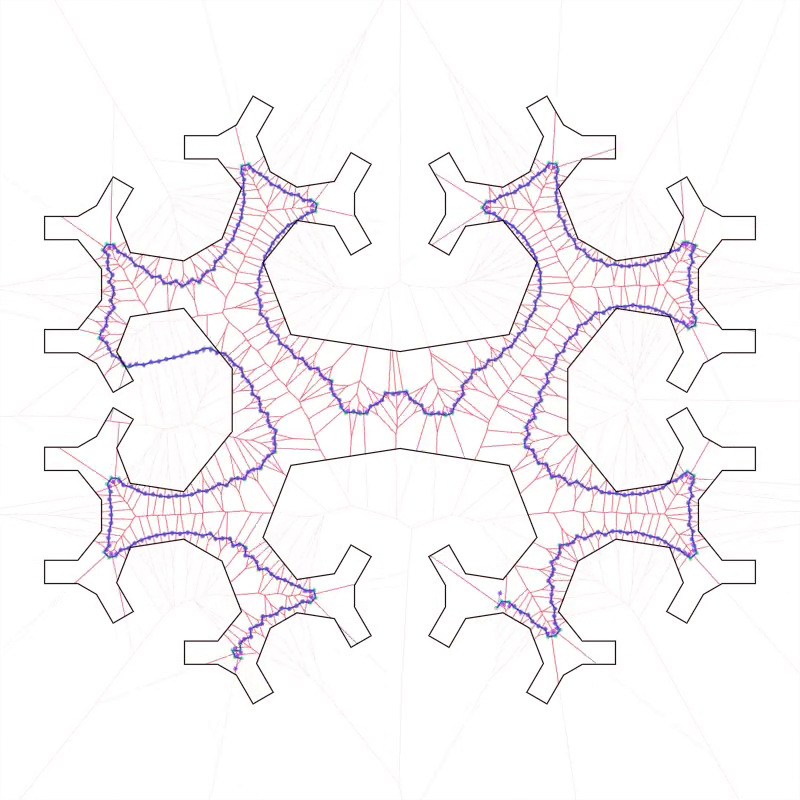}
    \caption{$i=800$}
  \end{subfigure}
  \vspace{-1.5em}
\end{figure}
\begin{figure}[H]
  \ContinuedFloat
  \begin{subfigure}[c]{.49\linewidth}
    \centering
    \includegraphics[scale=.15]{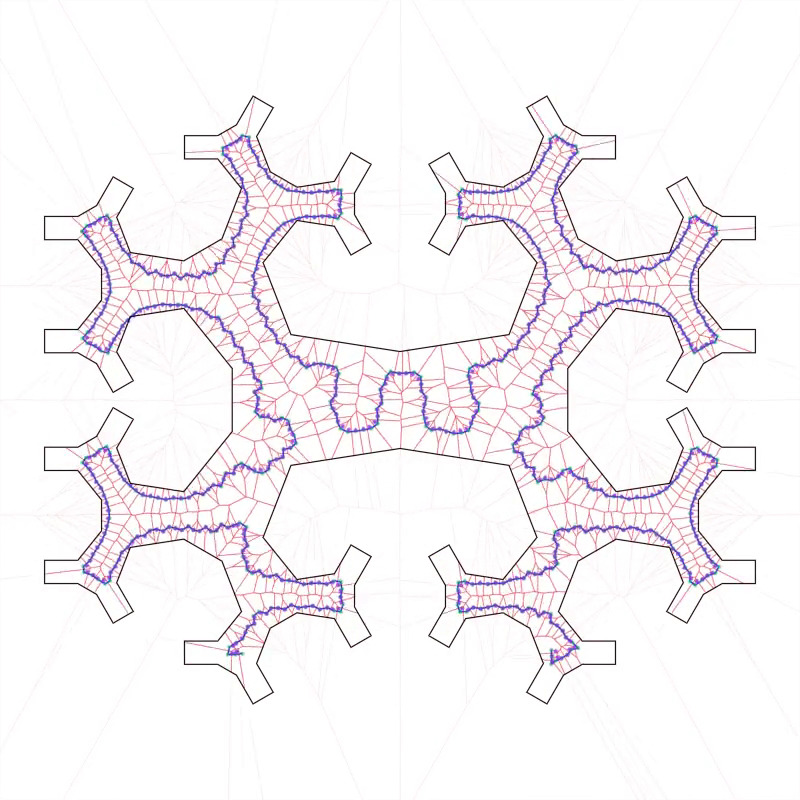}
    \caption{$i=1120$}
  \end{subfigure}
  \begin{subfigure}[c]{.49\linewidth}
    \centering
    \includegraphics[scale=.15]{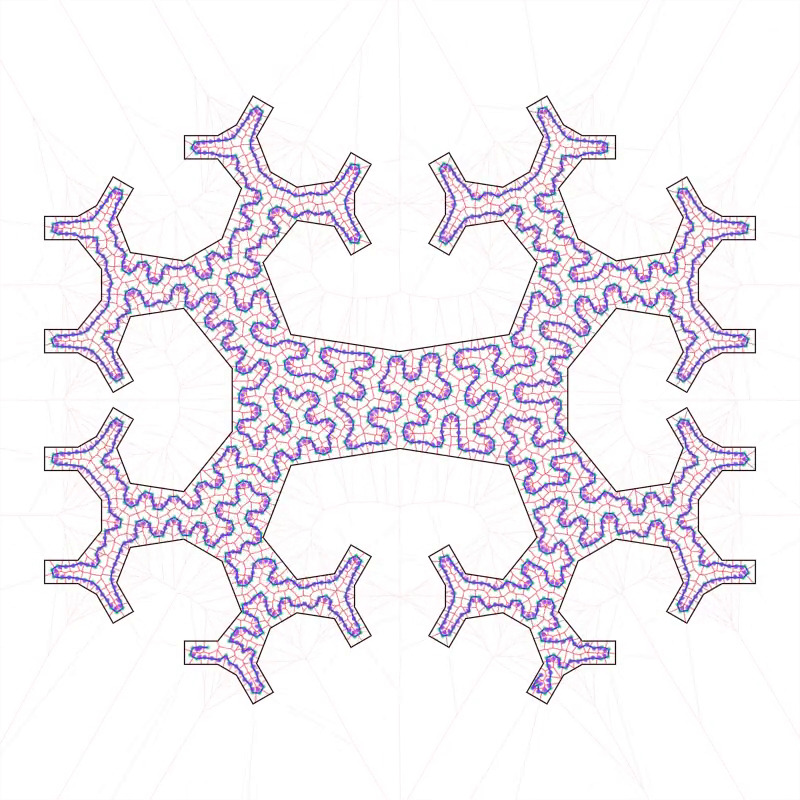}
    \caption{$i=1850$}
  \end{subfigure}
  \caption{Regularized curves.}
  \label{fig:regularized-curves}
\end{figure}
By contrast, the unregularized curve (\cref{fig:unregularized-curves})
initialized with the same parameters but $\lambda(i) \equiv 0$
improves the objective much more rapidly (though in practice we found
the gap can be substantially narrowed by decreasing the coefficient on
$\lambda(i)$, at the cost of introducing two or three persistent arcs
outside of $\supp(\cmeas)$). However, the resulting shape suffers from
many self-intersections and arcs escaping $\supp(\cmeas)$, as well as
a higher point density near the center than near the edges.
Additionally, we see the formation of a very dense spiral shape on the
right edge. Thus this shape would require a very complicated
reparametrization $\varphi$ to yield a good approximation
$\sigma_\theta$.

If the reparametrization were taken to be a neural network, the
network would typically require very extreme weights to effectively
remove density from the arcs outside of $\supp(\cmeas)$. This would be
a poor fit for learning with gradient descent as very large weights
tend to lead to slow convergence and unstable training.

\begin{figure}[H]
  \centering
  \begin{subfigure}[c]{.49\linewidth}
    \centering
    \includegraphics[scale=.15]{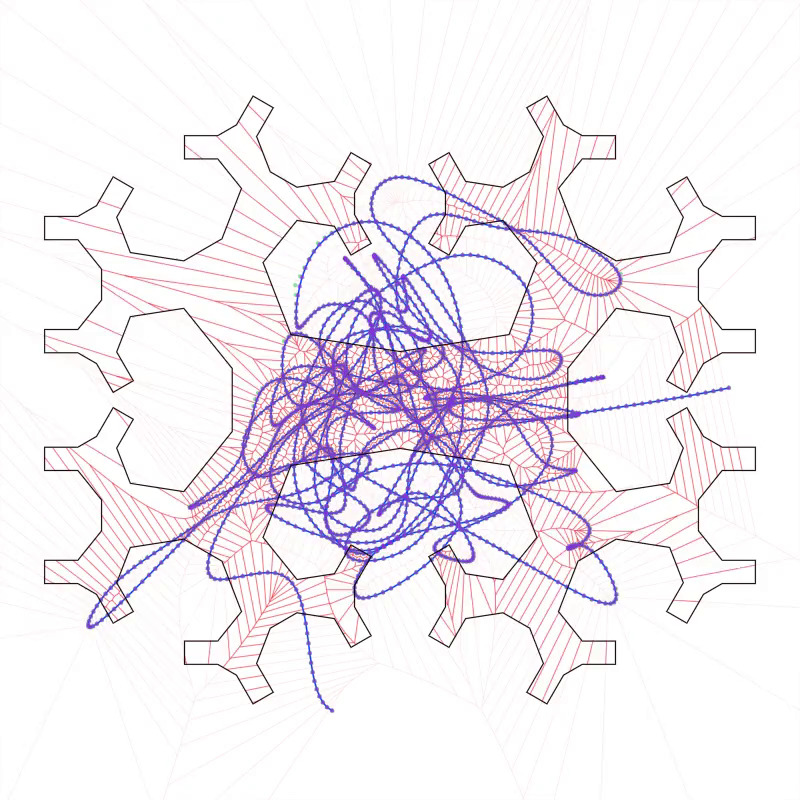}
    \caption{$i=1$}
  \end{subfigure}
  \begin{subfigure}[c]{.49\linewidth}
    \centering
    \includegraphics[scale=.15]{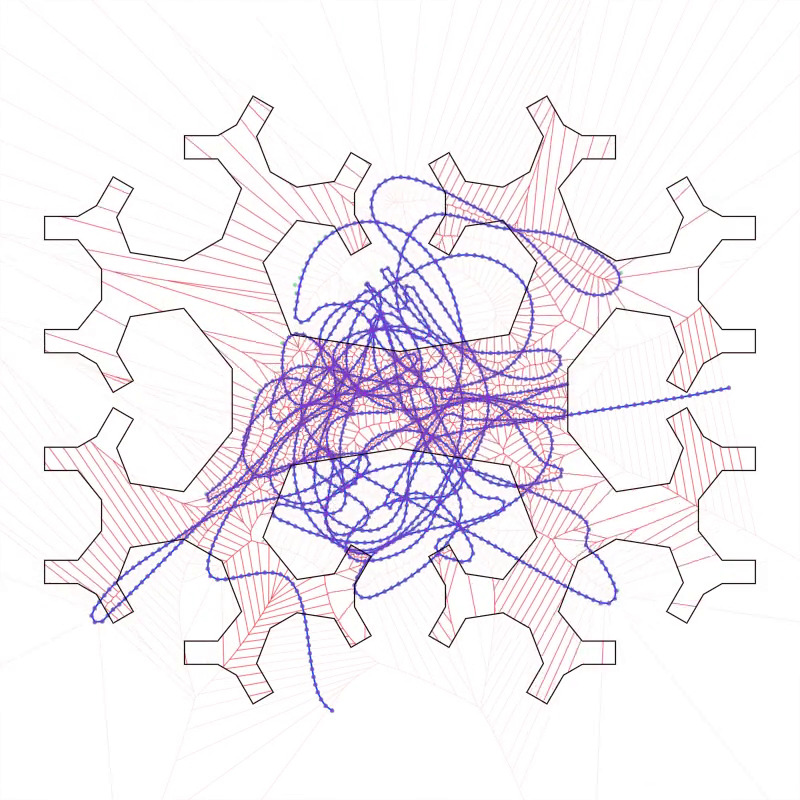}
    \caption{$i=20$}
  \end{subfigure}
  \vspace{-1.5em}
\end{figure}
\begin{figure}[H]
  \ContinuedFloat
  \begin{subfigure}[c]{.49\linewidth}
    \centering
    \includegraphics[scale=.15]{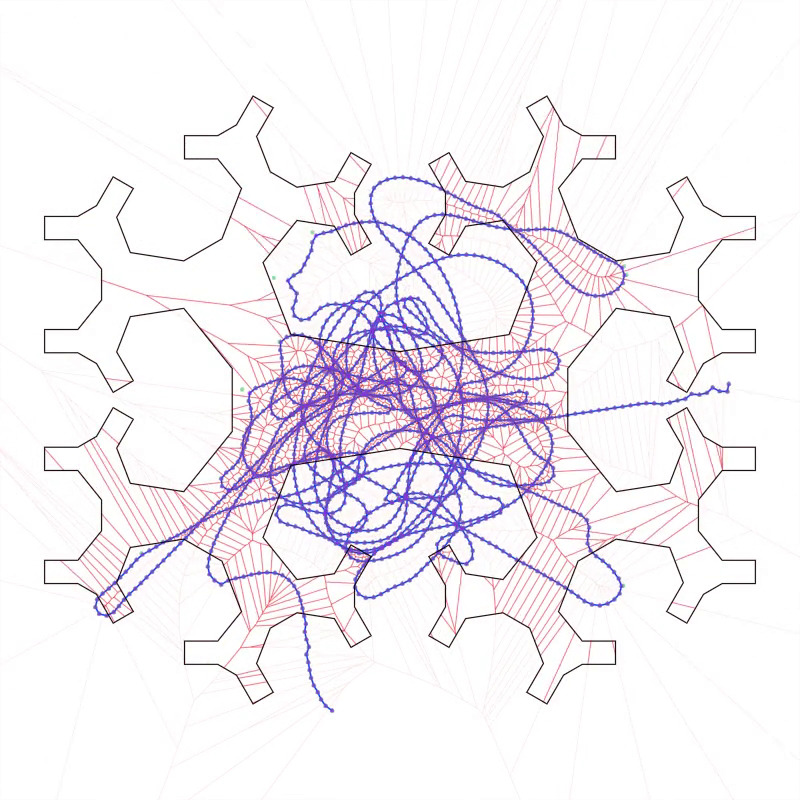}
    \caption{$i=110$}
  \end{subfigure}
  \begin{subfigure}[c]{.49\linewidth}
    \centering
    \includegraphics[scale=.15]{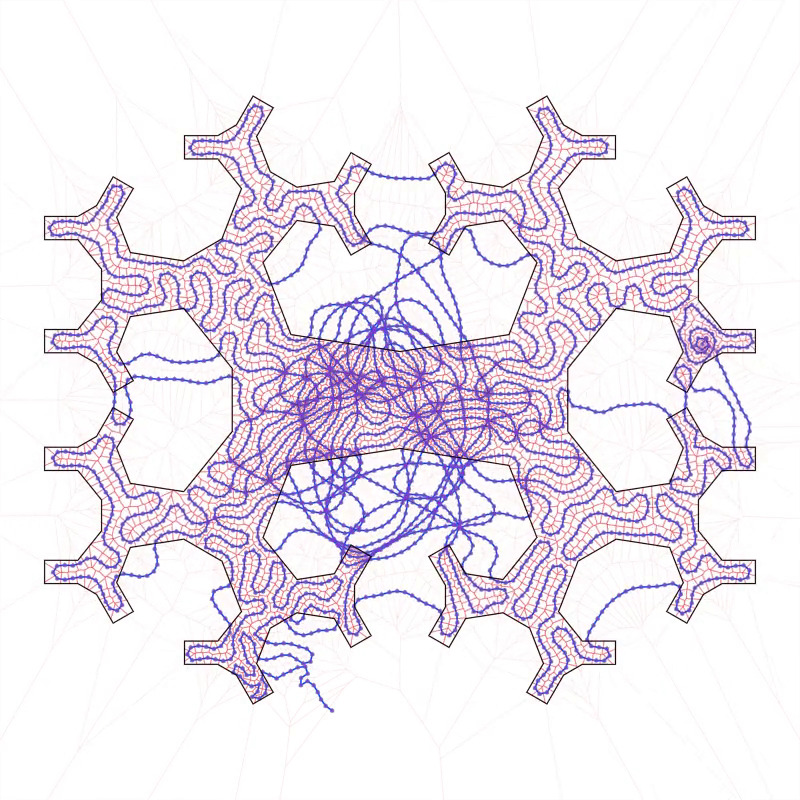}
    \caption{$i=800$}
  \end{subfigure}
  \vspace{-1.5em}
\end{figure}
\begin{figure}[H]
  \ContinuedFloat
  \begin{subfigure}[c]{.49\linewidth}
    \centering
    \includegraphics[scale=.15]{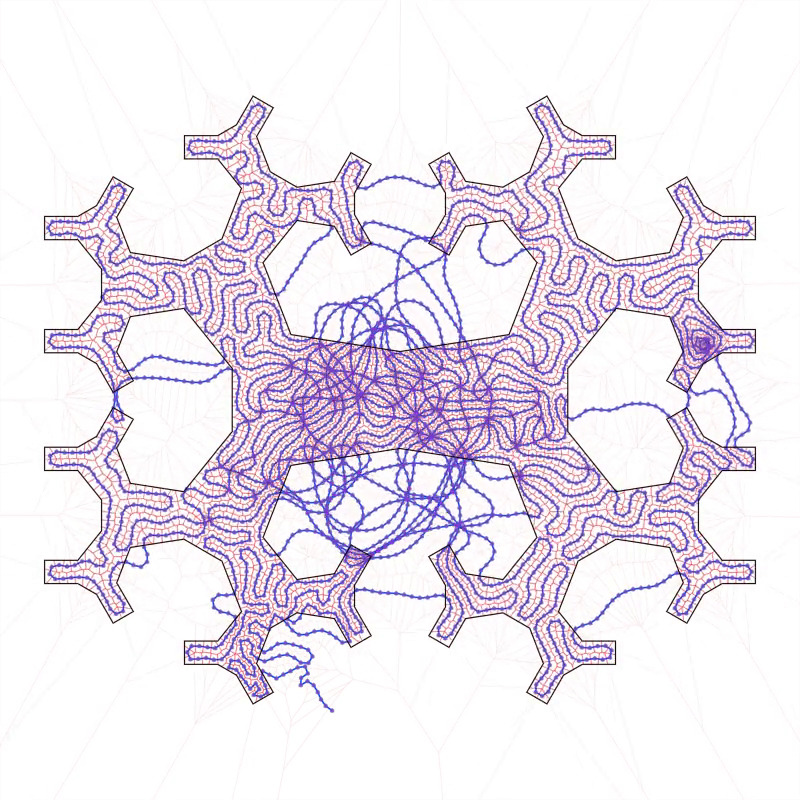}
    \caption{$i=1120$}
  \end{subfigure}
  \begin{subfigure}[c]{.49\linewidth}
    \centering
    \includegraphics[scale=.15]{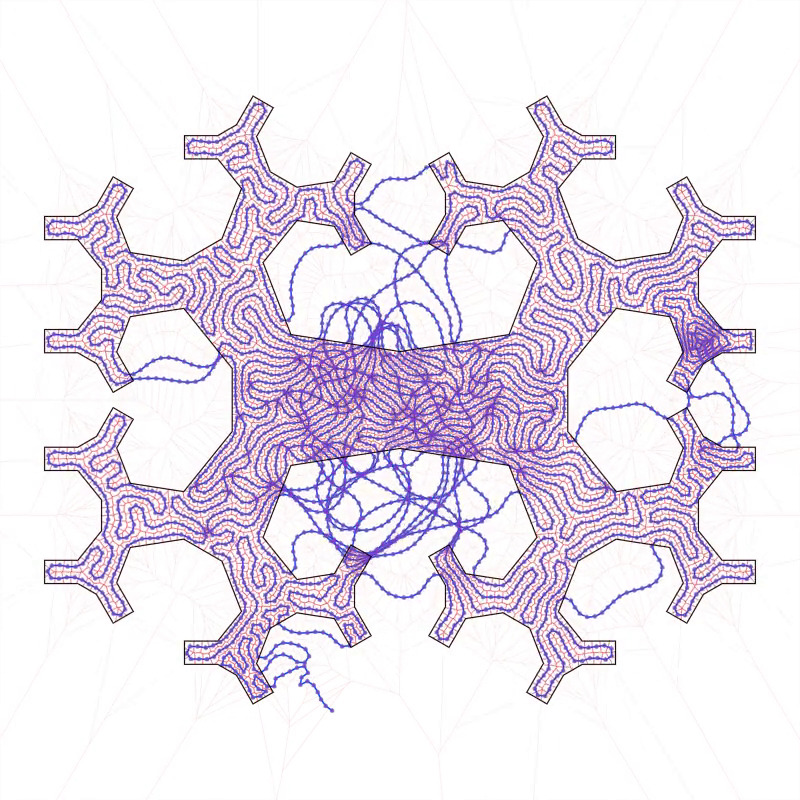}
    \caption{$i=1850$}
  \end{subfigure}
  \caption{Unregularized curves.}
  \label{fig:unregularized-curves}
\end{figure}
We propose that a similar mechanism could operate in general
generative learning tasks, even in the case where the problem is not
``factored'' into the two steps we described earlier.

To that end in \cref{sec:continuous-opt} below we offer two
proof-of-concept experiments demonstrating the utility of a $\sobk =
1$ penalty in training neural nets to produce the uniform measure on a
disk in $\RR^2$, as well as in generating images of handwritten
digits. Then, to highlight the distinction between the
``reparametrization'' benefits we described in this subsection and the
``overfitting avoidance'' benefits we described in
\cref{sec:classical-regularization-benefits}, we perform an MNIST
experiment in \cref{sec:discrete-optimization} which demonstrates that
even when $\lambda$ is too small to prevent overfitting, the
regularizing effects of $\cost$ encourage $\sigma_\theta$ to
``interpolate'' $\cmeas_{\msf M}$ more intelligently, thus possibly
producing a $\sigma_\theta$ that better approximates $\cmeas$ than an
unregularized solution might.

\subsection{Training Neural Networks with a $\sobk=1$
  Penalty}\label{sec:continuous-opt}
In this subsection we exhibit two simple proof-of-concept experiments
demonstrating that adding a $\lambda \cost$ term to $\objfp$ can yield
training improvements in some generative learning tasks. Note that as
it is generally computationally intensive to compute higher-order
derivatives for neural networks, we will focus our attention
particularly on the case $\sobk = 1$. The setup is the same in both
cases (up to a choice of parameters), hence we begin by describing it
in generality.

\subsubsection{Setup}
We consider the ``real-world'' case of \cref{sec:real-vs-theoretical}
where we do not know \(\cmeas\) and instead receive \(\msf M\) samples
\(\nndataset = \{\cpt_{\msf 1}, \ldots, \cpt_{\msf M}\}\) from
\(\cmeas\). In each of the experiments below we consider $\ddim=1$ and
model \(\fmap: \RR \to \RR^\cdim\) using an \(\nnnumlayers\) layer
fully connected neural network \(\nnffnet_\nnnumlayers:
\RR^\nnlatentdim{} \to \RR^\cdim\) of hidden dimension \(d\) with
\[
  \nnffnet_0(\nnlatent) = \nnweight_0\nnlatent,\qquad
  \nnffnet_\nnlayerindex(\nnlatent) =
  \nnnonlin(\nnweight_\nnlayerindex
  \nnffnet_{\nnlayerindex-1}(\nnlatent) + \nnbias_\nnlayerindex)
\]
where \(A_0 \in \RR^{\cdim \times \nnlatentdim}\),
\(\nnweight_\nnlayerindex \in \RR^{\nnlatentdim \times \nnlatentdim},
\nnbias_\nnlayerindex \in \RR^\nnlatentdim\) for all
\(\nnlayerindex\), and \(\nnnonlin\) notates the GeLU nonlinearity
\cite{hendrycks2016gaussian}. It is well-known that such networks
perform poorly when modeling functions with behavior at varying scales
in \(\cset\) \cite{tancik2020fourfeat,mildenhall2021nerf}, so we adopt
a modification of the positional encoding \(\nnposencoder : \RR \to
\RR^\nnlatentdim\) of \cite{mildenhall2021nerf},
\[
  \nnposencoder(x) = \pn{\frac{\sin(2^1\pi x)}{\sqrt{2^1}},
    \frac{\cos(2^1\pi x)}{\sqrt{2^1}},\ \ldots,\ \frac{\sin(2^{d/2}\pi
      x)}{\sqrt{2^{d/2}}}, \frac{\cos(2^{d/2}\pi x)}{\sqrt{2^{d/2}}}}.
\]
Compared to \cite{mildenhall2021nerf}, we damp each coordinate by the
square root of its frequency, otherwise we found the initial arc
length would tend to explode as \(d\) increased, leading to poor
numerical stability at initialization. Finally, we have \(\fmap = h_L
\circ \gamma\).

We approximate the \xref{prob:soft-penalty} objective \(\objfpl =
\objfp + \lambda \cost\) via an empirical loss $\nnloss$ defined as
follows. First, to approximate $\objfp$, for each evaluation we sample
a set of \(\nnbsize\) points \(\msf{\dset_N} = \set{\nncand_0, \ldots,
  \nncand_\nnbsize} \sim \dmeas\) and define
\[
  \nnlossobj = \frac{1}{\msf M}\sum_{i=1}^{\msf M} \min_{\nncand \in
    \msf{\dset_N}} \norm[]{\cpt_i - \fmap(\nncand)}.
\]
Note that because this is a $\objp = 1$ objective, properly speaking,
in order to interpret the gradient of $\nnlossobj$ in terms of the
barycenter field (\cref{prop:fbary-grad-J-discrete}), we need to know
$\{\omega_i\} \cap \fmap(\msf{X_N}) = \varnothing$. Typically this
will not be an issue: Since $\{\omega_i\}$ is countable, as long as
$\fmap$ is nonconstant and $\mu$ is absolutely continuous,
$\fmap^{-1}(\{\omega_i\})$ should be $\mu$-null, whence $\{\omega_i\}
\cap \fmap(\msf{X_N}) = \varnothing$ $\mu$-a.s.

For the $\cost$ part, since we will be primarily interested in
large-budget solutions we ignore the $0$\textsuperscript{th}-order
term and use the approximation
\[
  \nnlosscost = \frac{1}{\nnbsize}\sum_{i=1}^{\nnbsize}
  \norm{\fmap'(\nncand_{\msf i})}.
\]
Note that $\nnlosscost$ is implicitly a $\sobp = 1$ penalty, whereas
in our earlier theory work we considered only the case $1 < \sobp <
\infty$. In practice we do not expect this to yield different
qualitative behavior than the $\sobp = 1 + \varepsilon \approx 1$
case; hence we chose $\sobp =1$ simply as a matter of computational
convenience.

In any case, the final loss is then \(\nnloss = \nnlossobj + \lambda
\nnlosscost\) for some regularization parameter \(\lambda > 0\), just
as in \xref{prob:soft-penalty}.

\subsubsection{Experiment 1: Fitting Toy Data}
\label{sec:fitting-the-disk}
Here we consider the case $\ddim=1$, $\cdim=2$ and take \(\cmeas\) to
be the uniform measure on a set of 50 points randomly sampled from the
unit circle in \(\RR^2\). We perform 300,000 iterations of gradient
descent to optimize \(\fmap\), choosing \(\nnlatentdim = 10^3\) and
\(\nnbsize = 10^4\). We show results for \(\lambda = 10^{-4}\) as well
as \(\lambda = 0\) in \cref{fig:curvegan50}. These results highlight
the regularizing effect of the constraint. In
\cref{fig:curvegan50-noreg} we see that the curve passes through all
50 points almost perfectly, but interpolates them in a messy fashion.
In contrast, the curve in \cref{fig:curvegan50-reg} interpolates
almost linearly between successive points in a manner reminiscent of a
traveling salesperson path, but there remain three points quite far
from the curve.
\begin{figure}[H]
\centering
\begin{subfigure}{.49\linewidth}
  \centering
  \includegraphics[scale=.44]{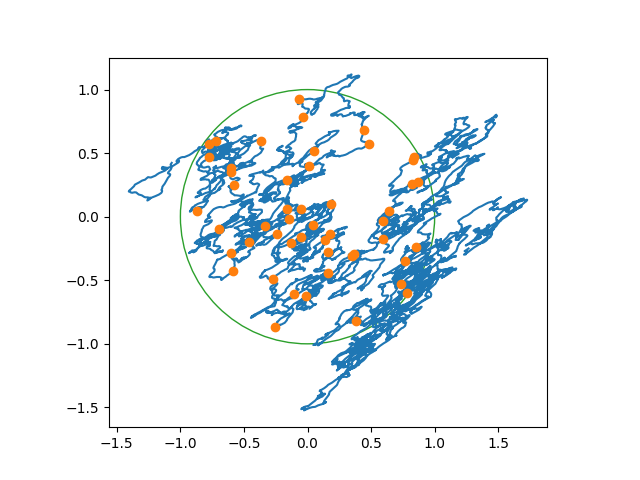}
  \caption{\(\lambda = 0\)}\label{fig:curvegan50-noreg}
\end{subfigure}
\begin{subfigure}{.49\linewidth}
  \centering
  \includegraphics[scale=.44]{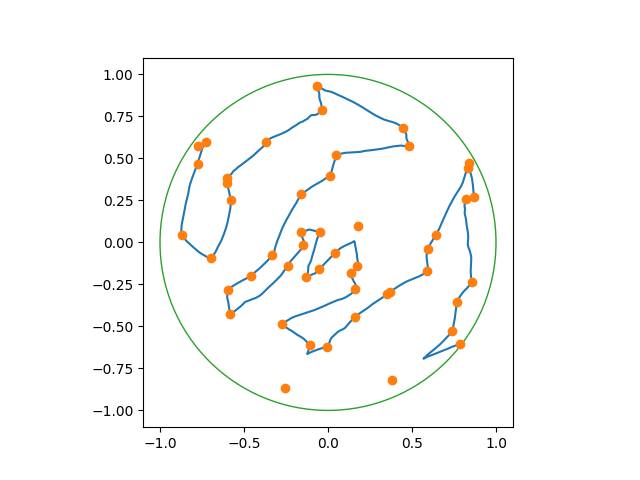}
  \caption{\(\lambda = 10^{-4}\)}\label{fig:curvegan50-reg}
\end{subfigure}
\caption{Visualization of \(\fmap\) for different values of
  \(\lambda\).}\label{fig:curvegan50}
\end{figure}

This failure to fit all of the points is due to the Lagrangification
of the constraint. As \(\lambda \to 0\) we would expect this effect to
vanish while maintaining the desired regularizing effect. Of course,
as \(\lambda \to 0\), our optimization problem becomes increasingly
ill-conditioned, which prevents gradient descent from quickly finding
effective minimizers.

\subsubsection{Fitting MNIST}
\label{sec:fitting-mnist}
Here we take \(\dset = [0, 1]\) with \(\dmeas\) the uniform measure on
\(\dset\) and \(\nndataset\) given by MNIST.

For our experiments, we chose a small value of \(\lambda = 10^{-4}\).
We optimized the weights \((\nnweight_\nnlayerindex,
\nnbias_\nnlayerindex)\) to minimize \(\nnloss\) in expectation using
stochastic gradient descent, performing 1000 epochs over the dataset.
We also chose
\(\nnlatentdim = 10^3\) and \(\nnbsize = 10^4\).

We plot the evolution of \(\fmap\) over the course of training in
\cref{fig:curvegan-mnist}, where we can see that the visual quality of
the digits produced improves as training progresses.
\begin{figure}[H]
\centering
\begin{subfigure}{.49\linewidth}
  \centering
  \includegraphics[scale=.47]{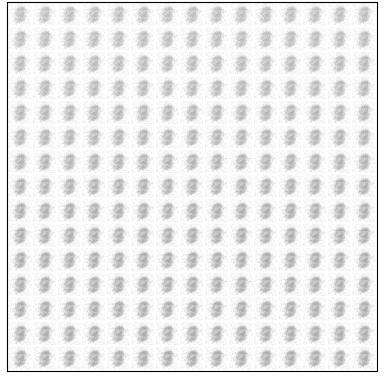}
  \caption{1 epoch}\label{fig:curvegan-mnist-1}
\end{subfigure}
\begin{subfigure}{.49\linewidth}
  \centering
  \includegraphics[scale=.47]{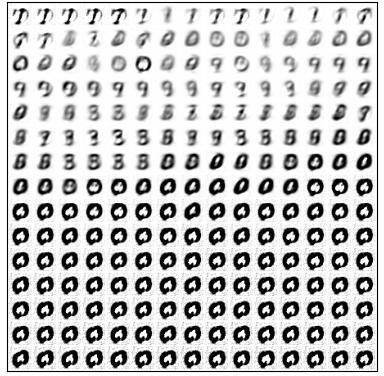}
  \caption{30 epochs}\label{fig:curvegan-mnist-30}
\end{subfigure}
\end{figure}
\begin{figure}[H]\continuedfloat
\begin{subfigure}{.49\linewidth}
  \centering
  \includegraphics[scale=.47]{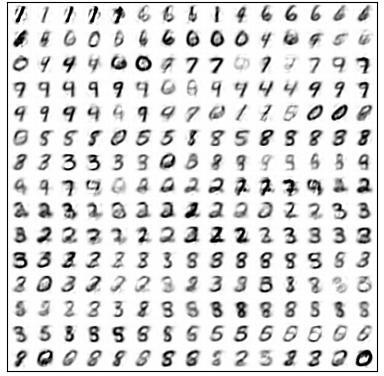}
  \caption{100 epochs}\label{fig:curvegan-mnist-100}
\end{subfigure}
\begin{subfigure}{.49\linewidth}
  \centering
  \includegraphics[scale=.47]{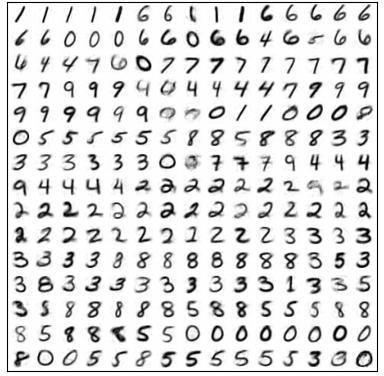}
  \caption{1000 epochs}\label{fig:curvegan-mnist-1000}
\end{subfigure}
\caption{Visualization of \(\fmap\) sampled at \(15^2\)
  uniformly-spaced points (for convenience) from \(0\) to \(1\) at
  several points during the training process.}\label{fig:curvegan-mnist}
\end{figure}

To demonstrate the importance of the constraint, we perform a similar
run setting \(\lambda = 0\). The results are shown in
\cref{fig:curvegan-mnist-noreg}.
\begin{figure}[H]
\centering
\begin{subfigure}{.49\linewidth}
  \centering
  \includegraphics[scale=.47]{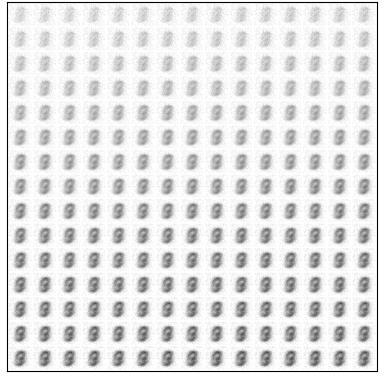}
  \caption{1 epoch}\label{fig:curvegan-mnist-noreg-1}
\end{subfigure}
\begin{subfigure}{.49\linewidth}
  \centering
  \includegraphics[scale=.47]{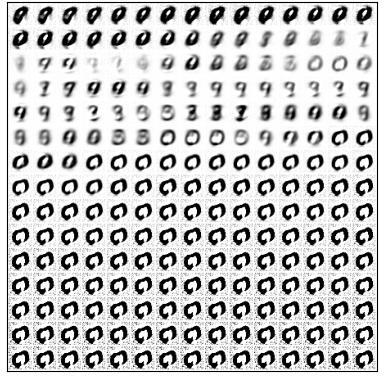}
  \caption{30 epochs}\label{fig:curvegan-mnist-noreg-30}
\end{subfigure}
\end{figure}
\begin{figure}[H] \continuedfloat
\begin{subfigure}{.49\linewidth}
  \centering
  \includegraphics[scale=.47]{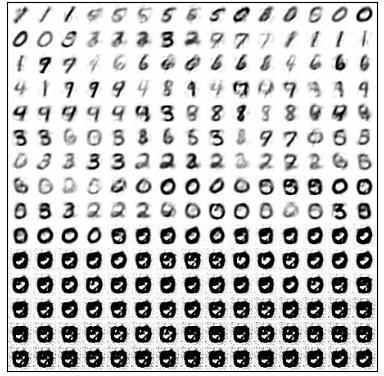}
  \caption{100 epochs}\label{fig:curvegan-mnist-noreg-100}
\end{subfigure}
\begin{subfigure}{.49\linewidth}
  \centering
  \includegraphics[scale=.47]{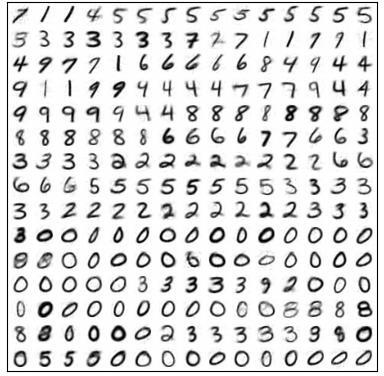}
  \caption{1000 epochs}\label{fig:curvegan-mnist-noreg-1000}
\end{subfigure}
\caption{Visualization of \(\fmap\) as in \cref{fig:curvegan-mnist}
  but with \(\lambda = 0\).}\label{fig:curvegan-mnist-noreg}
\end{figure}
Here we can see that after 100 epochs, the result without
regularization in \cref{fig:curvegan-mnist-noreg-100} is far worse
than the regularized result in \cref{fig:curvegan-mnist-100}. In this
case in \cref{fig:curvegan-mnist-noreg-100}, a large portion of
\(\dset\) is ``wasted'' because large intervals of the domain map to
nearly the same image, which leads to poor coverage of the rest of
MNIST.

Although the arclength penalty does not directly penalize this
behavior (since regions of \(\dset\) where $\fmap$ is approximately
constant do not contribute to the arclength), it is clearly beneficial
in this situation. Note that after sufficient iterations, even the
unregularized network begins to make full use of \(\dset\), as seen in
\cref{fig:curvegan-mnist-noreg-1000}. We hypothesize the following
explanation for the behavior in these two cases.

Recall from our earlier discussion of $\nnlossobj$ that we may
interpret the gradient of $\nnlossobj$ (with respect to the locations
of the sample points $\fmap(\msf{x})$) in terms of the barycenter
field. So, at each iteration, our weight updates should be modifying
$\fmap$ in a way that approximates following the barycenter field.

Suppose there is a region $\dset_{\rm static}$ where $\fmap$ is
approximately constant. Unless the $\{\omega_i\}$ are highly
concentrated around $\fmap(\dset_{\rm static})$, this typically
implies the barycenter will approximately vanish everywhere except a
few points near the boundary of $\fmap(\dset_{\rm static})$ (since
only these outer points will get data samples projecting to them).

Consider such an ``outer'' point $\msf x \in \dset_{\rm static}$. In
the update step, $\msf{x}$ will be pulled towards more useful regions
of \(\cset\) by the barycenter field. Without the arclength penalty,
the curve will simply lengthen to accommodate this, leaving
$\dset_{\rm static}$ largely intact. In contrast, the arclength
penalty causes the shifted point $\msf x$ to ``pull'' some of
$\dset_{\rm static}$ with it. Thus, the regularization will promote
effective usage of the entirety of \(\dset\) much more quickly than in
the unregularized case.

\subsubsection{Comparison with weight decay}
\label{sec:weight-decay-comparison}

One popular form of regularization for neural networks is \(L_2\)
regularization of the network's weights which is sometimes called
``weight decay'' in the context of gradient descent
\cite{krogh1991simple,zhang2018three,gnecco2009weight}. Under gradient
descent with weight decay, we multiply the weights of the network by
\(1 - \nnwdparam\) after each gradient step, where \(\nnwdparam > 0\)
is typically quite small. A notable connection between our constraint
and weight decay was explored in \cite{zhang2018three}; there the
authors propose that weight decay may behave similarly to a penalty on
the Frobenius norm of the network's Jacobian, which has a similar form
to our constraint when \(\sobk = 1\) and \(\sobp = 2\). We show
results when running under weight decay in
\cref{fig:curvegan-mnist-wd}. We chose \(\nnwdparam = 10^{-4}\) so
that the amount of regularization would be similar to that in
\cref{sec:fitting-mnist}. Comparing \cref{fig:curvegan-mnist-100} and
\cref{fig:curvegan-mnist-wd-100} we see that our constraint produces
better results than weight decay after 100 epochs although after 1000
epochs the gap has narrowed considerably.
\begin{figure}[H]
\centering
\begin{subfigure}{.49\linewidth}
  \centering
  \includegraphics[scale=.47]{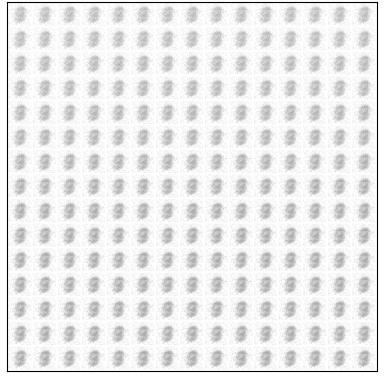}
  \caption{1 epoch}\label{fig:curvegan-mnist-wd-1}
\end{subfigure}
\begin{subfigure}{.49\linewidth}
  \centering
  \includegraphics[scale=.47]{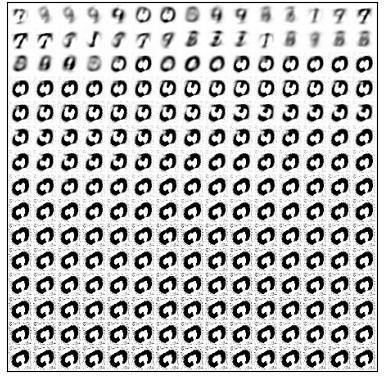}
  \caption{30 epochs}\label{fig:curvegan-mnist-wd-30}
\end{subfigure}
\end{figure}
\begin{figure}[H] \continuedfloat
\begin{subfigure}{.49\linewidth}
  \centering
  \includegraphics[scale=.47]{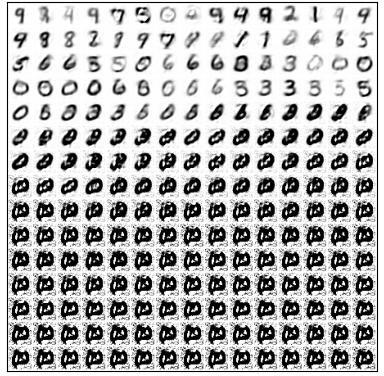}
  \caption{100 epochs}\label{fig:curvegan-mnist-wd-100}
\end{subfigure}
\begin{subfigure}{.49\linewidth}
  \centering
  \includegraphics[scale=.47]{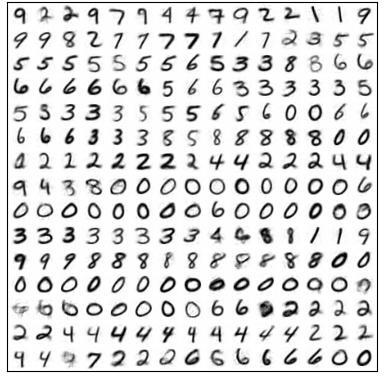}
  \caption{1000 epochs}\label{fig:curvegan-mnist-wd-1000}
\end{subfigure}
\caption{Visualization of \(\fmap\) as in \cref{fig:curvegan-mnist}
  but using weight decay with \(\nnwdparam =
  10^{-4}\).}\label{fig:curvegan-mnist-wd}
\end{figure}
Unlike our constraint, the effect of weight decay varies with the
architecture and parametrization of the neural network being
optimized. This may be seen as a weakness or a strength depending on
the details of the task being performed. In this particular setting,
weight decay does not appear to offer much advantage over not using
regularization at all.

\subsection{Demonstrating the Interpolation Benefits of a $\sobk = 1$
  Penalty}\label{sec:discrete-optimization}

Recall that in \cref{sec:applying-to-ml} we hypothesized that, in
addition to avoidance of overfitting, the penalty term $\lambda \cost$
could improve training efficiency by encouraging regularity in the
optimal reparametrization map $\varphi$. To illustrate this point, we
initialized our ``theoretical'' algorithm from \cref{sec:algorithm}
with and without regularization, yielding the outputs in
\cref{fig:regularized-curves} and \cref{fig:unregularized-curves}.

In this final subsection, we perform an analogous ``real-world''
experiment using MNIST data. So as to emphasize the distinction with
the ``avoidance of overfitting'' effects described in
\cref{sec:classical-regularization-benefits}, we will compare two
curves that both \emph{fully interpolate} the data but differ in their
$\cost$ values. Typically such interpolation can occur only when
$\lambda = 0$ in \xref{prob:soft-penalty}; to avoid some technical
hand-wringing around this case we instead formalize our experiments
via an ``adjoint'' formulation of \xref{prob:hard-constraint}:
\begin{adjustwidth}{1em}{0em}
  \vspace{.25em}
  \begin{leftbar} \vspace{-.75em}
    \begin{problem}[``Adjoint'' Problem (AdjHC)]
      \xlabel[(AdjHC)]{prob:adjoint}
      Minimize $\cost(\fmap)$ over $\set{\objfp(\fmap) \leq \zeta}$.
    \end{problem} \vspace{-.25em}
  \end{leftbar}
  \vspace{-.375em}
\end{adjustwidth}
While \xref{prob:adjoint} is similar to \xref{prob:hard-constraint} in
the sense that we have a hard constraint instead of a soft penalty,
note that \xref{prob:adjoint} recovers the property that all
optimizers must be maximally-efficient, and is thus equivalent to
\xref{prob:soft-penalty} and not \xref{prob:hard-constraint} (see
\cref{sec:comparison-to-soft-penalty}).

Now, let $\ddim = 1$, $\sobk = 1$, and $\sobp = 1$ (as in
\cref{sec:continuous-opt} we do not expect $\sobp = 1$ to cause
problems), and suppose our target is a discrete measure $\cmeas_{\msf
  M} = \sum_{j=1}^{\msf M} \delta_{\cpt_j}$. We now show that in this
case the $\zeta = 0$ \xref{prob:adjoint} essentially recovers the
Euclidean traveling salesperson problem (TSP), whence we may
approximate solutions to \xref{prob:adjoint} using TSP solvers.

To that end, note that if $\fmap$ is a $\zeta=0$ optimizer of
\xref{prob:adjoint} then $\fmap$ must visit every point in \(S =
\supp(\cmeas_{\msf M})\). Due to our choice of \(\sobk = 1\) and
negligibility of the $0$\textsuperscript{th}-order term for large
$\budg$ (a reasonable assumption for large $\msf M$), the optimal way
to travel between two points of \(S\) is via a straight line.
Accordingly, we are interested in choosing an ordering on \(S\) such
that visiting the points of \(S\) in that order, connected by straight
lines, gives the shortest overall arc length. This is precisely the
Euclidean TSP.

Although the Euclidean TSP is known to be NP-hard, mature numerical
solvers exist which produce solutions of very high quality. We use the
LKH solver \cite{helsgaun2000effective}, an optimized implementation
of the Lin-Kernighan heuristic, to produce an approximate tour of the
MNIST training set. By drawing random samples uniformly from the
domain set $[0,1]$, we may visually compare random samples from this
optimized tour against those from a random tour:
\begin{figure}[H]
  \centering
  \def\figdir{figures}
  \def\locscale{.5}
  \begin{subfigure}{.32\linewidth}
    \includegraphics[scale=\locscale]{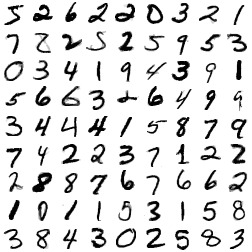}
    \caption{TSP}\label{fig:mnist-tsp}
  \end{subfigure}\hspace{4em}
  \begin{subfigure}{.32\linewidth}
    \includegraphics[scale=\locscale]{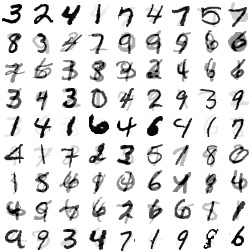}
    \caption{Random}\label{fig:mnist-random}
  \end{subfigure}
  \caption{Samples from tours of MNIST with \cref{fig:mnist-tsp} a
    tour optimized by LKH and \cref{fig:mnist-random} a random
    tour.}\label{fig:mnist}
\end{figure}
Perhaps unsurprisingly, the interpolated images from the optimized
tour in \cref{fig:mnist-tsp} appear more similar to real MNIST digits
than the ones from the random tour in \cref{fig:mnist-random} do.
Since the domain samples were chosen uniformly at random from $[0,1]$,
we see that in this case taking $\varphi$ to be identity yields good
performance in the second step of our ``factorization'' in
\cref{sec:applying-to-ml}.

In particular, note that this improvement is not due to ``prevention
of overfitting,'' as both tours fully interpolate the data. This
suggests that even when $\lambda$ is too small to prevent overfitting,
the inclusion of a regularization term may still yield training
improvements via the mechanisms in
\cref{sec:additional-regularization-benefits}.

While it is certainly possible to reparametrize the random curve to
produce reasonable-looking images (by placing point masses at the
original data points, for example) this model does not effectively
generalize the examples seen. In contrast, the TSP path is able to
produce novel images that are visually consistent with those of MNIST.

\section{Conclusion}
In this paper, we analyzed a certain variational problem regarding
approximating high-dimensional measure $\rho$ via a lower-dimensional
measure $\nu$ whose complexity is bounded. We showed that in fact it
is sufficient to treat the \emph{support} of $\nu$ (parametrized by
some $\fmap$) as the fundamental variable of optimization, yielding
the problem \xref{prob:hard-constraint}, as well as the
nearly-equivalent and nicer-to-simulate problem
\xref{prob:soft-penalty}. Importantly, the fact that we optimize over
$\fmap$ instead of $\nu$ in these problems allowed us to use simpler,
geometric techniques in our analysis.

We used this perspective to propose a two-step ``factorization'' of a
WGAN problem in which there is regularization on the generator. In the
first step, we obtain an $\fmap$ that is optimal in the sense of
\xref{prob:hard-constraint}/\xref{prob:soft-penalty}, and then in the
second step, $\fmap$ is reparametrized to properly push-forward a
given measure $\mu$. Since the essential qualitative behavior of the
solution is determined primarily in the first step, we proposed that
theoretical study of
\xref{prob:hard-constraint}/\xref{prob:soft-penalty} could provide
insights into the qualitative behavior of WGAN, and other similar
problems.

To that end, we examined basic properties of
\xref{prob:hard-constraint} over a general class of regularity
parameters for $\fmap$. We placed a particular focus on understanding
the ``gradient'' of the objective functional $\objfp(\fmap)$; we
called this gradient the \emph{barycenter field}. Under mild
hypotheses, we showed (in both the continuous and discrete cases) that
the barycenter field has a simple geometric interpretation that is
amenable to visualization. Then, in the discretized case, we proposed
an efficient, deterministic numerical scheme that used the barycenter
field to simulate \xref{prob:soft-penalty} for a special family of
targets $\rho$, provided $\fmap$ is a curve.

The resulting simulations had intuitive qualitative behavior that
offered a simple explanatory narrative for how the regularization on
$\fmap$ in the first step of the ``factored'' WGAN problem results in
a much more well-behaved reparametrization in the second step.
Importantly, we hypothesized such regularization could yield benefits
in both training and generalization performance, even when no explicit
factorization of WGAN is performed. In this vein, we gave simple,
proof-of-concept experiments demonstrating both effects manifesting in
an ``unfactored'' model when generating MNIST data in the
presence/absence of our regularization.

In future work it would be interesting to quantitatively examine the
extent to which similar effects persist in more nontrivial datasets.

\bibliographystyle{amsplain}
\bibliography{bib}

\appendix

\section{Gallery of Numerical Experiments}
\label{sec:gallery-of-sims}
\subsection{Triangular Domain}
The numerics below were obtained with the following choices of
parameters:
\begin{itemize}
  \item Domain: Triangle with vertices $\set{(1, 2\pi k/3) \MID k =
    0,1,2}$,
  \item Objective functional exponent: $\objp = 2$,
  \item Sobolev parameters: $\sobk = 2$, $\sobp = 2$
  \item Initial condition: 500 equally-spaced samples of
    $(t, \frac{1}{200} \sin(200t))$ on $[-1/20,1/20]$
  \item Spline resampling resolution: $\frac{3}{1000}$
  \item Learning rate: $\learningrate(i) = (i/500)^2$
  \item Regularization penalty: $\lambda(i) = \frac{1}{100} (1 - \learningrate(i))$
  \item Area rescaling: $\fbarysf(\iptsfj) \mapsto \fbarysf(\iptsfj) /
    \cmeas(\psin^{-1}(\iptsfj))^{17/20}$
  \item Smoothing window width: 1 point each side for the $\iptsfj$, 3
    points each side for $\fbarysfj$, 3 points each side for $\nabla
    \cost$.
\end{itemize}

\begin{figure}[H]
  \centering
  \begin{subfigure}{.49\linewidth}
    \centering
    \includegraphics[scale=.09]{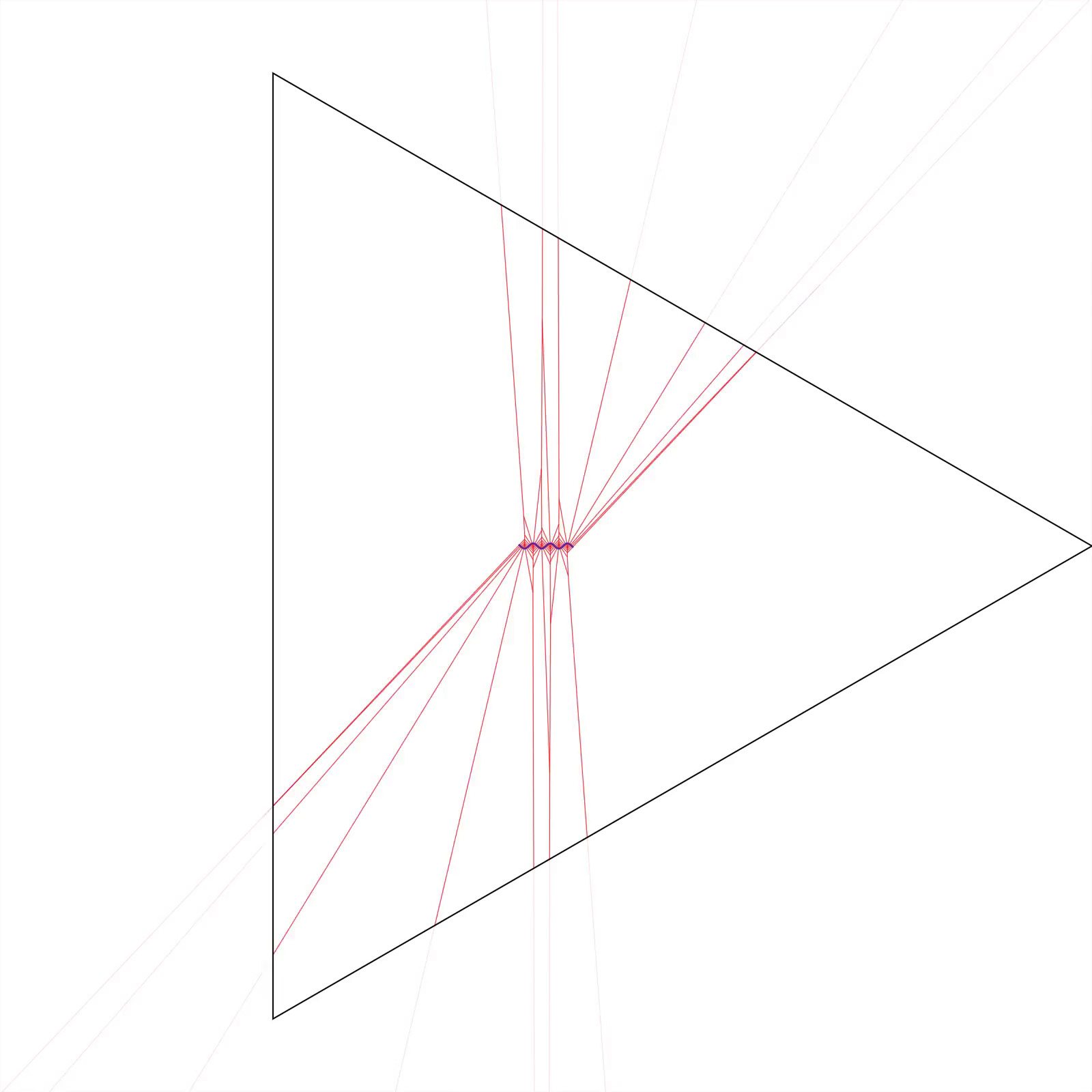}
    \caption{$i=1$}
  \end{subfigure}
  \begin{subfigure}{.49\linewidth}
    \centering
    \includegraphics[scale=.09]{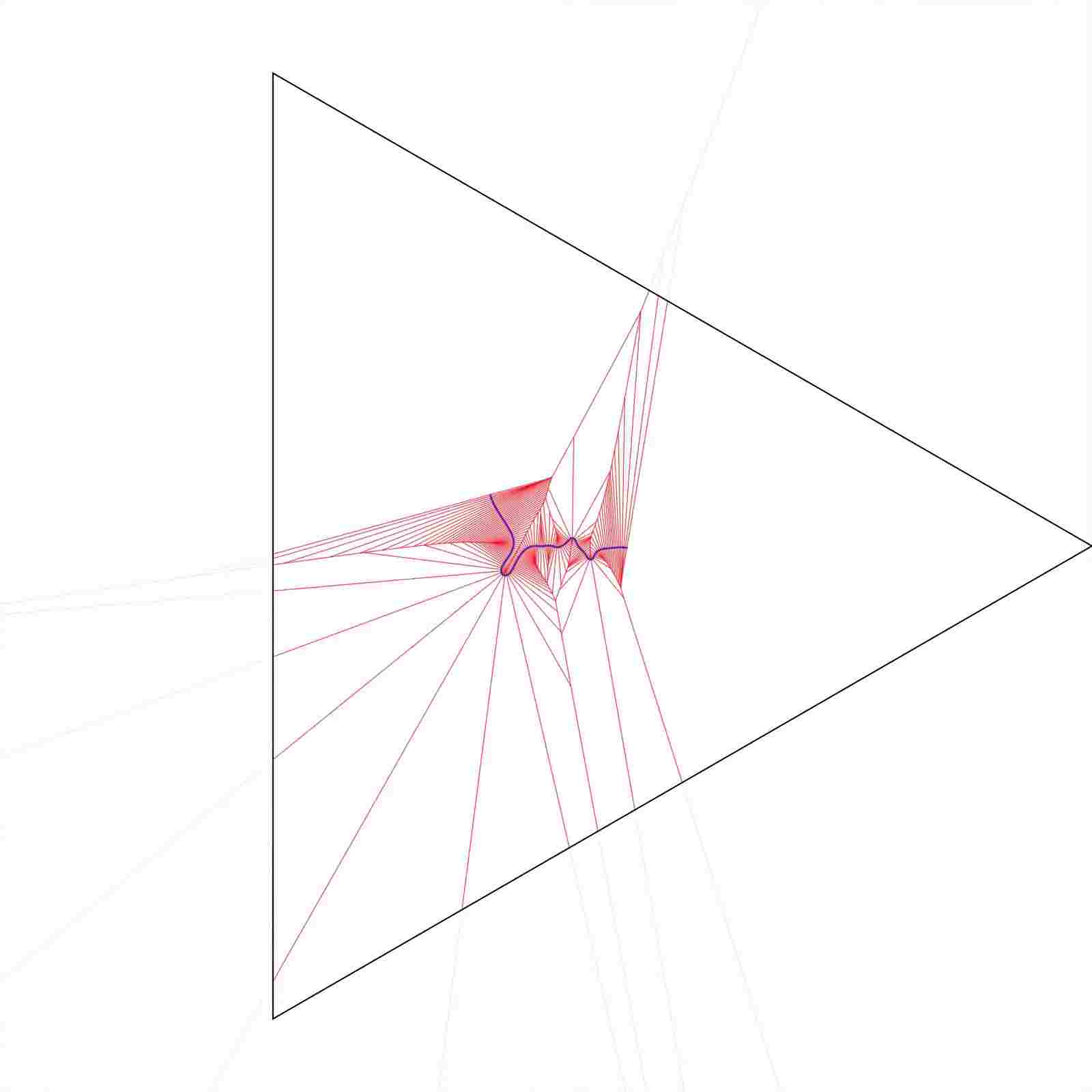}
    \caption{$i=50$}
  \end{subfigure}
\end{figure}
\begin{figure}[H] \ContinuedFloat
  \centering
  \begin{subfigure}{.49\linewidth}
    \centering
    \includegraphics[scale=.09]{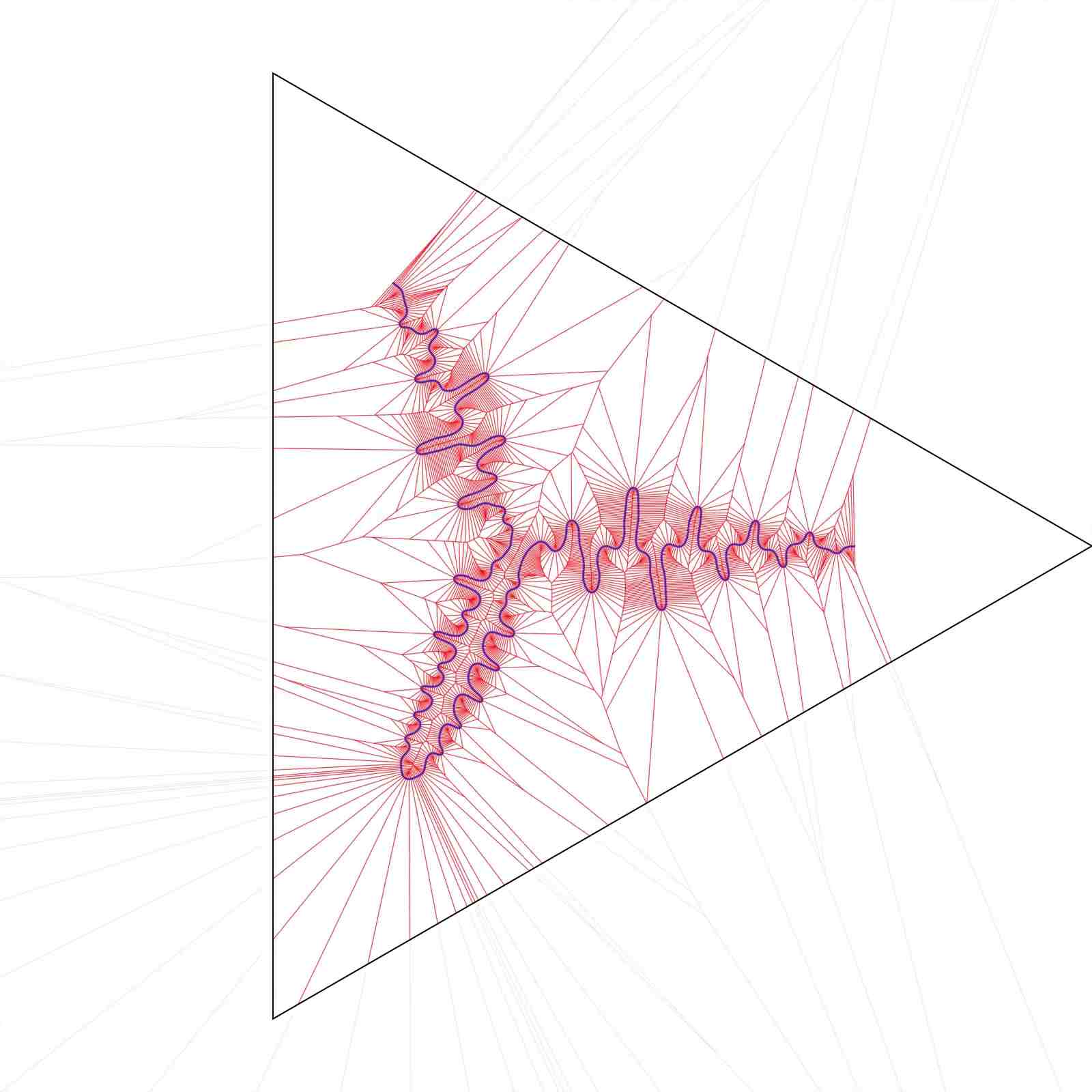}
    \caption{$i=100$}
  \end{subfigure}
  \begin{subfigure}{.49\linewidth}
    \centering
    \includegraphics[scale=.09]{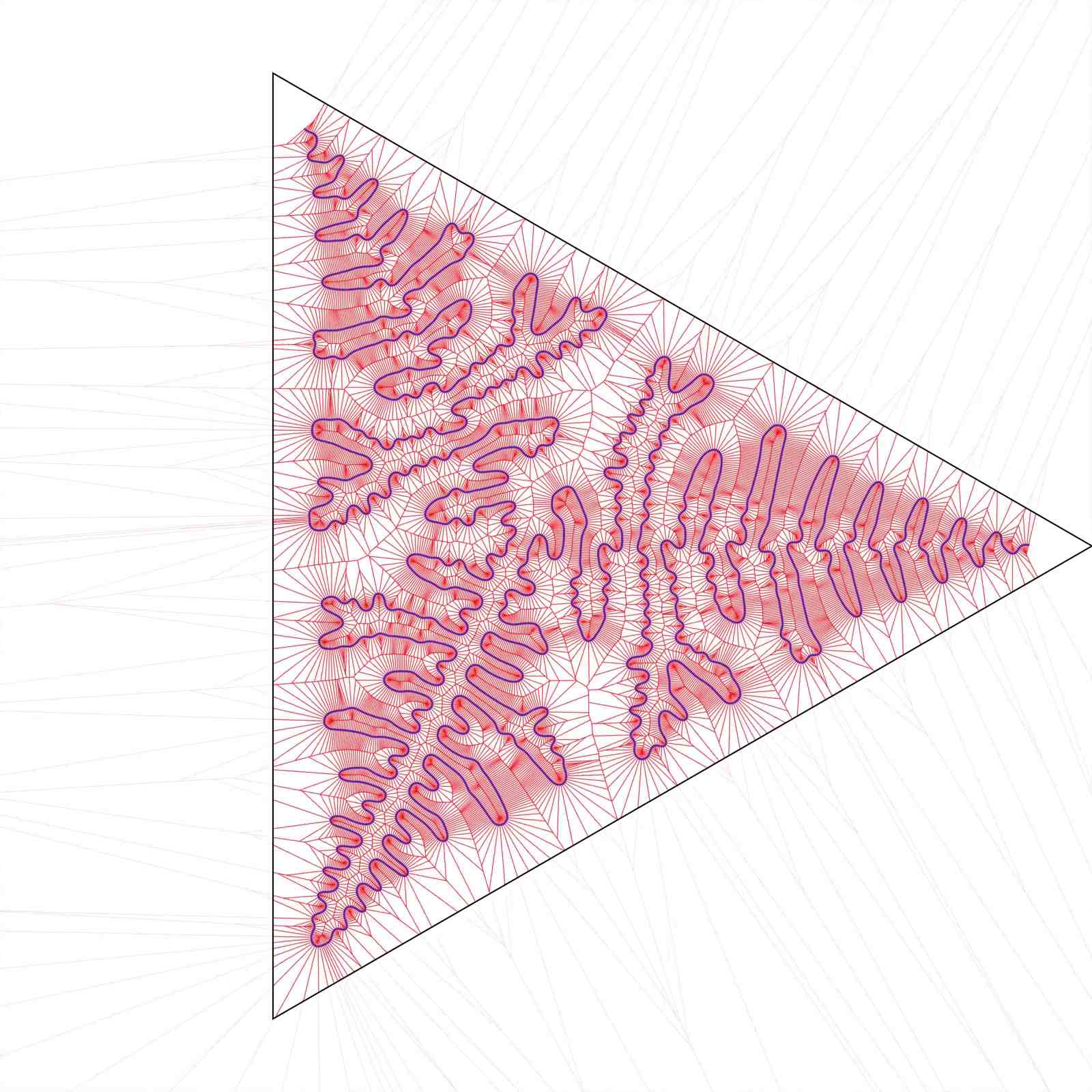}
    \caption{$i=150$}
  \end{subfigure}
\end{figure}
\begin{figure}[H] \ContinuedFloat
  \centering
  \begin{subfigure}{.49\linewidth}
    \centering
    \includegraphics[scale=.09]{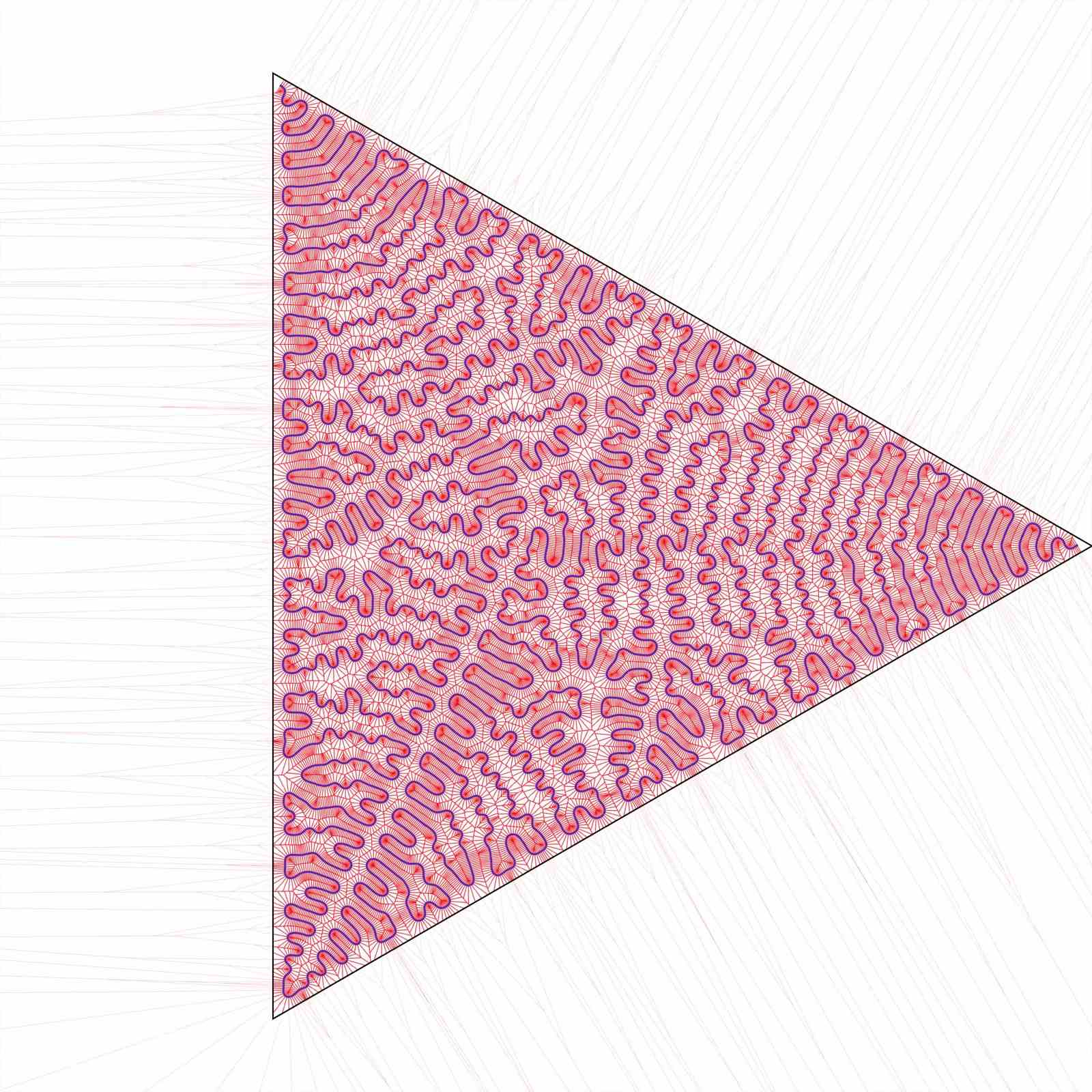}
    \caption{$i=200$}
  \end{subfigure}
  \begin{subfigure}{.49\linewidth}
    \centering
    \includegraphics[scale=.09]{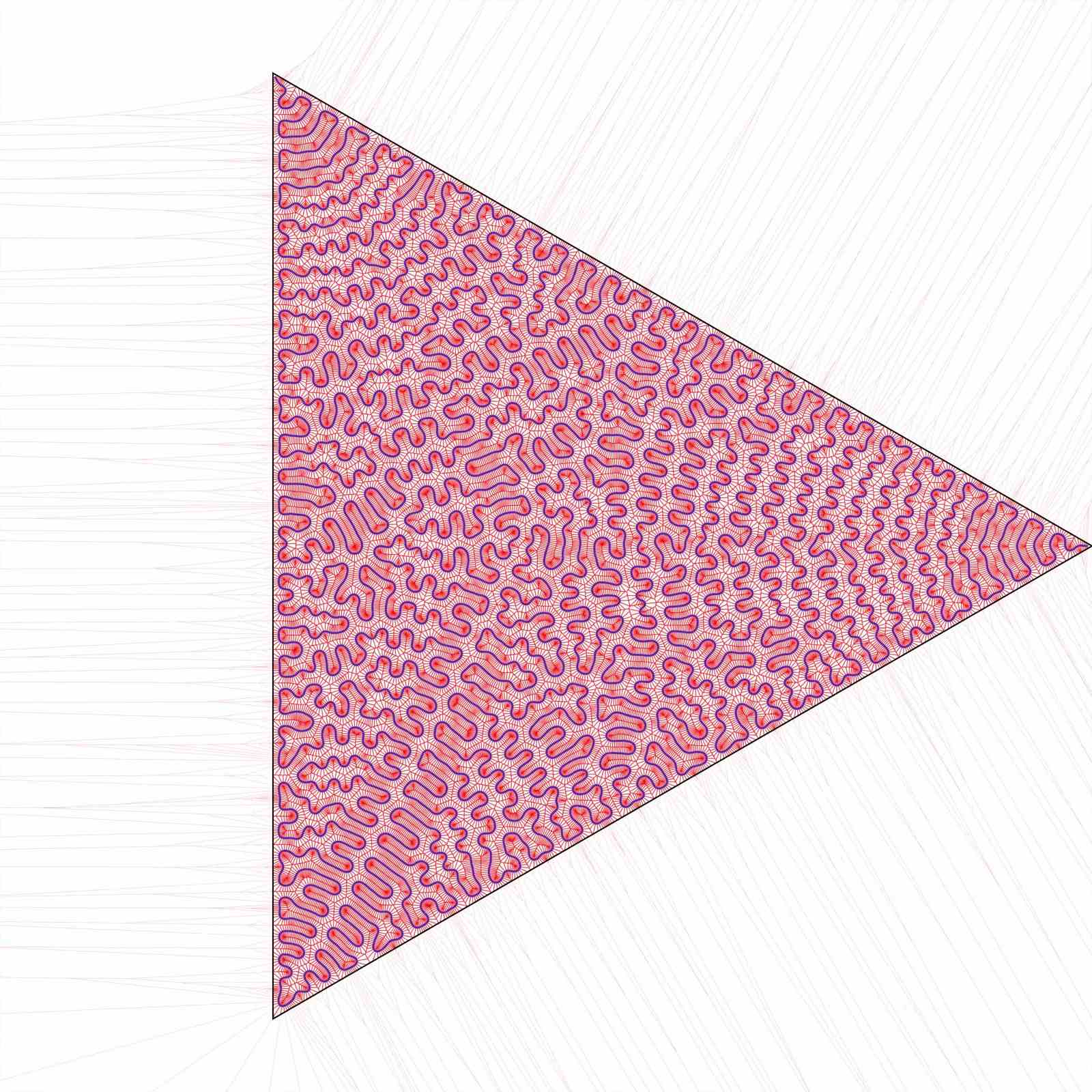}
    \caption{$i=250$}
  \end{subfigure}
\end{figure}
\begin{figure}[H] \ContinuedFloat
  \centering
  \begin{subfigure}{.49\linewidth}
    \centering
    \includegraphics[scale=.09]{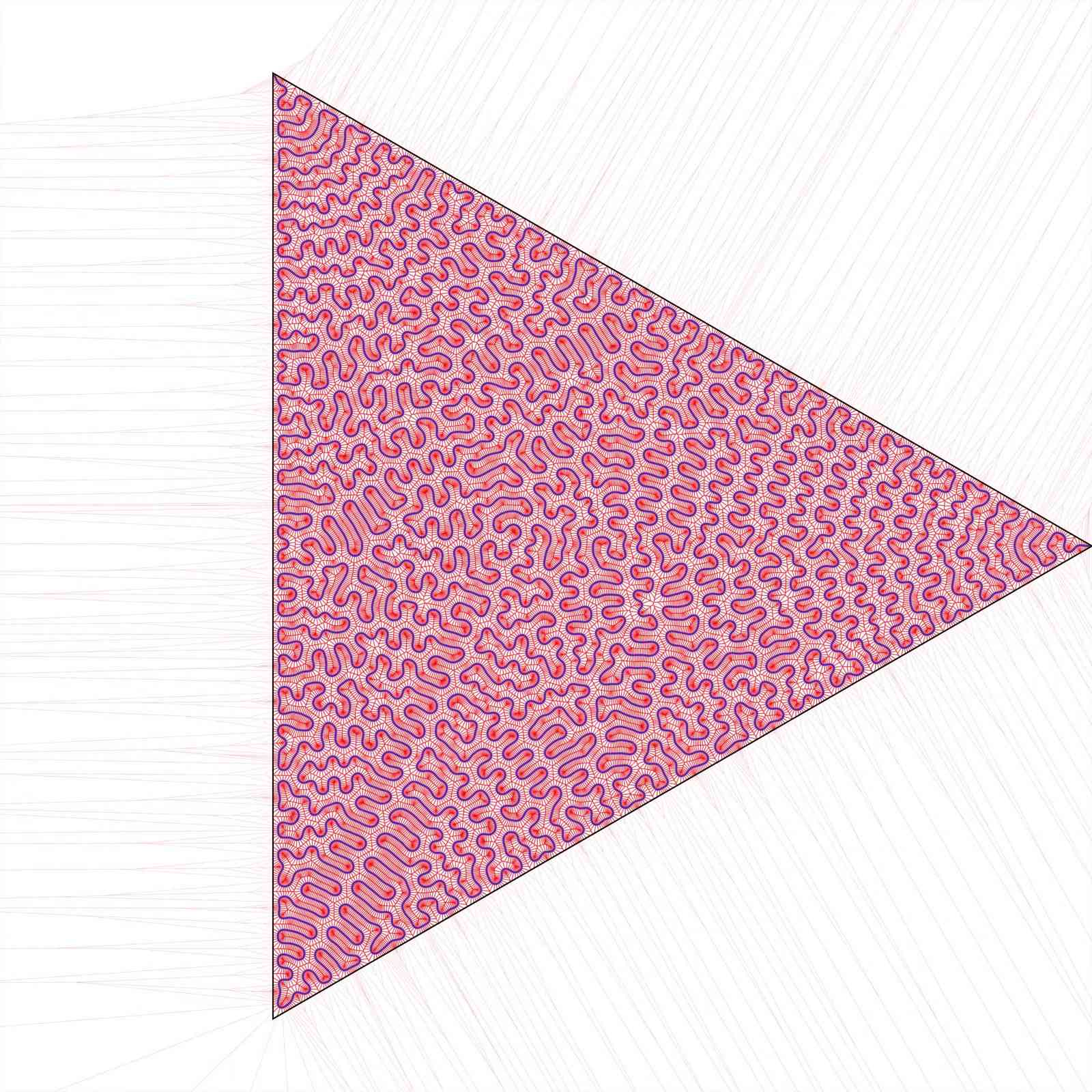}
    \caption{$i=300$}
  \end{subfigure}
\end{figure}

\subsection{Hexagonal Star Domain}
\label{sec:hexagonal-star-domain}
The numerics below were obtained with the following choices of
parameters:
\begin{itemize}
  \item Domain: Hexagonal star with vertices $\set{(\frac{3}{10} +
    \frac{7}{10} ((i+1) \bmod 2), \frac{2\pi k}{12}) \MID
    k = 0, \ldots, 11}$,
  \item Objective functional exponent: $\objp = 2$,
  \item Sobolev parameters: $\sobk = 2$, $\sobp = 2$
  \item Initial condition: 500 equally-spaced samples of
    $(t, t/250)$ on $[-1/5,1/5]$
  \item Spline resampling resolution: $\frac{1}{125}$
  \item Learning rate: $\learningrate(i) = \frac{9}{10} (i/1000)^2$
  \item Regularization penalty: $\lambda(i) = \frac{1}{10000} (1 -
    \learningrate(i))$
  \item Area rescaling: $\fbarysf(\iptsfj) \mapsto \fbarysf(\iptsfj) /
    \cmeas(\psin^{-1}(\iptsfj))^{17/20}$
  \item Smoothing window width: 1 point each side for the $\iptsfj$, 1
    point each side for $\fbarysfj$, 1 point each side for $\nabla
    \cost$.
\end{itemize}

\begin{figure}[H]
  \centering
  \begin{subfigure}{.49\linewidth}
    \centering
    \includegraphics[scale=.09]{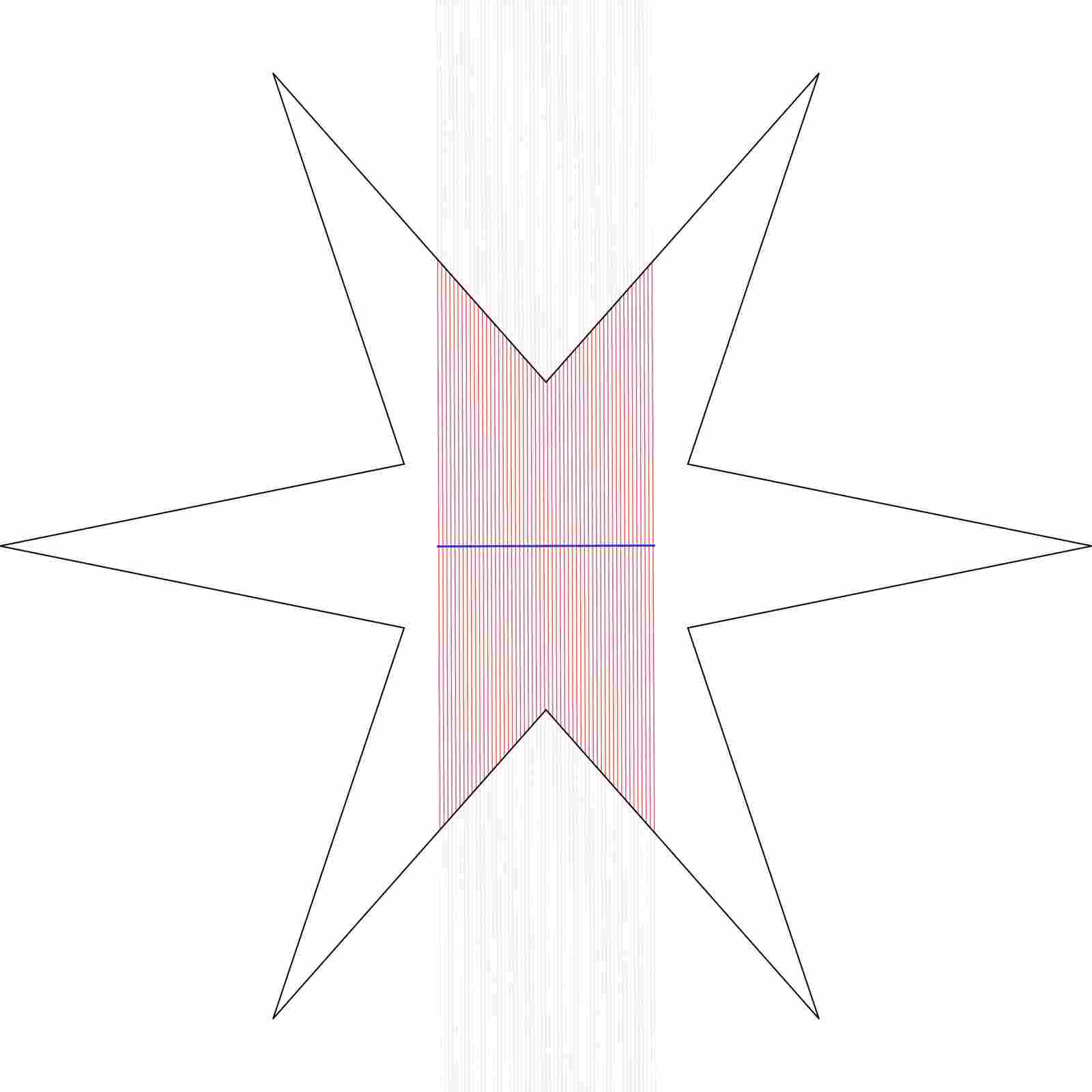}
    \caption{$i=1$}
  \end{subfigure}
  \begin{subfigure}{.49\linewidth}
    \centering
    \includegraphics[scale=.09]{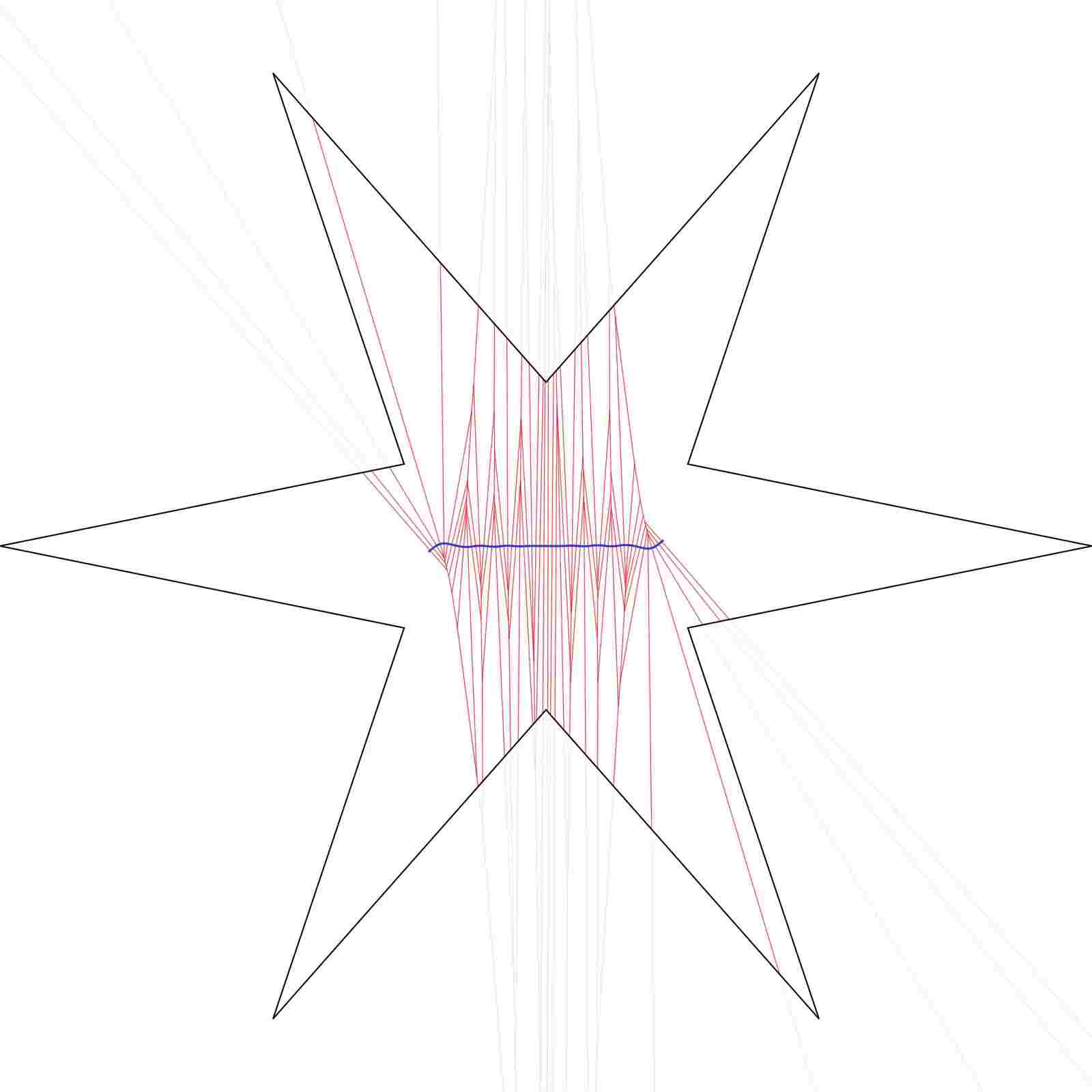}
    \caption{$i=50$}
  \end{subfigure}
\end{figure}
\begin{figure}[H] \ContinuedFloat
  \centering
  \begin{subfigure}{.49\linewidth}
    \centering
    \includegraphics[scale=.09]{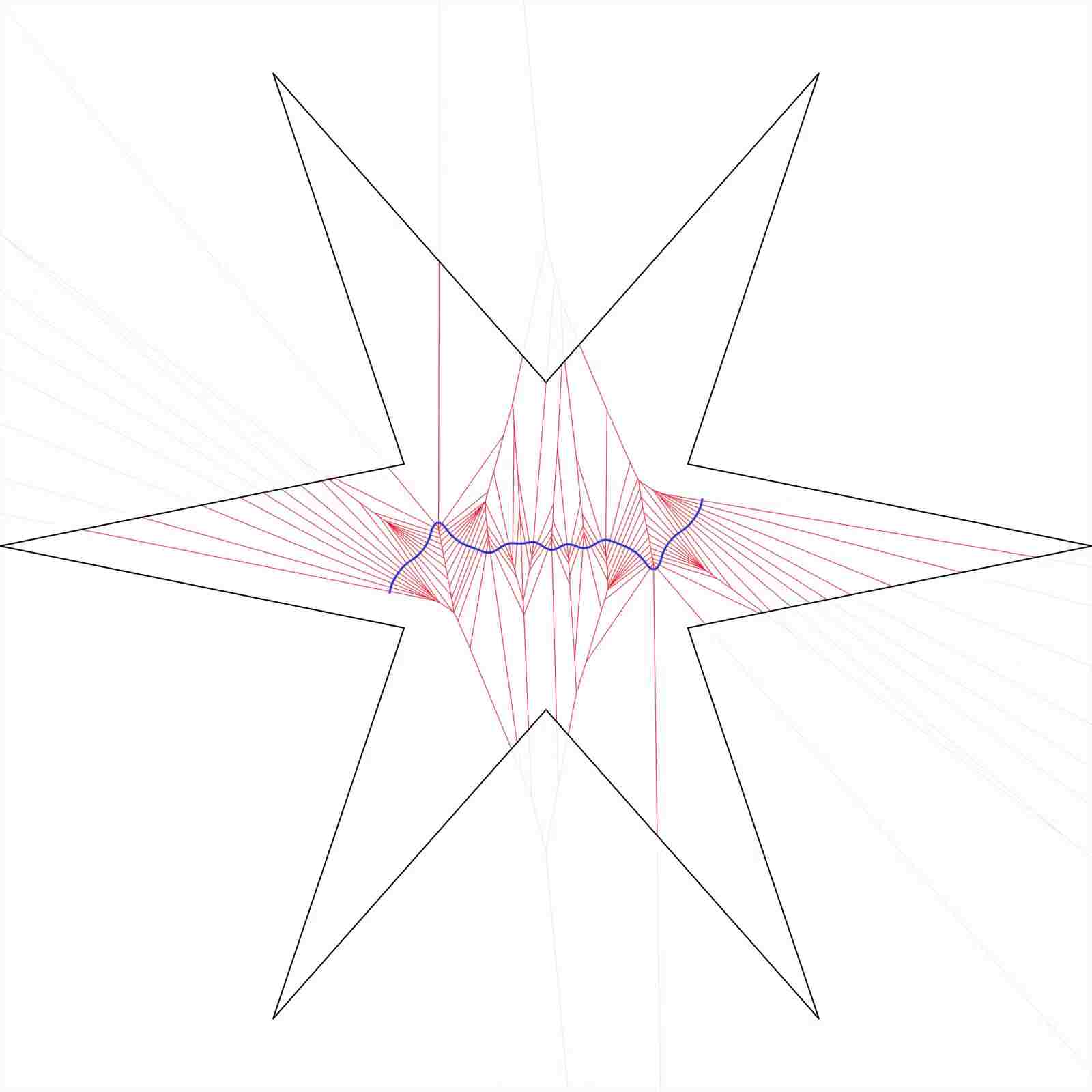}
    \caption{$i=100$}
  \end{subfigure}
  \begin{subfigure}{.49\linewidth}
    \centering
    \includegraphics[scale=.09]{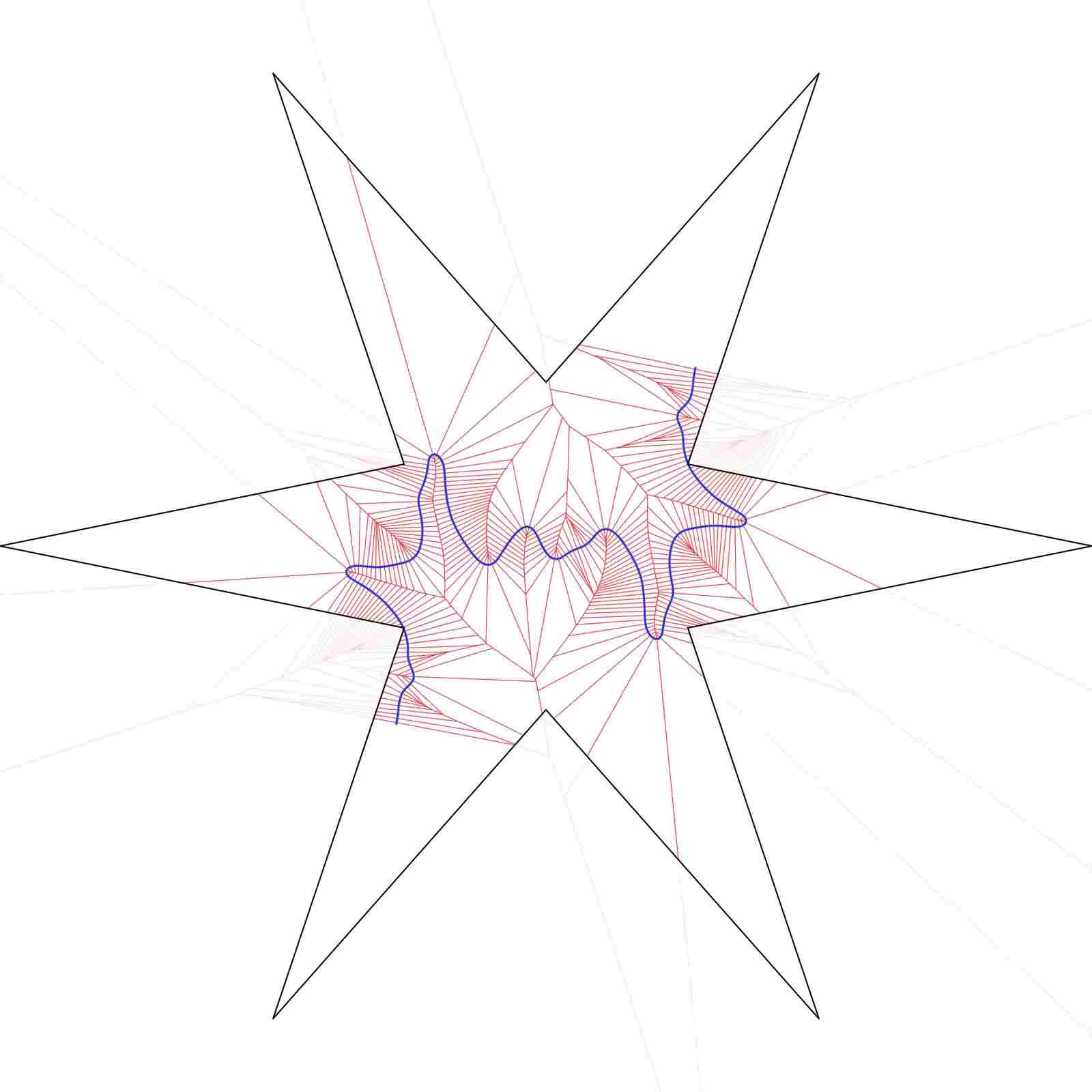}
    \caption{$i=150$}
  \end{subfigure}
\end{figure}
\begin{figure}[H] \ContinuedFloat
  \centering
  \begin{subfigure}{.49\linewidth}
    \centering
    \includegraphics[scale=.09]{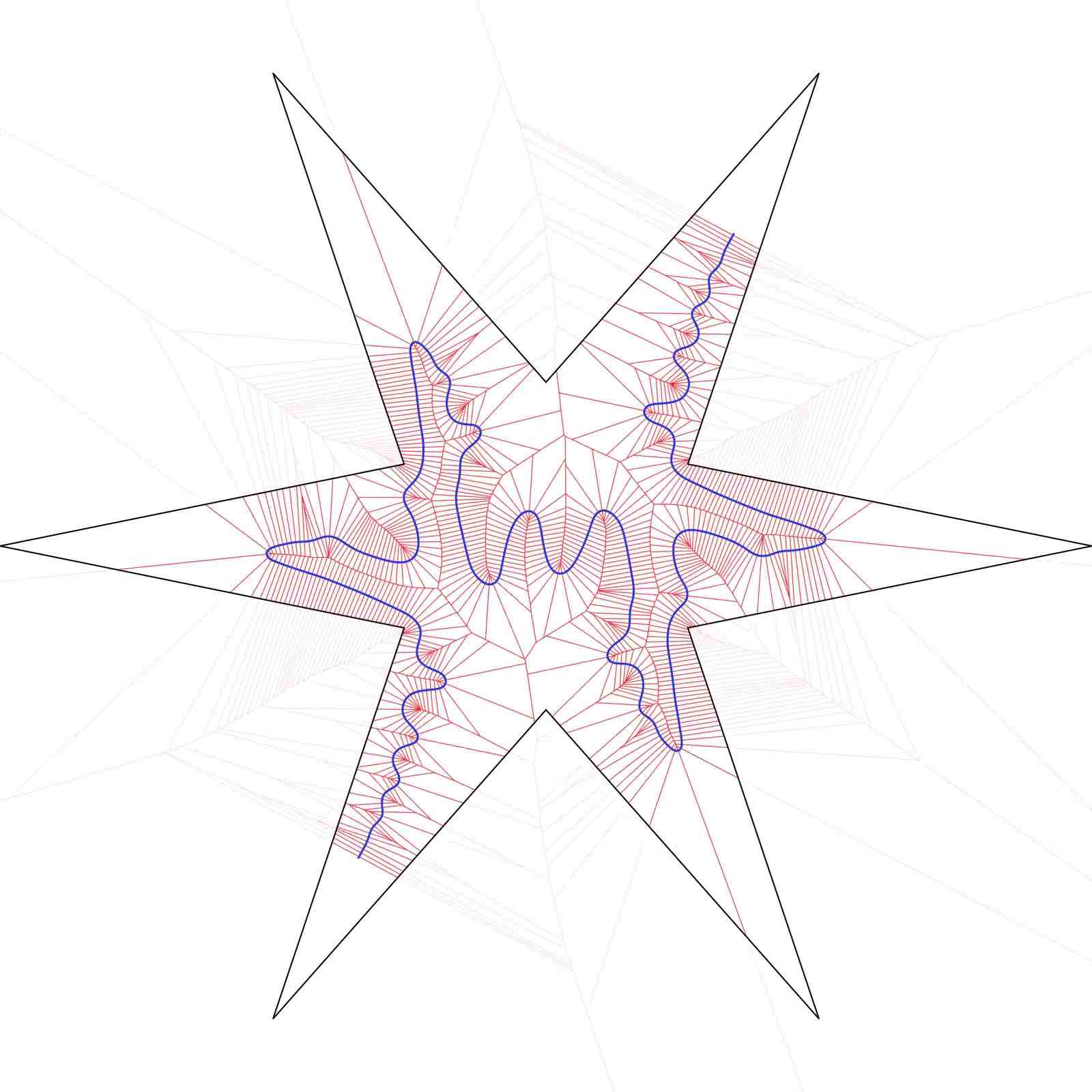}
    \caption{$i=200$}
  \end{subfigure}
  \begin{subfigure}{.49\linewidth}
    \centering
    \includegraphics[scale=.09]{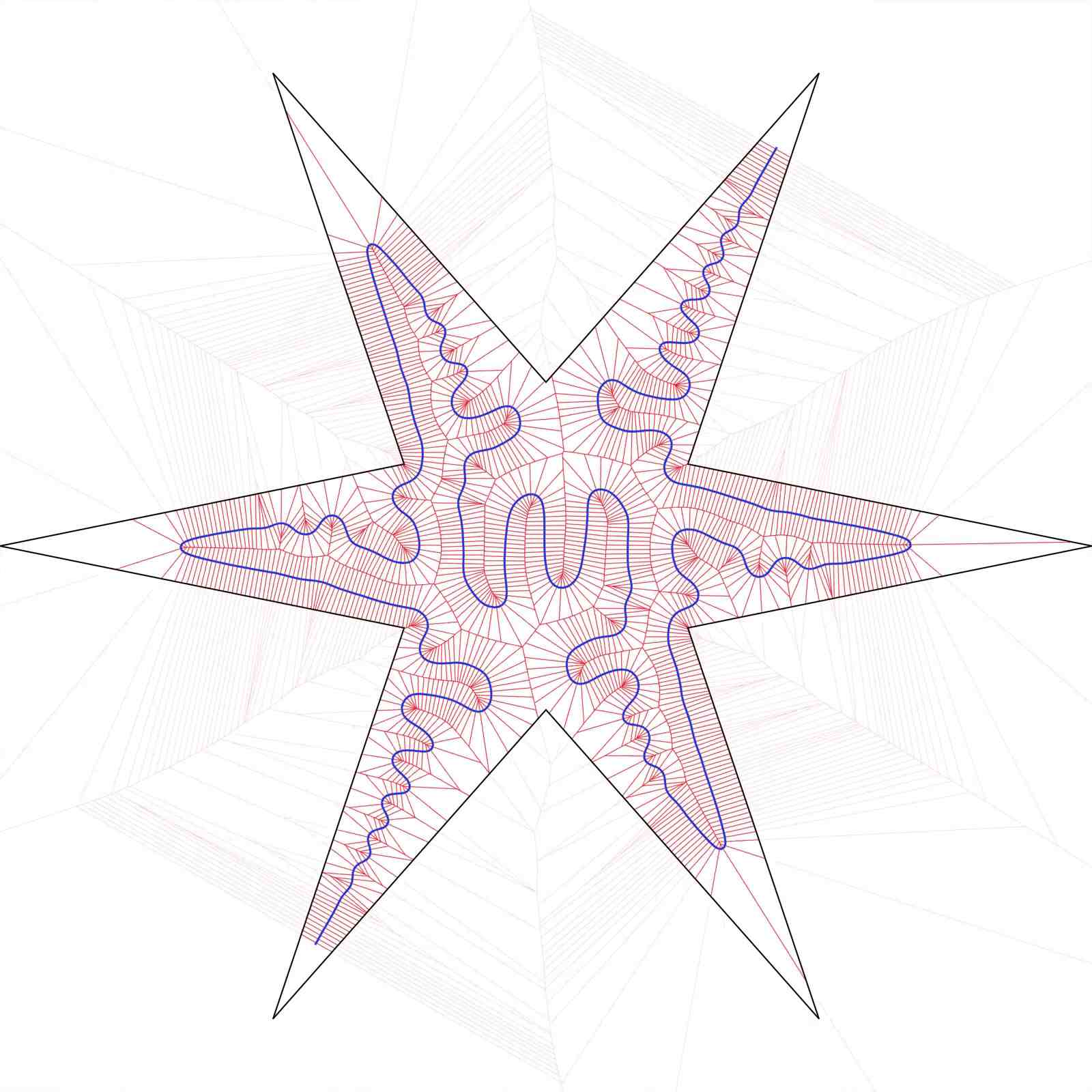}
    \caption{$i=250$}
  \end{subfigure}
\end{figure}
\begin{figure}[H] \ContinuedFloat
  \centering
  \begin{subfigure}{.49\linewidth}
    \centering
    \includegraphics[scale=.09]{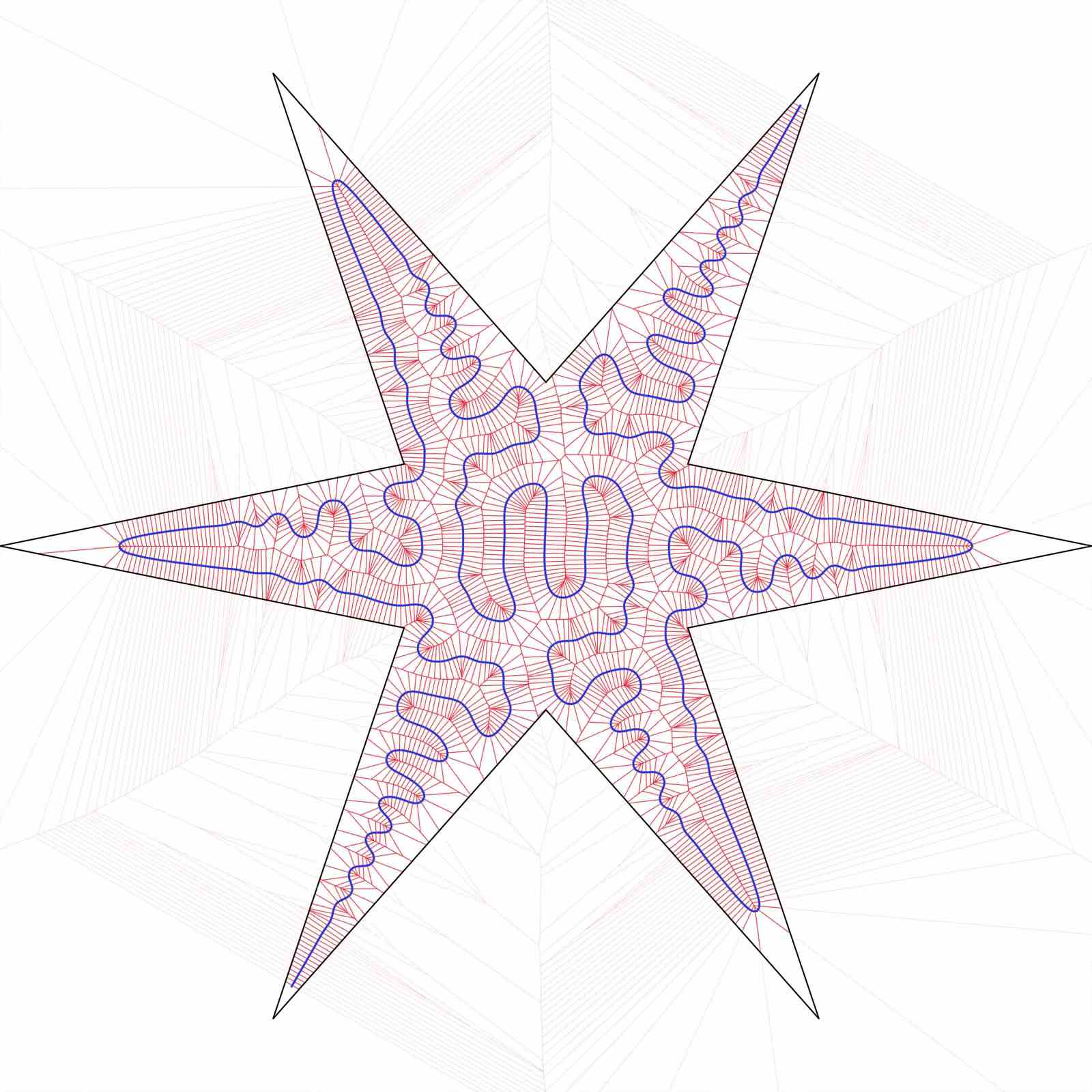}
    \caption{$i=300$}
  \end{subfigure}
  \begin{subfigure}{.49\linewidth}
    \centering
    \includegraphics[scale=.09]{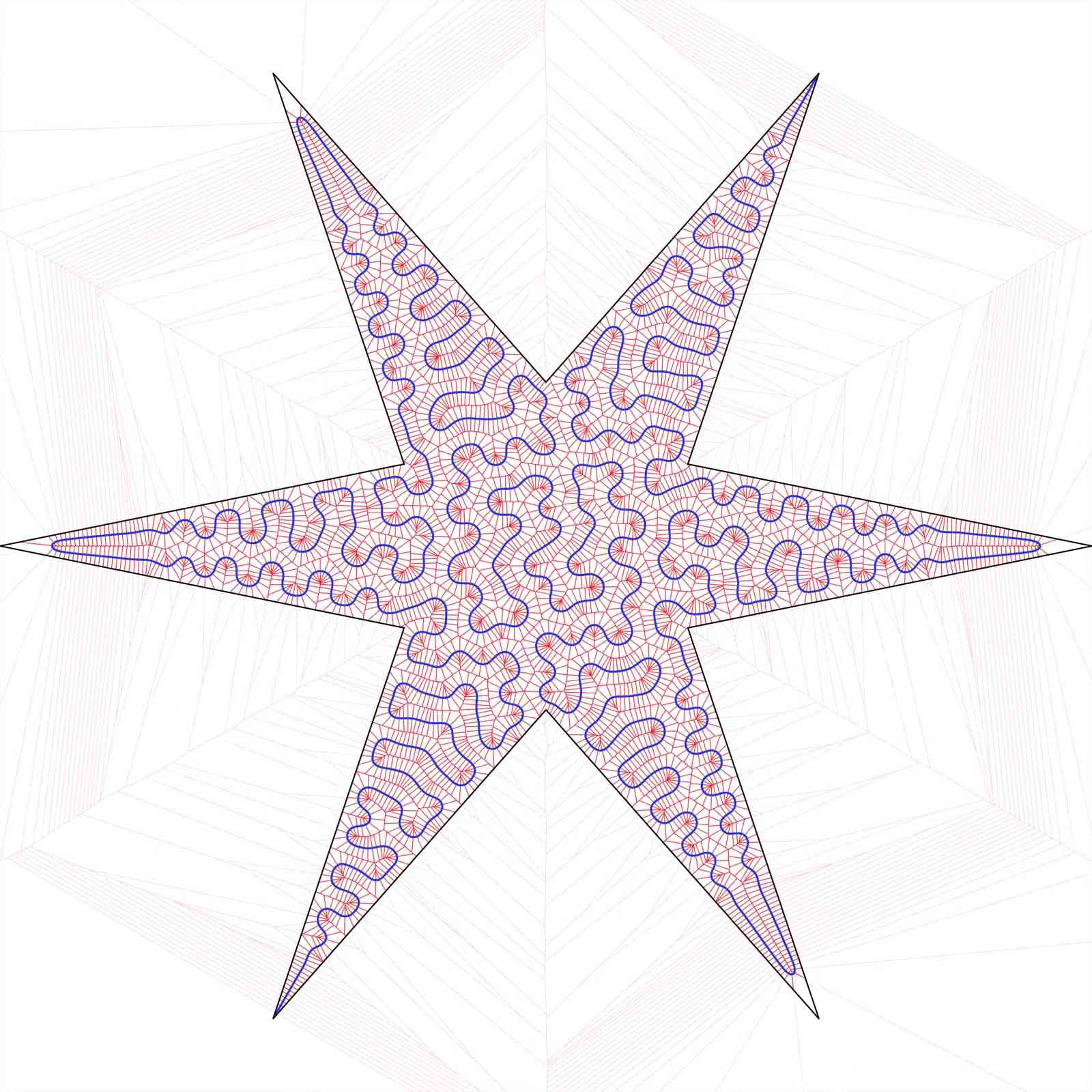}
    \caption{$i=500$}
  \end{subfigure}
\end{figure}
\begin{figure}[H] \ContinuedFloat
  \centering
  \begin{subfigure}{.49\linewidth}
    \centering
    \includegraphics[scale=.09]{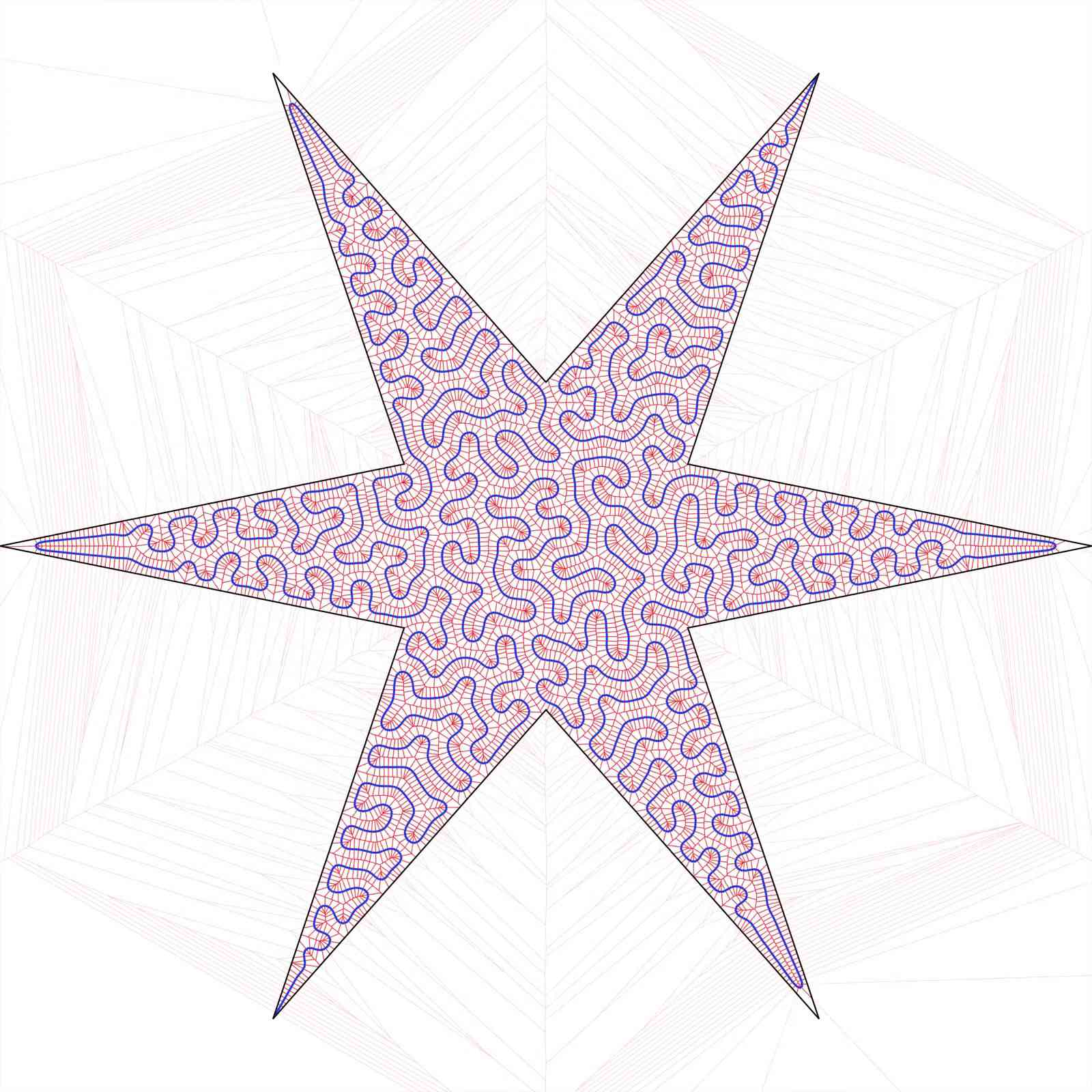}
    \caption{$i=750$}
  \end{subfigure}
  \begin{subfigure}{.49\linewidth}
    \centering
    \includegraphics[scale=.09]{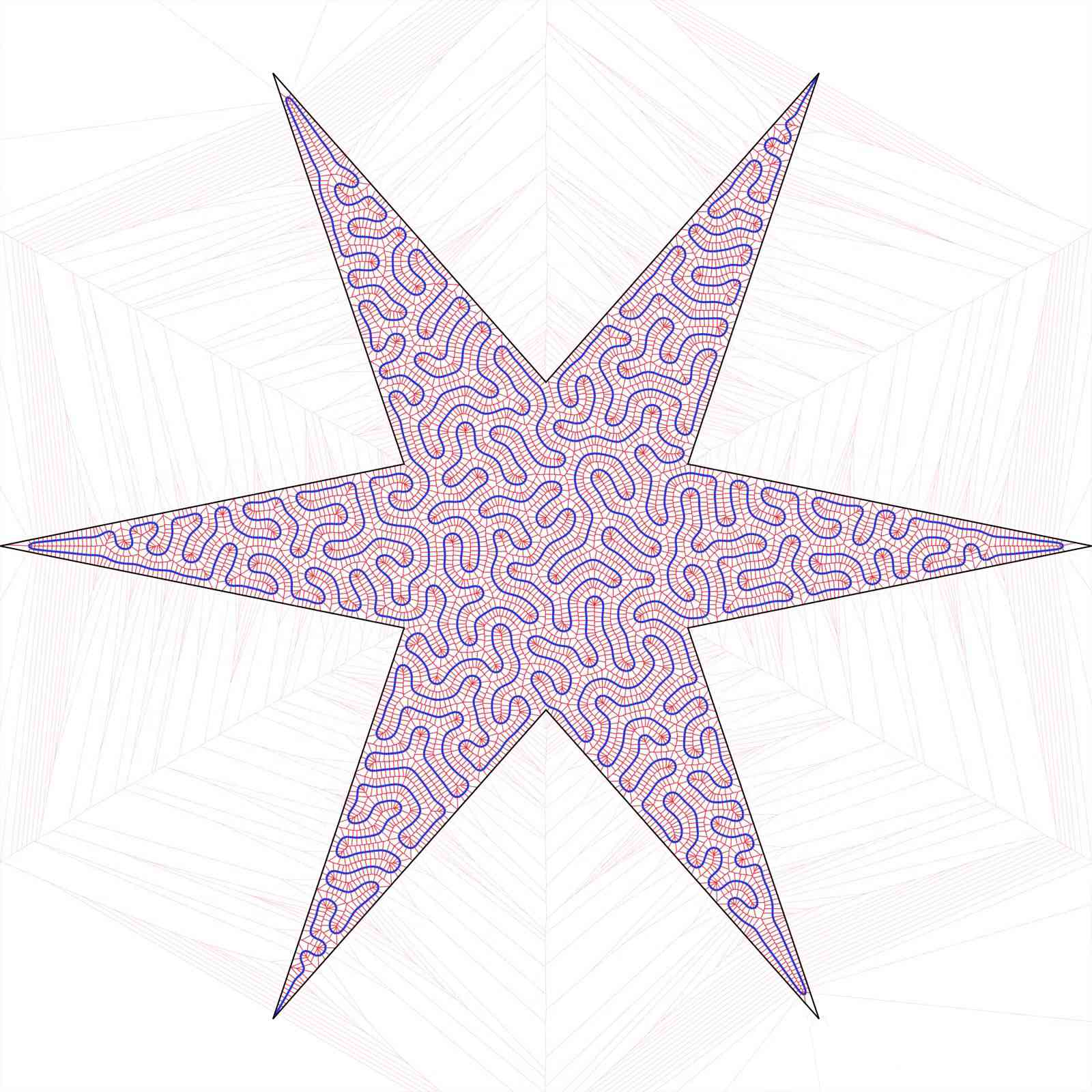}
    \caption{$i=1000$}
  \end{subfigure}
\end{figure}

\subsection{Chevron Domain}
The numerics below were obtained with the following choices of
parameters:
\begin{itemize}
  \item Domain: Chevron shape with vertices $\set{(-1, -\frac{3}{4}), (0,
    \frac{3}{4}), (1, -\frac{3}{4}), (0, -\frac{1}{4})}$
  \item Objective functional exponent: $\objp = 2$,
  \item Sobolev parameters: $\sobk = 2$, $\sobp = 2$
  \item Initial condition: 500 equally-spaced samples of
    $(t, t/250)$ on $[-1/5,1/5]$
  \item Spline resampling resolution: $\frac{1}{125}$
  \item Learning rate: $\learningrate(i) = \frac{9}{10} (i/1000)^2$
  \item Regularization penalty: $\lambda(i) = \frac{1}{10000} (1 -
    \learningrate(i))$
  \item Area rescaling: $\fbarysf(\iptsfj) \mapsto \fbarysf(\iptsfj) /
    \cmeas(\psin^{-1}(\iptsfj))^{17/20}$
  \item Smoothing window width: 1 point each side for the $\iptsfj$, 1
    point each side for $\fbarysfj$, 1 point each side for $\nabla
    \cost$.
\end{itemize}

\begin{figure}[H]
  \centering
  \begin{subfigure}{.49\linewidth}
    \centering
    \includegraphics[scale=.09]{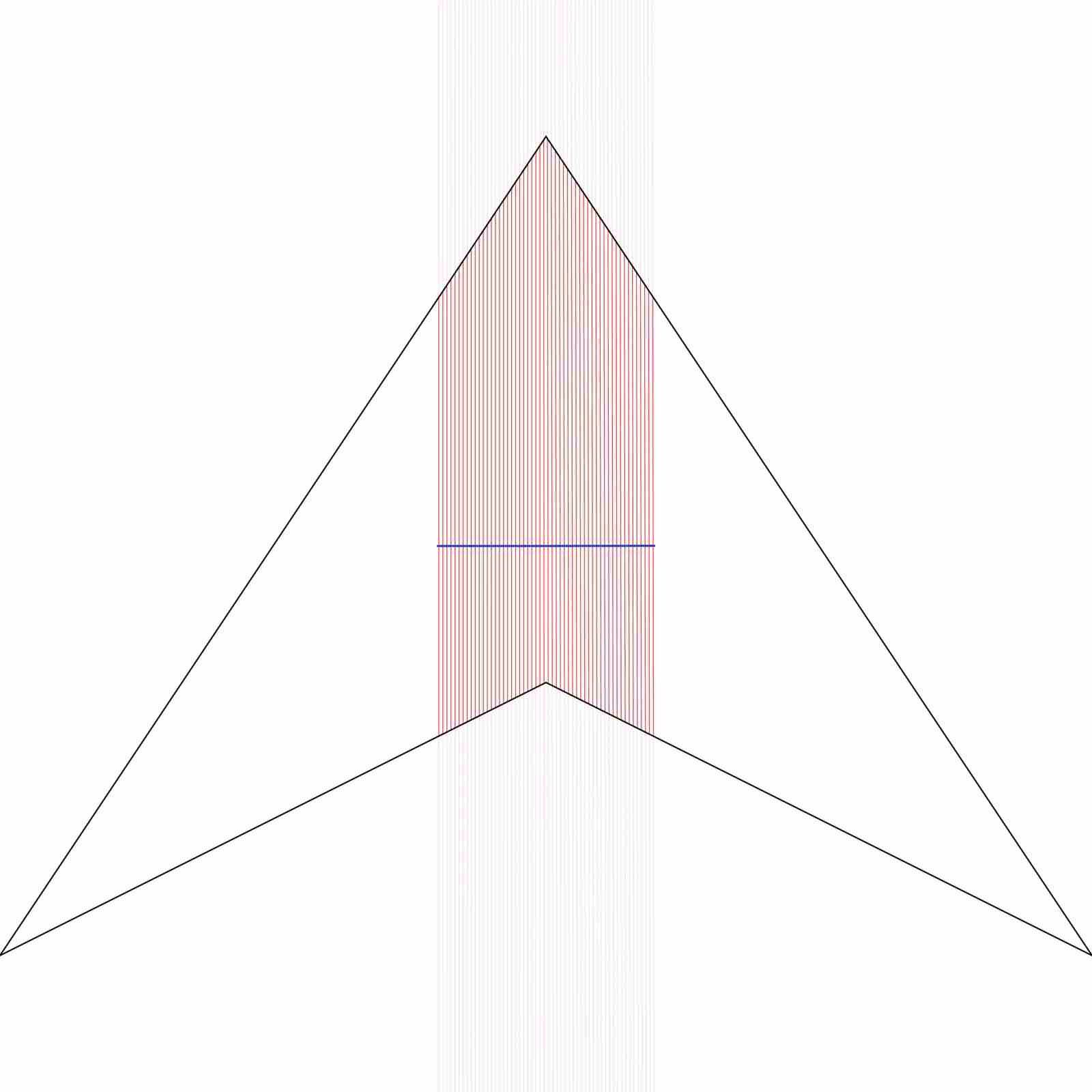}
    \caption{$i=1$}
  \end{subfigure}
  \begin{subfigure}{.49\linewidth}
    \centering
    \includegraphics[scale=.09]{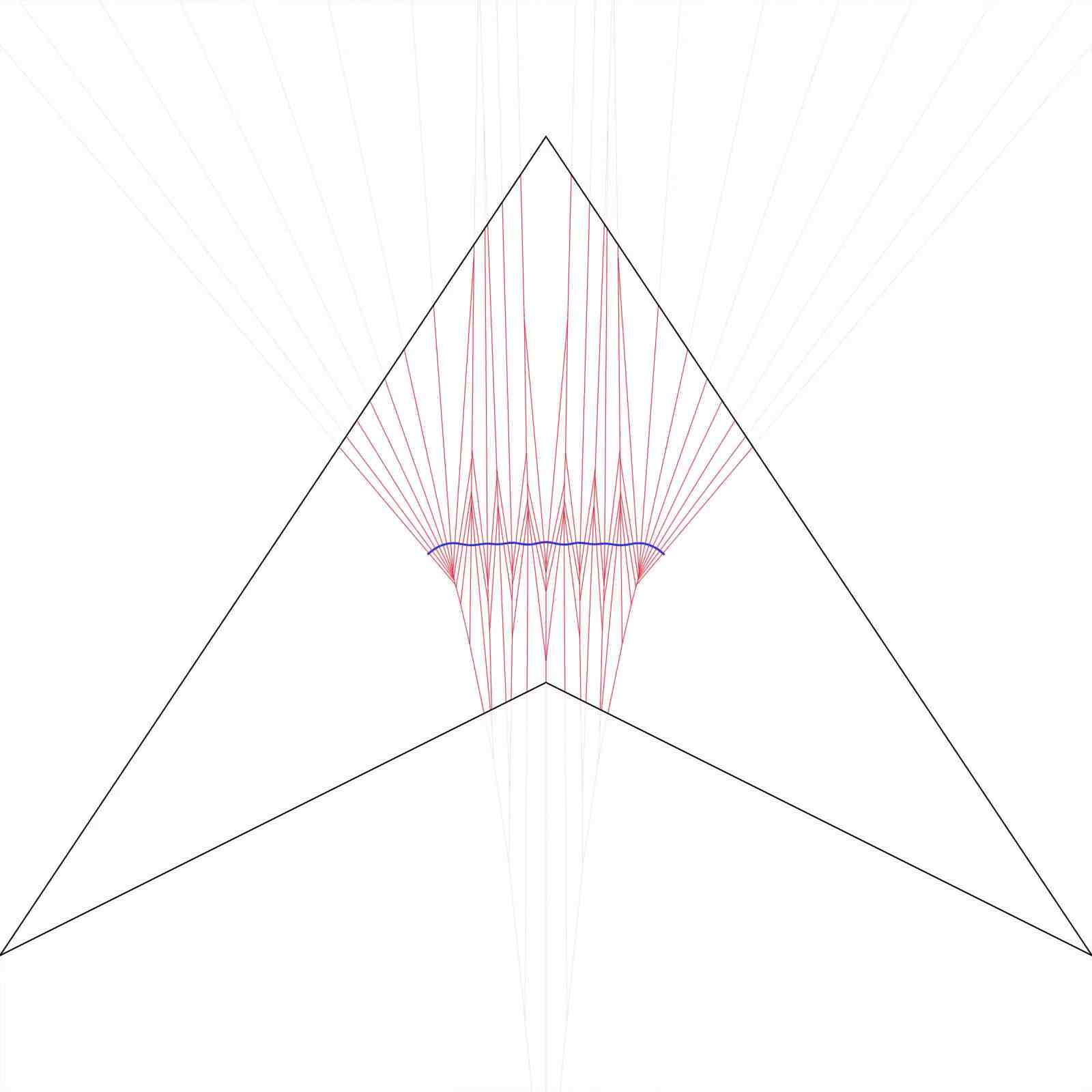}
    \caption{$i=50$}
  \end{subfigure}
\end{figure}
\begin{figure}[H] \ContinuedFloat
  \centering
  \begin{subfigure}{.49\linewidth}
    \centering
    \includegraphics[scale=.09]{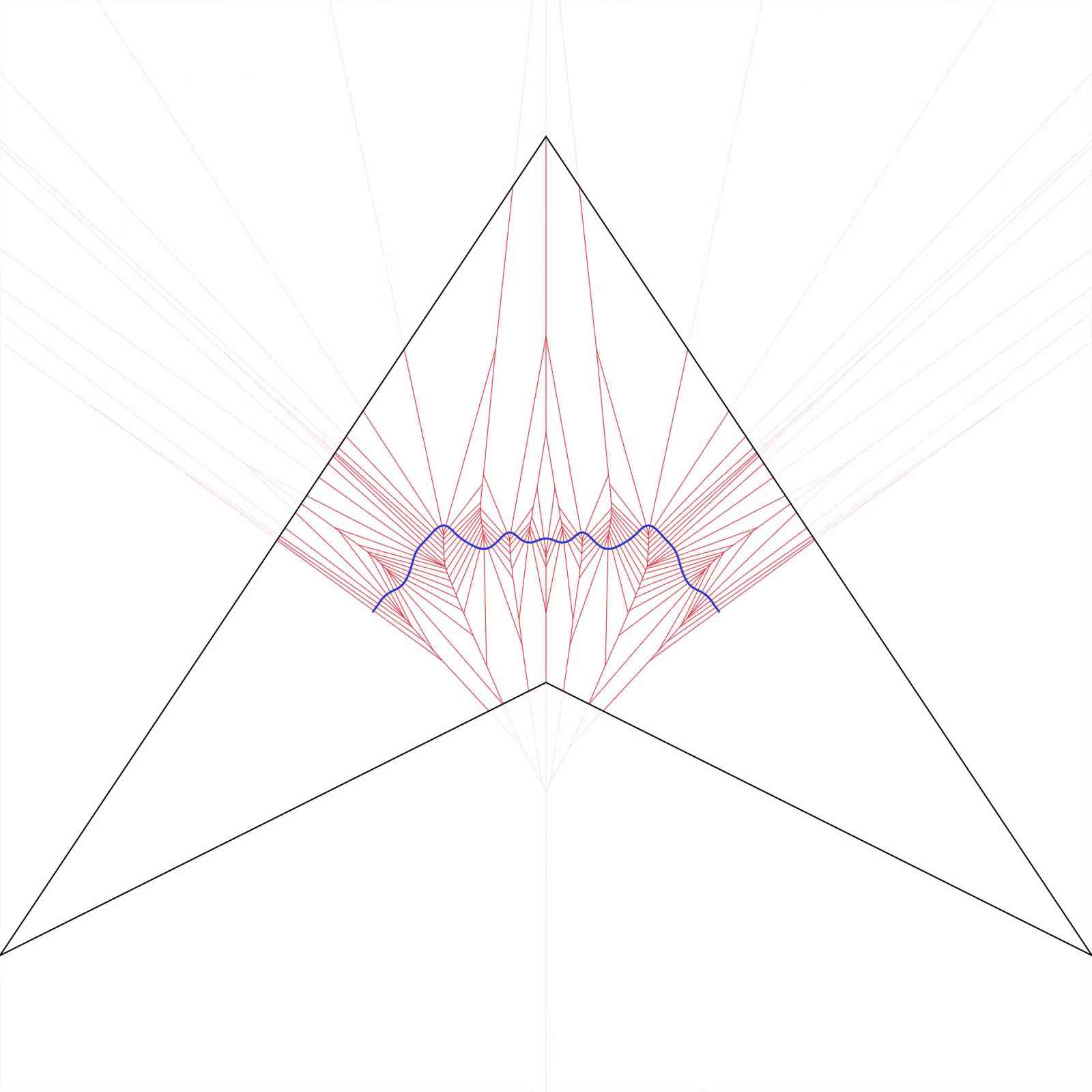}
    \caption{$i=100$}
  \end{subfigure}
  \begin{subfigure}{.49\linewidth}
    \centering
    \includegraphics[scale=.09]{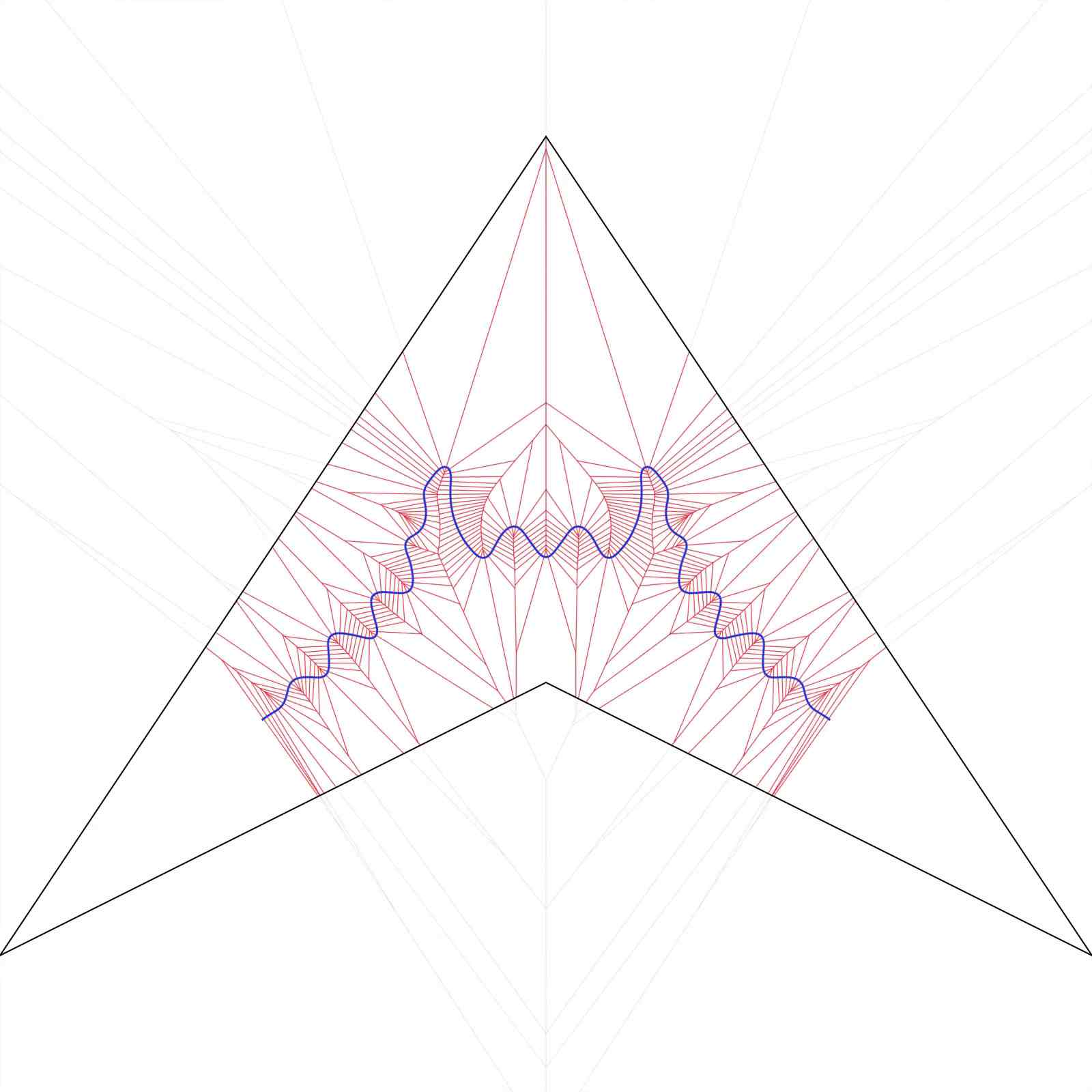}
    \caption{$i=150$}
  \end{subfigure}
\end{figure}
\begin{figure}[H] \ContinuedFloat
  \centering
  \begin{subfigure}{.49\linewidth}
    \centering
    \includegraphics[scale=.09]{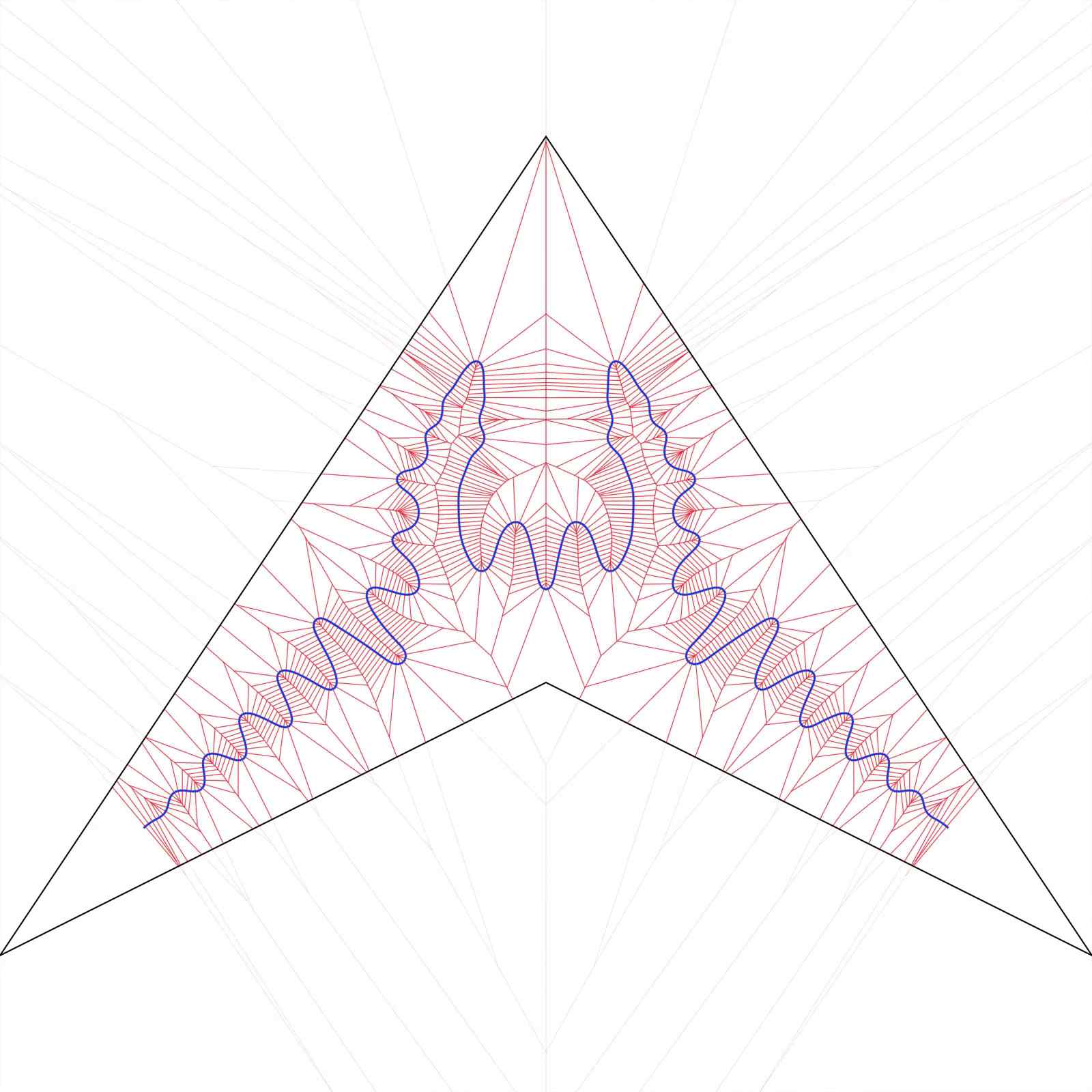}
    \caption{$i=200$}
  \end{subfigure}
  \begin{subfigure}{.49\linewidth}
    \centering
    \includegraphics[scale=.09]{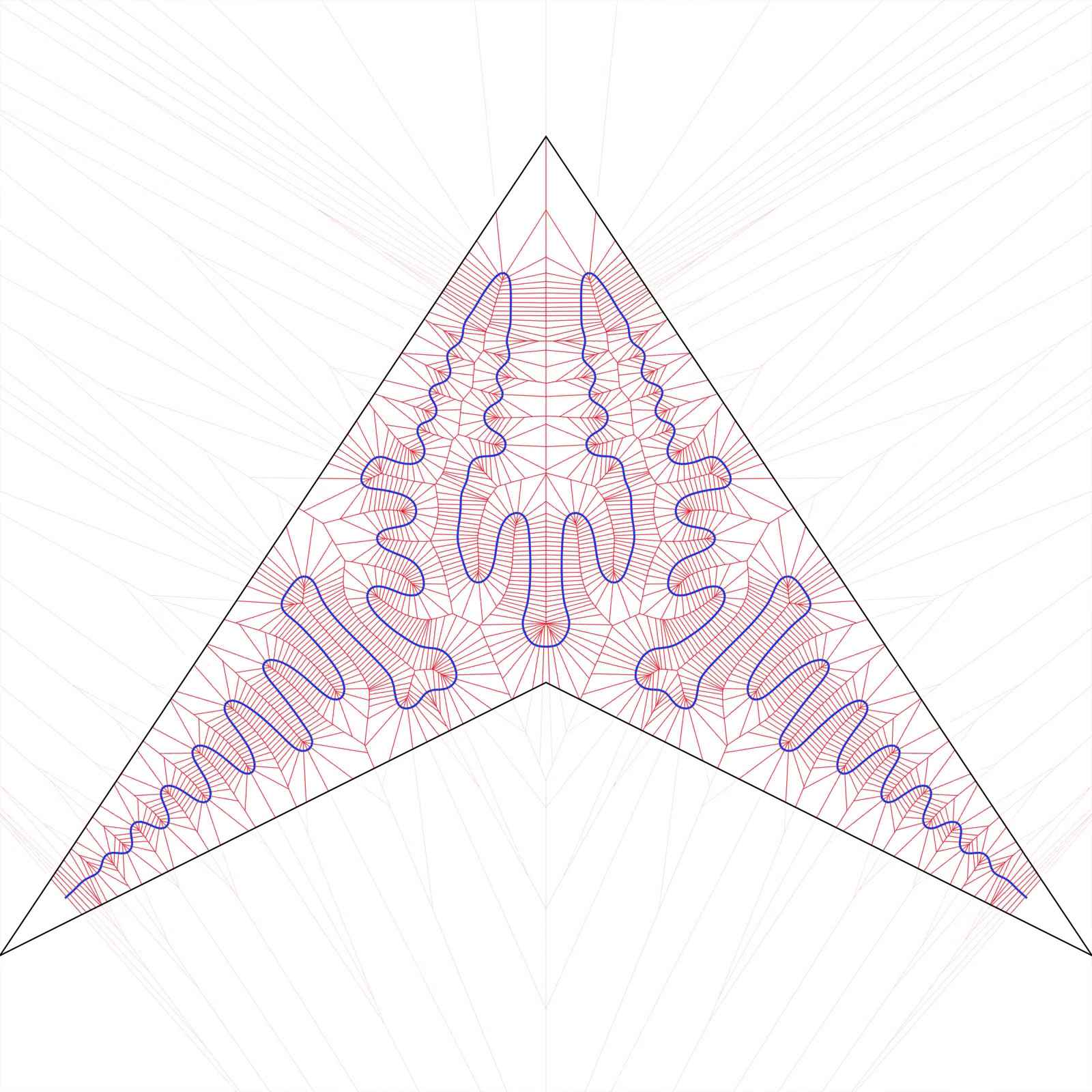}
    \caption{$i=250$}
  \end{subfigure}
\end{figure}
\begin{figure}[H] \ContinuedFloat
  \centering
  \begin{subfigure}{.49\linewidth}
    \centering
    \includegraphics[scale=.09]{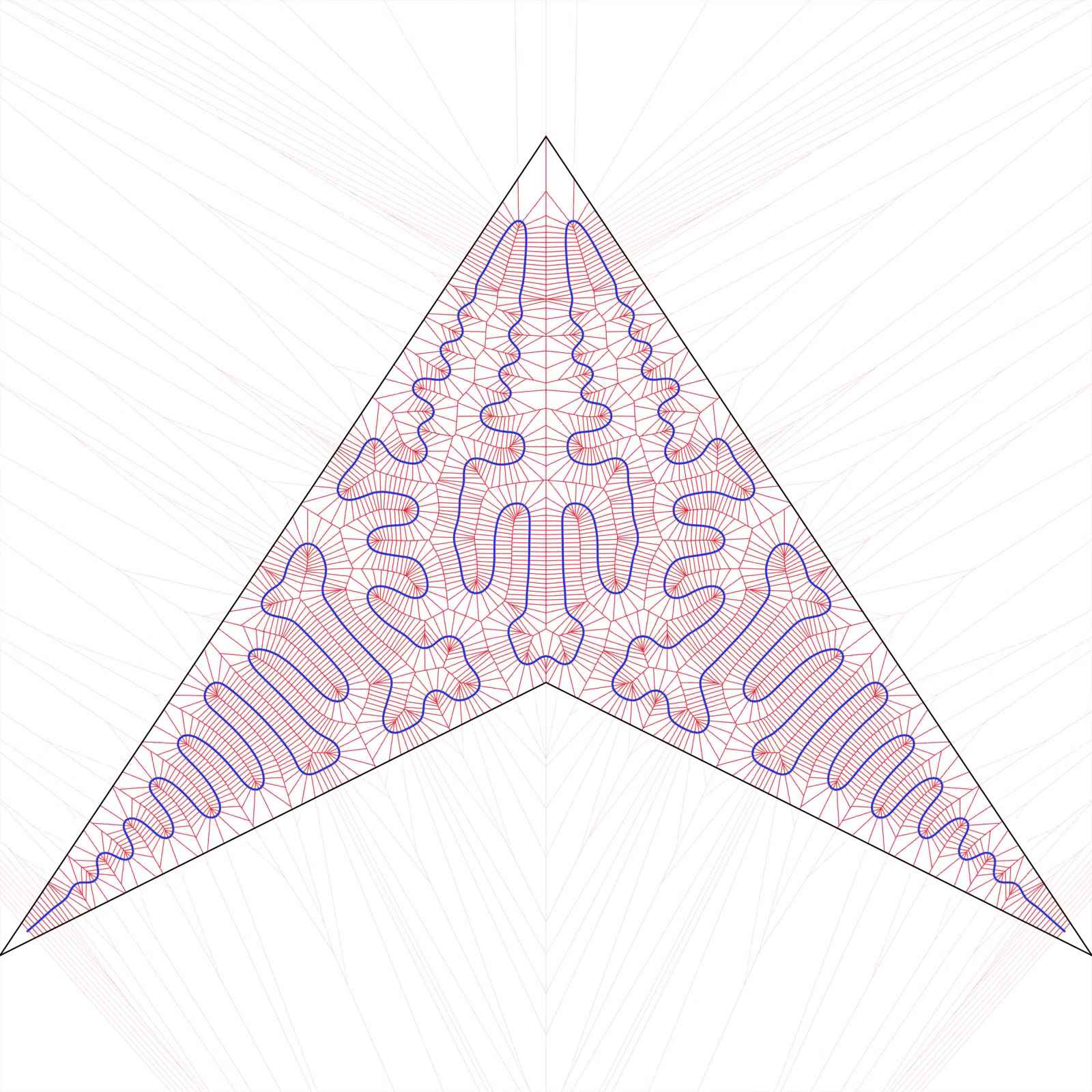}
    \caption{$i=300$}
  \end{subfigure}
  \begin{subfigure}{.49\linewidth}
    \centering
    \includegraphics[scale=.09]{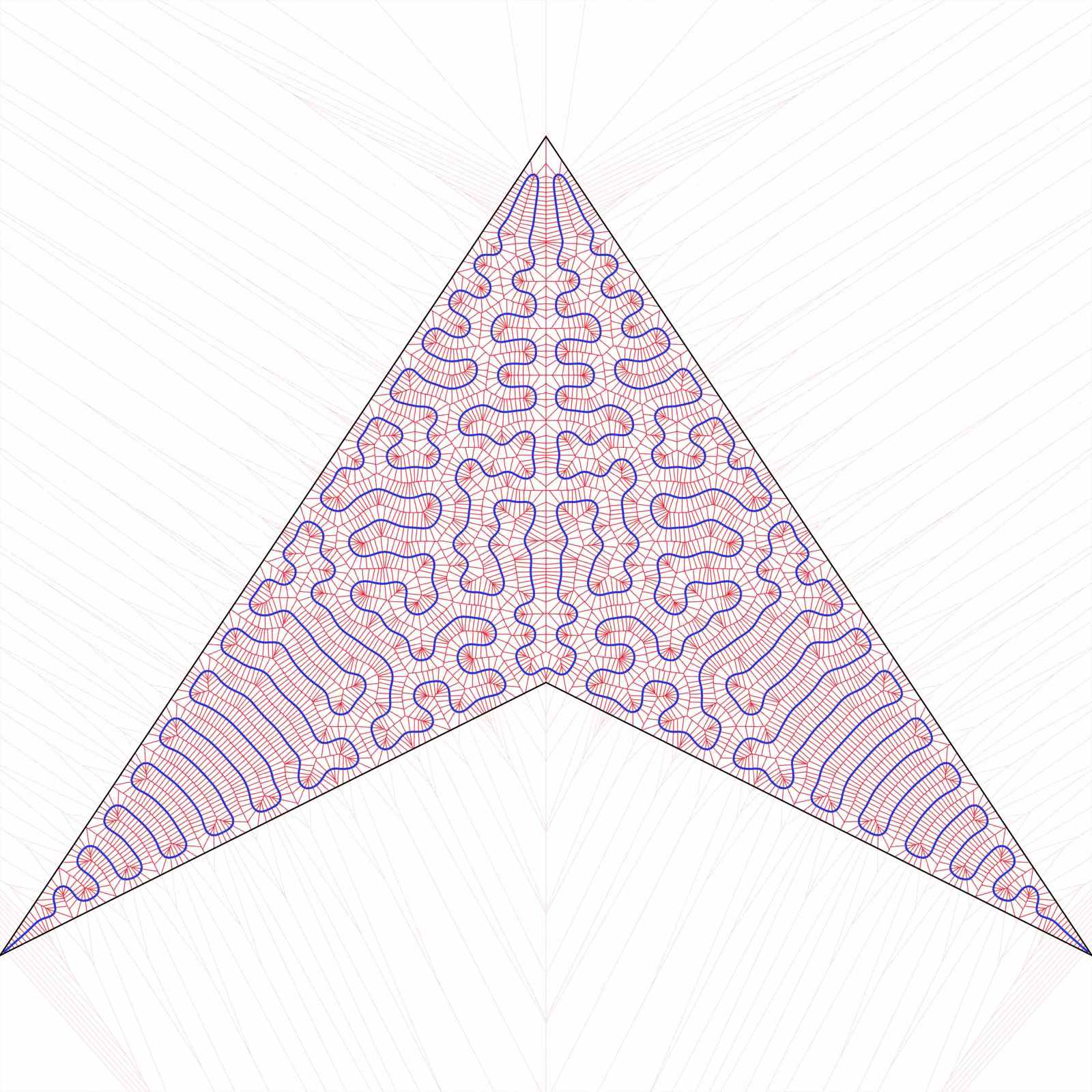}
    \caption{$i=500$}
  \end{subfigure}
\end{figure}
\begin{figure}[H] \ContinuedFloat
  \centering
  \begin{subfigure}{.49\linewidth}
    \centering
    \includegraphics[scale=.09]{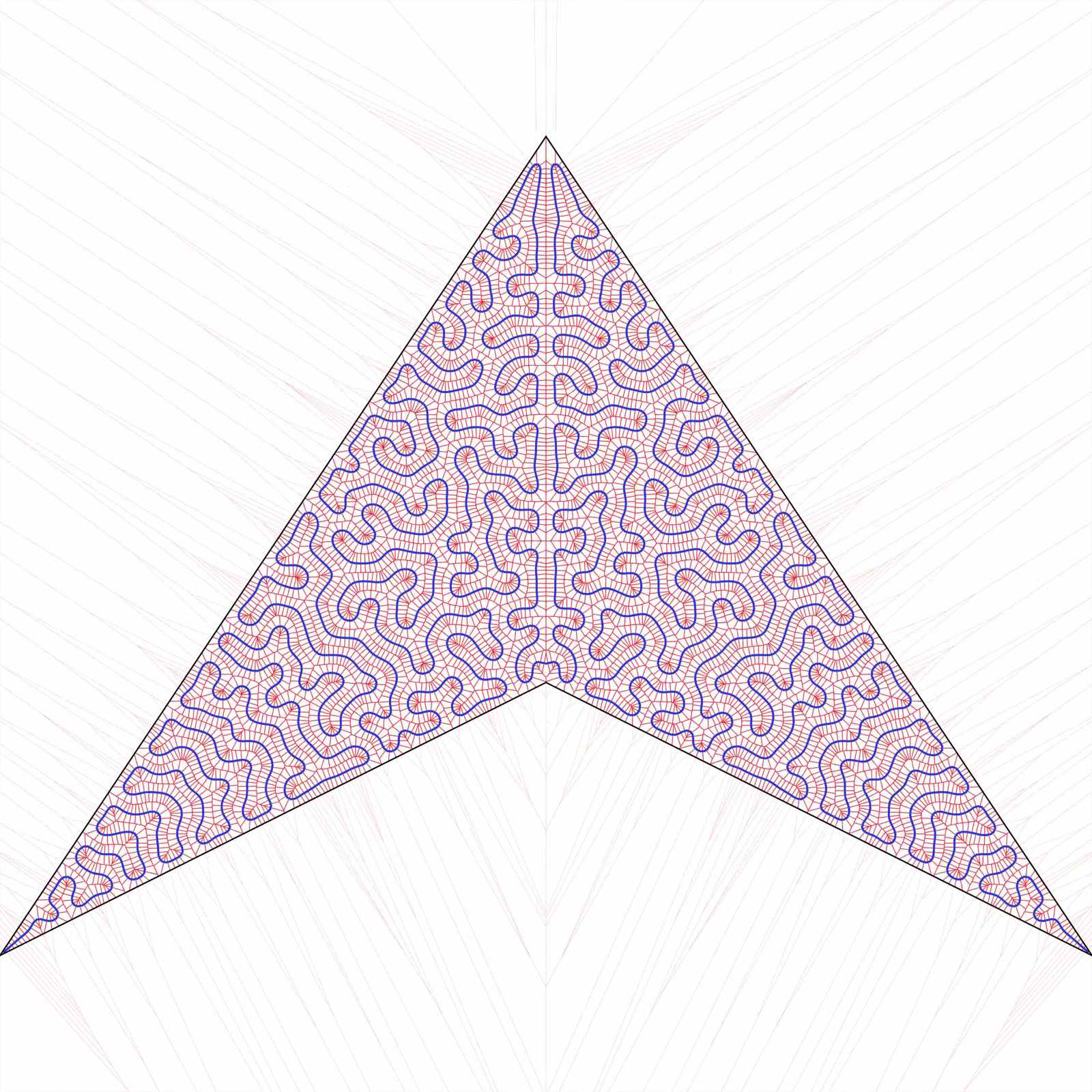}
    \caption{$i=750$}
  \end{subfigure}
  \begin{subfigure}{.49\linewidth}
    \centering
    \includegraphics[scale=.09]{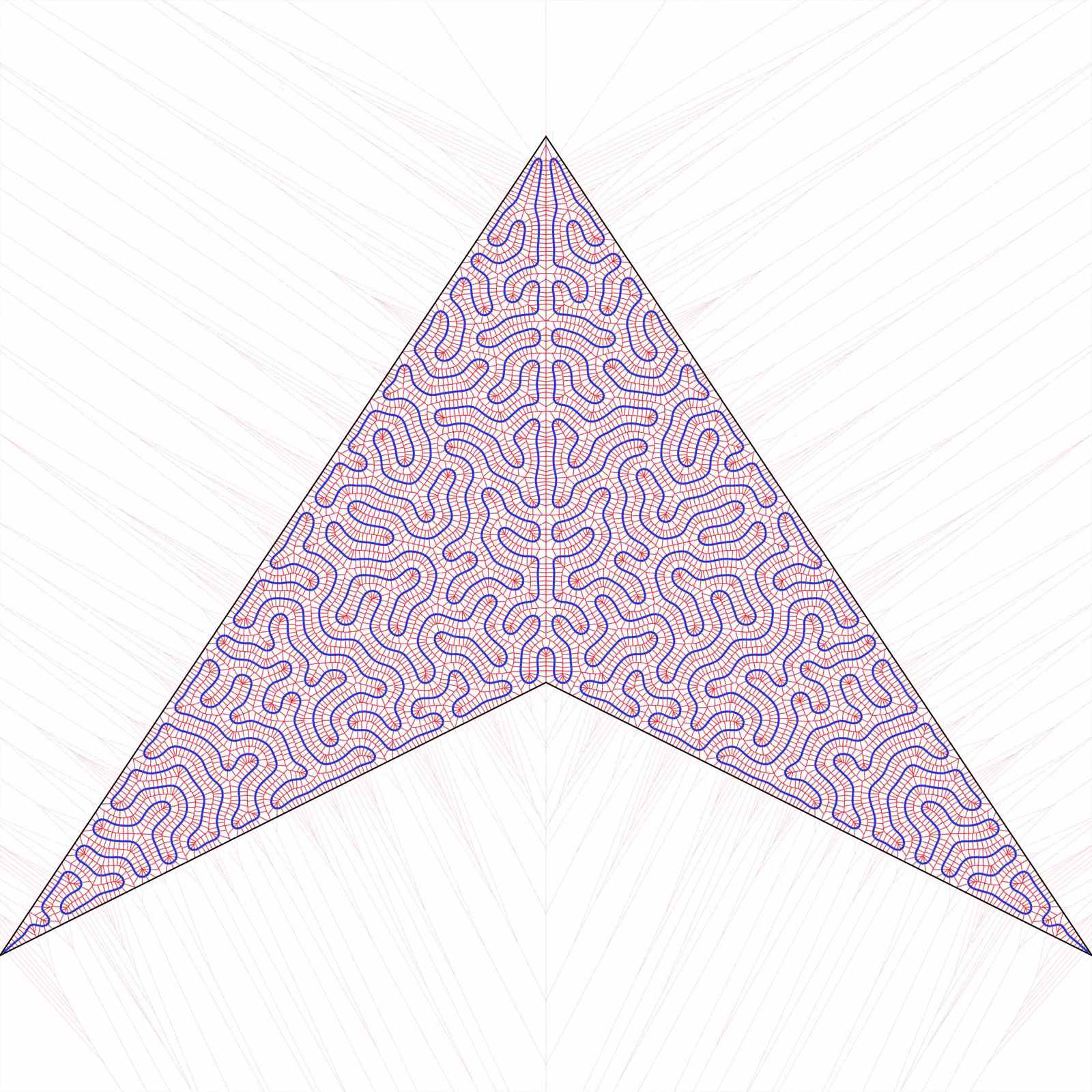}
    \caption{$i=1000$}
  \end{subfigure}
\end{figure}

\end{document}